\documentclass[letterpaper, 10 pt, journal, twoside]{IEEEtran}
\usepackage{amsthm}
\usepackage[pdftex]{graphicx}
\usepackage{subfigure}
\usepackage{amsmath}%
\usepackage{amssymb}
\usepackage[space]{cite}
\usepackage{soul}
\usepackage{multicol}
\usepackage{stfloats}
\everymath{\displaystyle}
\everydisplay{\textstyle}
\let\displaystyle\textstyle
\usepackage{booktabs}
\usepackage{threeparttable}
\usepackage{leftidx}
\renewcommand{\qedsymbol}{$\blacksquare$}

\usepackage{listings}
\usepackage{xcolor}
\usepackage{enumerate}
\usepackage{color}
\usepackage{algorithm}
\usepackage{algpseudocode}
\usepackage{bm}
\newtheorem{theorem}{Theorem}
\newtheorem{lemma}[theorem]{Lemma}

\theoremstyle{definition}

\usepackage{diagbox}
\usepackage{float}
\usepackage{epstopdf}
\usepackage{url}
\usepackage{multirow}
\usepackage{xcolor}
\usepackage{tikz}
\usetikzlibrary{arrows,shapes,chains}
\usepackage[linkcolor=black,citecolor=black,urlcolor=black,colorlinks=true]{hyperref}
\usepackage{enumitem}

\soulregister\cite7
\soulregister\citep7
\soulregister\citet7
\soulregister\ref7
\soulregister\it7
\soulregister\pageref7

\usepackage{arydshln}

\graphicspath{{figures/}}
\DeclareGraphicsExtensions{.pdf,.png,.jpg,.eps}

\IEEEoverridecommandlockouts
\title{\LARGE \bf  Kalman Filters on Differentiable Manifolds}%*\thanks{*Research is supported by DJI innovation.}}

\author{Dongjiao He$^{1}$, Wei Xu$^{1}$, Fu Zhang$^{1}$$^{*}$

\thanks{$^{*}$Corresponding author.}% <-this % stops a space  \vspace{-1cm}
\thanks{This project is supported by Hong Kong RGC ECS under grant 27202219.}
	\thanks{$^{1}$All authors are with Department of Mechanical Engineering, University of Hong Kong. {\tt\small \{hdj65822, xuweii, fuzhang\}@hku.hk}}
}%
\begin{document}	
	\maketitle
	\begin{abstract}
Kalman filter is presumably one of the most important and extensively used filtering techniques in modern control systems. Yet, nearly all current variants of Kalman filters are formulated in the Euclidean space $\mathbb R^n$, while many real-world systems (e.g., robotic systems) are really evolving on manifolds. In this paper, we propose a method to develop Kalman filters for such on-manifold systems. Utilizing $\boxplus\backslash\boxminus$ operations and further defining a $\oplus$ operation on the respective manifold, we propose a canonical representation of the on-manifold system. Such a canonical form enables us to separate the manifold constraints from the system behaviors in each step of the Kalman filter, ultimately leading to a generic and symbolic Kalman filter framework that are naturally evolving on the manifold. Furthermore, the on-manifold Kalman filter is implemented as a toolkit in $C$++ packages which enables users to implement an on-manifold Kalman filter just like the normal one in $\mathbb R^n$: the user needs only to provide the system-specific descriptions, and then call the respective filter steps (e.g., predict, update) without dealing with any of the manifold constraints. The existing implementation supports full iterated Kalman filtering for systems on manifold $\mathcal{M} = \mathbb{R}^m \times SO(3) \times \cdots \times SO(3) \times \mathbb{S}^2 \times \cdots \times \mathbb{S}^2 $ or any of its sub-manifolds, and is extendable to other types of manifold when necessary. The proposed symbolic Kalman filter and the developed toolkit are verified by implementing a tightly-coupled lidar-inertial navigation system. Results show that the developed toolkit leads to superior filtering performances and computation efficiency comparable to hand-engineered counterparts. Finally, the toolkit is opened sourced at~\url{https://github.com/hku-mars/IKFoM} to assist practitioners to quickly deploy an on-manifold Kalman filter. 
    \end{abstract}
    
\section{Introduction}

Kalman filter and its variants have been widely used in modern control systems. However, Kalman filters typically treat the state space as a Euclidean space $\mathbb{R}^n$ while many real-world systems (e.g., robotic systems) usually have their states evolving on manifolds (e.g., rotation group $SO(3)$). To circumvent the constraints imposed by the manifold, an elegant and effective way is to perform Kalman filtering steps (i.e., predict and update) in the error-state, i.e., the error-state extended Kalman filter (ESEKF), which have been used in various robotic applications such as attitude estimation~\cite{markley2003attitude,trawny2005indirect,markley2014fundamentals}, online extrinsic calibration~\cite{Mirzaei2008A, kelly2011visual}, GPS/IMU navigation~\cite{sola2017quaternion}, visual inertial navigation\cite{Kleinert2010Errorstate, Mourikis2007ICRA,Li2013EKFbasedVIO,SLynen2013sensorfusion, bloesch2017iterated, huai2018robocentric} and lidar-inertial navigation \cite{hesch2010laser, qin2020lins, xu2020fast}. The basic idea of ESEKF is to repeatedly parameterize the state trajectory $\mathbf{x}_{\tau} \in \mathcal{M}$ by an error state trajectory $\delta \mathbf{x}_{\tau | k} \in \mathbb{R}^n$ from the current state predict ${\mathbf{x}}_{\tau|k}$: $\mathbf{x}_{\tau}  = {\mathbf{x}}_{\tau|k} \boxplus \delta \mathbf{x}_{\tau | k} $. Then a normal extended Kalman filter (EKF) is performed on the error state trajectory $\delta \mathbf{x}_{\tau|k}$ to update the error state, and adds the updated error state back to the original state on manifolds. Since this error is small, minimal parameterization (e.g., axis-angle and Euler angle) can be employed without concerning the singularity issue (see Fig. \ref{fig:error_state}). In addition, compared to other techniques such as unscented Kalman filter (UKF), the efficiency of the extended Kalman filter is higher. With the superiority of accuracy, stability and efficiency, the ESEKF provides an elegant Kalman filter framework for nonlinear robotic systems.

Despite of these advantages, an ESEKF is usually much more difficult than normal EKFs for a certain on-manifold system. Due to the lack of canonical representation of systems on manifolds, existing ESEKFs are designed case by case, and usually require users to fully understand its underlying principles (e.g., switching between the original state and the error state) and manually derive each step (e.g., predict, update, reset) from scratch for a customized system. Although this may seem like a mere book-keeping issue but in practice it tends to be particularly cumbersome and error-prone, especially for systems of high dimension, such as robotic swarms and systems with augmented internal states~\cite{Lu2019IMU} or multiple extrinsic parameters~\cite{lin2020decentralized}. Besides system dimension, difficulty in hand-derivation also rapidly escalates when error-state is coupled with iteration (e.g., iterated error-state Kalman filter), which has recently found more applications in visual-inertial \cite{bloesch2017iterated} and lidar-inertial navigation \cite{qin2020lins, xu2020fast} to mitigate the linearization error in extended Kalman filters \cite{huang2013quadratic,huang2007convergence}. 

\begin{figure}[t]
\setlength{\abovecaptionskip}{-0.1cm} 
\setlength{\belowcaptionskip}{-0.2cm} 
\vspace{-0.3cm}
	\centering
	\includegraphics[width=0.6\columnwidth]{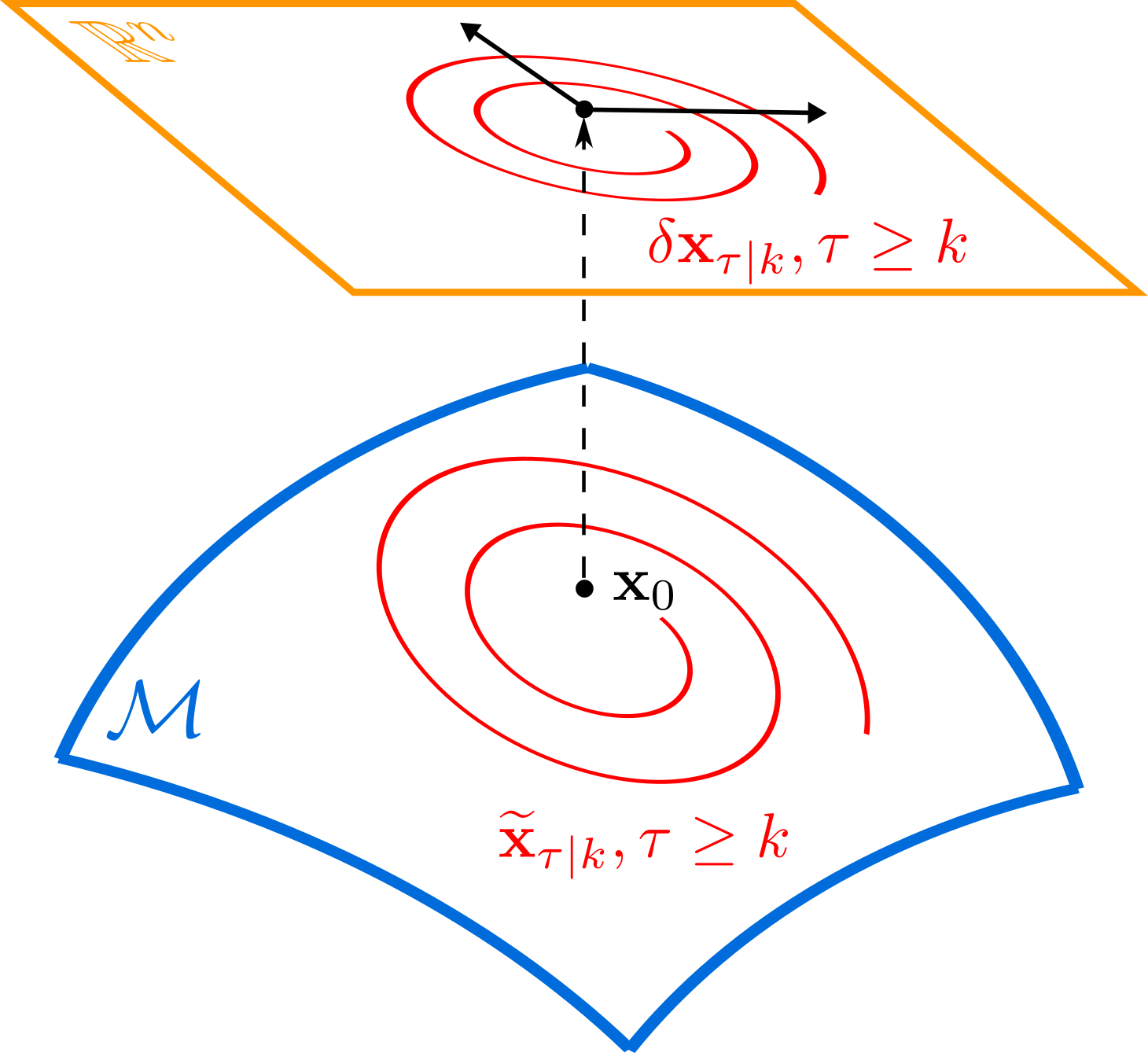}
	\caption{Illustration of the error state trajectory when $\mathcal{M}$ is a Lie group. The $\mathbb{R}^n$ space is locally homeomorphic to $\mathcal{M}$ space at the identity $\mathbf{x}_0$. $\delta \mathbf{x}_{\tau|k}$ is a minimal parameterization of the error state $\widetilde{\mathbf{x}}_{\tau|k} = \mathbf{x}_{\tau|k}^{-1} \cdot \mathbf{x}_{\tau} \in\mathcal{M}$.
		\label{fig:error_state}}
	\vspace{-0.3cm}
\end{figure}

In this paper, we address above issues by naturally integrating manifold constraints into the Kalman filter framework. Specifically, our contributions are as follows: 1) We propose a canonical and generic representation for on-manifold systems in discrete time, i.e., $\mathbf{x}_{k+1} = \mathbf{x}_k \oplus \left( \Delta t \mathbf{f} (\mathbf{x}_k, \mathbf{u}_k, \mathbf{w}_k) \right)$; 2) Based on the canonical system representation, we show that in each step of a Kalman filter, the manifold constraints are well separated from the system behaviors, enabling us to integrate the manifold constraints into the Kalman filter. We further derive a fully iterated, symbolic, and error-state Kalman filter termed as {\it IKFoM} on the canonical system representation; 3) We develop an open source $C$++ package. Its main advantage is hiding all the Kalman filter derivations and the treatment of manifold constraints, leaving the user to supply system-specific descriptions only and call the respective filter steps (e.g., predict, update) in the running time; 4) We verify our formulation and implementations with a tightly-coupled lidar-inetial navigation system and on various real-world datasets.

\section{Related Work}

Many approaches have been proposed to overcome the discrepancy between Kalman filters on $\mathbb{R}^n$ and real-world systems evolving on manifolds (e.g., rotation group $SO(3)$. One straightforward way is to use a parameterization of the state with minimal dimensions \cite{grewal2007global}. This minimal parameterization unfortunately has singularities. For example, Euler angle representation of $SO(3)$ has singularities at $\pm 90^{\circ}$ rotations along the second rotation axis, and the axis-angle representation has singularities at $180^{\circ}$ of rotations \cite{lynch2017modern}. Workarounds for this singularity exist and they either avoid these parts of the state space, as done in the Apollo Lunar Module \cite{KLUMPP1974133}, or switch between alternative orderings of the parameterization each of which exhibits singularities in different areas of the state space.

Another approach is representing system states using redundant parameters (i.e., over-parameterization). For example, unit quaternion is often used to represent rotations on $SO(3)$. Yet, the over-parameterization shifts the problem from system representation to filtering algorithm: viewing the over-parameterized state as a vector in the Euclidean space and applying the Kalman filter (or its variants) will make the predicted state no longer lie on the manifold (i.e., $\mathbf q^T\!\mathbf q\! \neq\! 1$). One ad-hoc way to ensure the predicted state on the manifold is normalization. Since the normalization imposes constraints on the state, the propagated covariance should be adjusted in parallel. For example, a unit quaternion $\mathbf q^T\! \mathbf q\! =\! 1$ leads to an error satisfying $\mathbf q^T\! \delta \mathbf q \!=\! 0$, which means the error is zero along the direction $\mathbf q$ and the corresponding covariance should be adjusted to zero~\cite{Davison2003Real} too, which is therefore singular. Although Kalman filters still work with this singular covariance as long as the innovation covariance is positive definite, it is unknown if this phenomenon causes further problems, e.g., the zero-uncertainty direction could create overconfidence in other directions after a nonlinear update~\cite{hertzberg2013integrating}. An alternative way to interpret the normalization is viewing $1$ as the measurement of $\mathbf q^T \!\mathbf q$, thus one more nonlinear measurement $\mathbf h\!\left(\mathbf q\right)\!=\!\mathbf q^T \!\mathbf q$ should be added to the system. The augmented measurements will then update the covariance in the Kalman filter framework. This approach is somewhat equivalent to the first one
\iffalse
(viewing $1$ as the measurement of $\mathbf q^T \mathbf q$ is equivalent to viewing $0$ as the measurement of $\mathbf q^T\delta \mathbf q$ to the first order)
\fi
and hence suffers from the same problem. 

A more elegant approach is transforming the system that operates on a manifold to its equivalent error space (i.e., local homeomorphic space) which is defined as the difference between the groundtruth state and its most recent estimate. Since this error is small when the Kalman filter converges, it can be safely parameterized by a minimal set of parameters without occurring singularity. Then a normal EKF is used to update the minimally-parameterized error state, which is finally added back to the original state on the manifold. Such an indirect way to update the state estimate has different names, such as ``error state" EKF (ESEKF) \cite{sola2017quaternion}, indirect EKF \cite{trawny2005indirect}, or multiplicative EKF \cite{markley2003attitude}. ESEKF provides an elegant way to incorporate filtering techniques into systems on manifolds, and has been widely used in a variety of robotic applications~\cite{markley2003attitude,trawny2005indirect,markley2014fundamentals, Mirzaei2008A, kelly2011visual, sola2017quaternion, Kleinert2010Errorstate, Mourikis2007ICRA,Li2013EKFbasedVIO,SLynen2013sensorfusion, huai2018robocentric, hesch2010laser}. To better describe the relation between the original state on manifolds and the error state, the $\boxplus \backslash \boxminus$ operations are introduced in~\cite{Hertzberg08aframework} and widely adopted by UKFs~\cite{hertzberg2013integrating, wuest2019online} and recently iterated Kalman filters~\cite{bloesch2017iterated, qin2020lins, xu2020fast}. The $\boxplus \backslash \boxminus$ operations have also been widely used in manifold-based optimizations \cite{kummerle2011g, ceres-solver} such as calibration \cite{wagner2011rapid}, graph-SLAM \cite{Grisetti2010ITSM} and parameter identification \cite{Michael2018parameter}. 

This paper focuses on deriving a generic and symbolic Kalman filter framework for systems naturally evolving on differentiable manifolds. We propose a canonical representation of on-manifold systems, based on which a fully iterated and symbolic Kalman fitler framework is derived. For well-studied Special Orthogonal group $SO(3)$, our work eventually leads to nearly the same Kalman filter as in~\cite{markley2003attitude,trawny2005indirect,markley2014fundamentals, Mirzaei2008A, kelly2011visual, sola2017quaternion, Kleinert2010Errorstate, Mourikis2007ICRA,Li2013EKFbasedVIO,SLynen2013sensorfusion, huai2018robocentric, hesch2010laser} for a specific system (up to the discretization accuracy), but unifies all of them into one canonical form. Moreover, our work provides a general way to integrate new manifolds that are less studied, such as the 2-sphere $\mathbb{S}^2$ for modeling the bearing vector of a visual landmark \cite{bloesch2017iterated}.

The rest of the paper is organized as follows: Sec.~\ref{boxplus_method} introduces the basic concepts of differentiable manifolds and three encapsulated operations: $\boxplus\backslash\boxminus$ and $\oplus$. Sec.~\ref{cano_repres} presents the canonical representation of on-manifold systems, based on which Sec.~\ref{esekfom} derives a fully iterated and symbolic Kalman filter. Sec.~\ref{C++_imp} implements the symbolic error-state iterated Kalman filter as a $C$++ package. Experiment results are presented in Sec.~\ref{experiments}. Finally, Sec.~\ref{conclusion} concludes this paper.

\section{Preliminaries of Differentiable Manifolds}~\label{boxplus_method}

\subsection{Differentiable manifolds}

A formal definition of differentiable manifolds involves much topological concepts \cite{carmo1992riemannian, lee2013smooth}. Informally, as shown in \cite{murray1994mathematical}, a manifold of dimension $n$ is a set $\mathcal{M}$ which is locally homeomorphic to $\mathbb{R}^n$ (called the {\it homeomorphic space}). That is, for any point $\mathbf x \in \mathcal{M}$ and an open subset $U \subset \mathcal{M}$ containing $\mathbf x$, there exists a bijective function \footnote{A bigective function is one-to-one (injective) and onto(surjective).} (called {\it homeomorphism}) $\phi$ that maps points in $U$ to an open subset of $\mathbb{R}^n$. The pair $(\phi, U)$ is called a local coordinate chart. If any two charts $(\phi, U)$ and $(\psi, V)$ sharing overlaps have their composite map $\phi \ \circ \  \psi^{-1}$ being differentiable, the manifold is said a differentiable manifold. 

\subsection{The \texorpdfstring{$\boxplus\backslash\boxminus$}{Lg} operations}\label{basic_operation}

The existence of homeomorphisms around any point in a manifold $\mathcal{M}$ enables us to encapsulate two operators $\boxplus_\mathcal{M}$ (“boxplus”) and $\boxminus_\mathcal{M}$ (“boxminus”) to the manifold~\cite{Hertzberg08aframework}:

\begin{equation}~\label{e:boxplus_def_manifold_oplus}
\setlength{\abovedisplayskip}{0.15cm} 
\setlength{\belowdisplayskip}{0.15cm} 
\begin{aligned}
    \boxplus : \mathcal{M} &\times \mathbb{R}^n \rightarrow \mathcal{M}\\
    \mathbf{x}&\boxplus_{\mathcal{M}}\mathbf{u}= \leftidx{^\mathcal{M}}\!{\bm\varphi}{_\mathbf{x}^{-1}}\left(\mathbf{u}\right) \\
    \boxminus : \mathcal{M} &\times \mathcal{M} \rightarrow \mathbb{R}^n\\
    \mathbf{y}&\boxminus_{\mathcal{M}}\mathbf{x}=\leftidx{^\mathcal{M}}\!{\bm\varphi}{_\mathbf{x}}\left(\mathbf{y}\right)
\end{aligned}
\end{equation}
where $\leftidx{^\mathcal{M}}\!{\bm\varphi}{_\mathbf{x}}$ is a homeomorphism around the point $\mathbf x \in \mathcal{M}$. It can be shown that $\mathbf{x} \boxplus_{\mathcal{M}} (\mathbf{y} \boxminus_{\mathcal{M}} \mathbf{x} ) = \mathbf{y}$ and $(\mathbf{x} \boxplus_{\mathcal{M}} \mathbf u) \boxminus_{\mathcal{M}} \mathbf{x} = \mathbf{u}, \ \forall \mathbf{x},\ \mathbf{y} \in \mathcal{M}, \mathbf u \in \mathbb{R}^n$. The physical interpretation of $\mathbf{y} = \mathbf{x}\boxplus_{\mathcal{M}}\mathbf{u}$ is adding a small perturbation $\mathbf{u} \in \mathbb{R}^n$ to $\mathbf{x}\in \mathcal{M}$, as illustrated in Fig. \ref{fig:boxplus}. And the inverse operation $\mathbf{u} = \mathbf{y}\boxminus_{\mathcal{M}}\mathbf{x}$ determines the perturbation $\mathbf{u}$ which yields $\mathbf{y}\in\mathcal{M}$ when $\boxplus_{\mathcal{M}}$-added to $\mathbf{x}$. These two operators create a local, vectorized view of the globally more complex structure of the manifold.

\begin{figure}[t]
\setlength{\abovecaptionskip}{-0.1cm} 
\setlength{\belowcaptionskip}{-0.2cm} 
\vspace{-0.3cm}
	\centering
	\includegraphics[width=0.6\columnwidth]{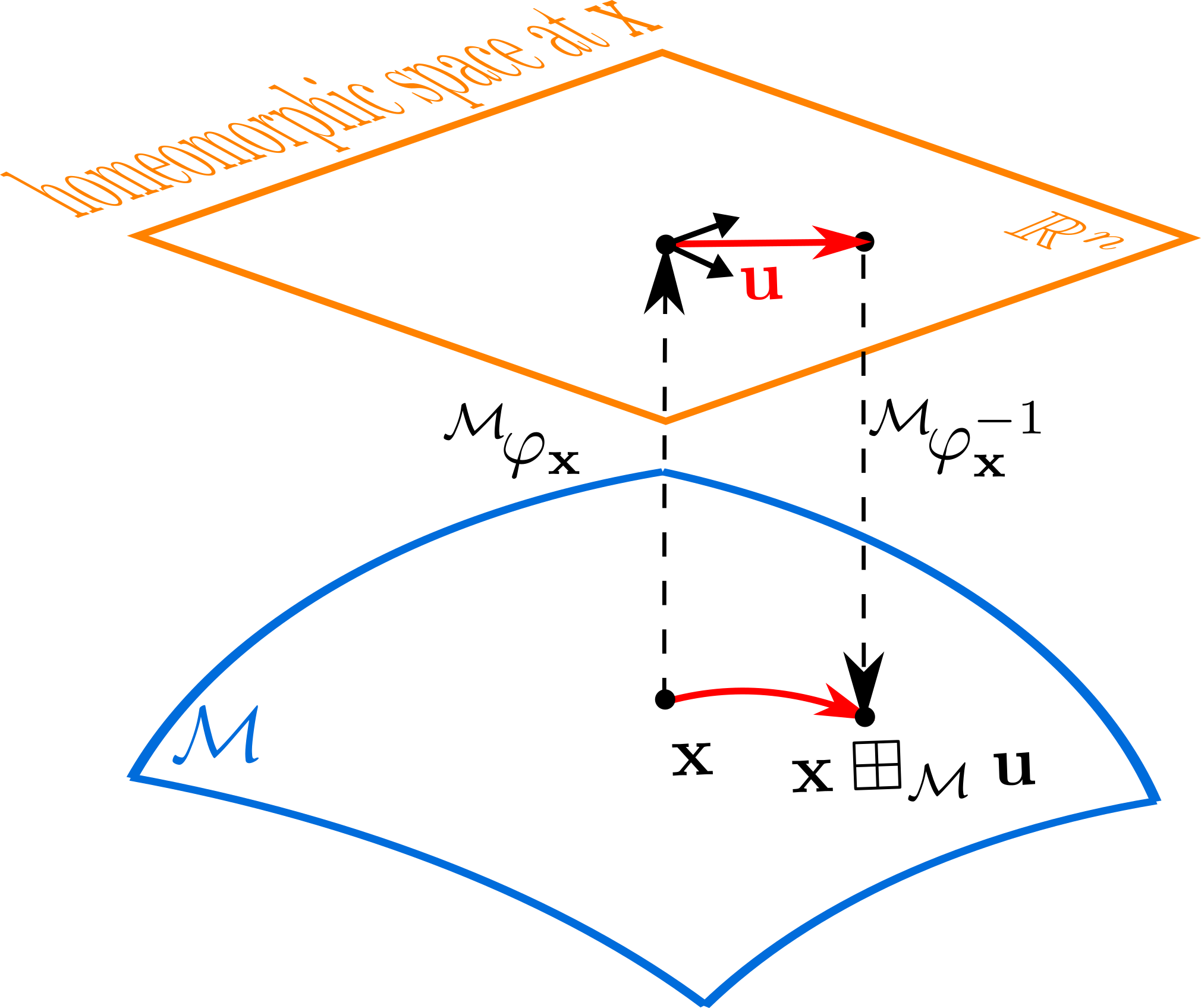}
	\caption{Illustration of the $\boxplus_{\mathcal{M}}$ operation on manifold $\mathcal{M}$, the $\mathbb{R}^n$ space is locally homeomorphic to $\mathcal{M}$ space at point $\mathbf{x}$. 
		\label{fig:boxplus}}
	\vspace{-0.3cm}
\end{figure}

As the homeomorphism is not unique for a subset of the manifold, the map $\leftidx{^\mathcal{M}}\!{\bm\varphi}{_\mathbf{x}}(\cdot)$ could also vary. Particularly, we could choose the minimal parameterization space of the tangent space as the local homeomorphic space. For a differentiable manifold, such tangent space always exists and its minimal parameterization naturally represents the perturbation. 

When $\mathcal{M}$ is a Lie group (e.g., $\mathbb{R}^n$, $SO(3)$, $SE(3)$), the tangent space possesses a Lie algebraic structure denoted as $\mathfrak{m}$ and an exponential map ${\rm exp}:\mathfrak{m} \mapsto \mathcal{M}$ \cite{murray1994mathematical}. Let $\mathfrak{f}: \mathbb{R}^n \mapsto \mathfrak{m}$ be the map from the minimal parameterization space to the Lie algebra and ${\rm Exp} = {\rm exp} \circ \mathfrak{f}$ with inverse ${\rm Log}$, the $\boxplus \backslash \boxminus$ operations are: 
\begin{equation}
\setlength{\abovedisplayskip}{0.15cm} 
\setlength{\belowdisplayskip}{0.15cm} 
\begin{aligned}\label{e:def_group}
    \mathbf{x}\boxplus_{\mathcal{M}}\mathbf{u}&=  \mathbf{x} \cdot \rm{Exp}(\mathbf{u});\quad \mathbf{y}\boxminus_{\mathcal{M}}\mathbf{x}= \rm{Log} (\mathbf{x}^{-1} \cdot \mathbf{y} )
\end{aligned}
\end{equation}
where $\cdot$ is the binary operation on $\mathcal{M}$ such that $(\mathcal{M}, \cdot)$ forms a Lie group, and $\mathbf{x}^{-1}$ is the inverse of $\mathbf{x}$ that always exists for an element on Lie groups by definition. 

When the manifold $\mathcal{M}$ is not a Lie group, there is no general guideline to find the homeomorphism between the manifold and its tangent space parameterization.  For example, the tangent space of a 2-sphere manifold $\mathbb{S}^2(r)\triangleq \{\mathbf{x}\in\mathbb{R}^3|\Vert\mathbf{x}\Vert=  r, r > 0\}$ at point $\mathbf x$ is simply the tangent plane at $\mathbf x$ (see Fig.~\ref{fig:S2_boxplus}). For a point $\mathbf x \in \mathbb{S}^2(r)$, the perturbation can be achieved by rotating along a vector in the tangent plane, the result would still remain on $\mathbb{S}^2({r})$ as required. The rotation vector in the tangent plane is minimally parameterized by $\mathbf u \in \mathbb{R}^2$ under two basis vectors $\mathbf{b}_1, \mathbf{b}_2 \in \mathbb{R}^3$ spanning the tangent plane. That is,
\begin{equation}
\setlength{\abovedisplayskip}{0.15cm} 
\setlength{\belowdisplayskip}{0.15cm} 
    \label{e:xboxplusu_S2}
    \mathbf{x}\boxplus\mathbf{u} \triangleq \mathbf{R}(\mathbf{B}(\mathbf{x})\cdot\mathbf{u})\cdot\mathbf{x}; \ \mathbf{B}(\mathbf{x}) = \begin{bmatrix} \mathbf b_1 & \mathbf b_2 \end{bmatrix} \in \mathbb{R}^{3 \times 2}
\end{equation}
where $\mathbf{R}(\mathbf w) = {\rm Exp}(\mathbf w) \in SO(3)$ denotes a rotation about an axis-angle represented by the vector $\mathbf w \in \mathbb{R}^3$, and the dependence on $\mathbf x$ of the two basis vectors $\mathbf{b}_1, \mathbf{b}_2$ are made explicit in $\mathbf B(\mathbf x)$. The choice of $\mathbf{b}_1, \mathbf{b}_2$ is not unique as long as they are orthonormal and both perpendicular to $\mathbf x$.  
 
\begin{figure}[t]
\vspace{-0.3cm}
\setlength{\abovecaptionskip}{-0.1cm} 
\setlength{\belowcaptionskip}{-0.2cm} 
	\centering
	\includegraphics[width=0.7\columnwidth]{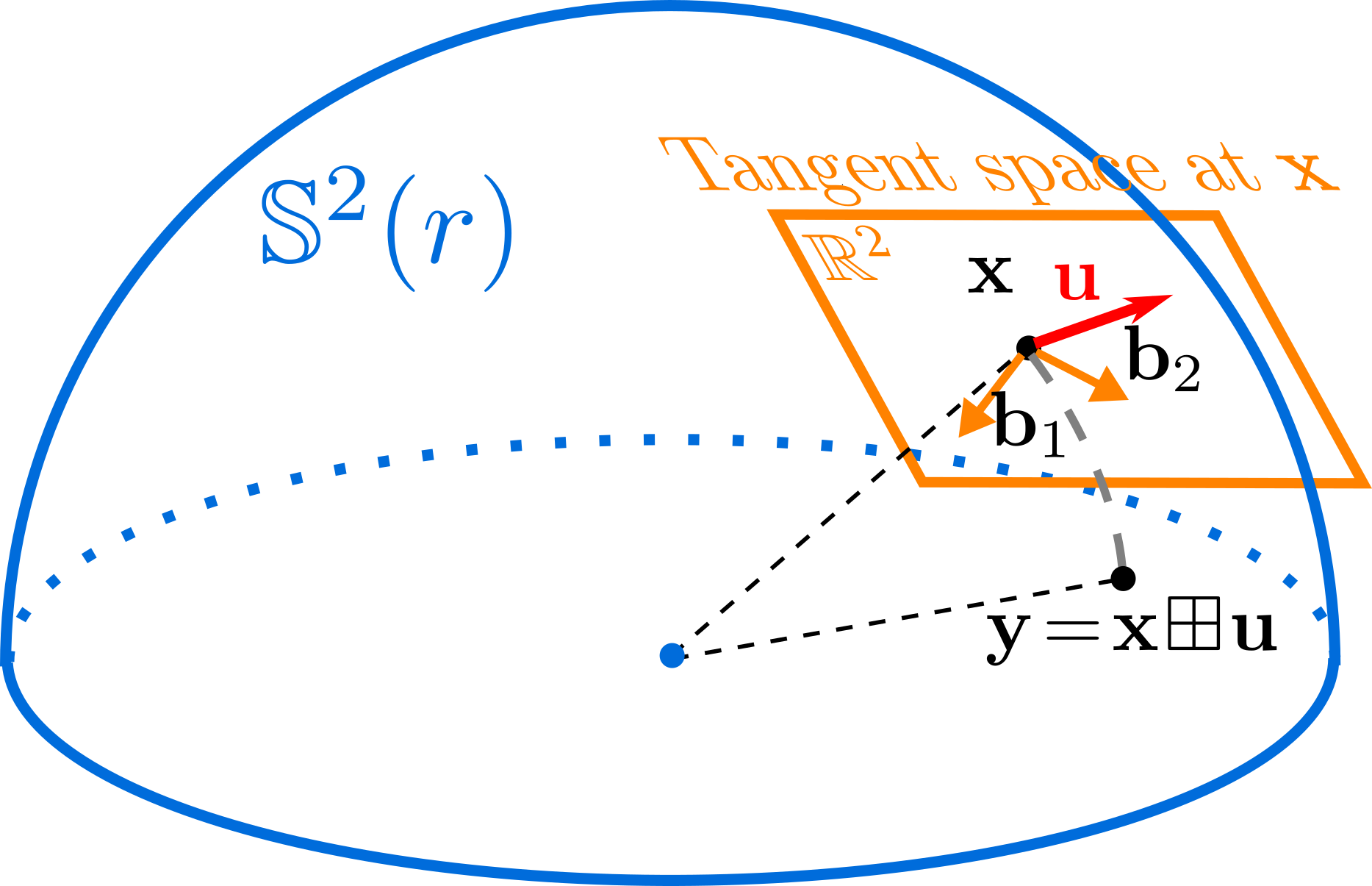}
	\caption{Illustration of the $\boxplus$ operation on the $\mathbb{S}^2(r)$ manifold. 
		\label{fig:S2_boxplus}}
	\vspace{-0.3cm}
\end{figure}

\subsection{The \texorpdfstring{$\oplus$}{Lg} operation}\label{basic_operation_oplus}

A real-world system is usually driven by some exogenous input. To model this phenomenon, besides describing the manifold itself where the state lies on, an additional operation is needed to describe how the state on the manifold is driven by a constant exogenous velocity over an infinitesimal time period. Although the effect of the velocity on the state is adding a perturbation to its original location on the manifold which are well described by the $\boxplus$ operator, the exogenous velocity is not necessarily in the same homeomorphic space (i.e., the tangent space) defining the $\boxplus$ operation, hence a new operation denoted as $\oplus_{\mathcal{M}}$ is needed. Assume the exogenous velocity is of dimension $l$, the new operation is then a map $\oplus_{\mathcal{M}}:\mathcal{M}\times\mathbb{R}^l\mapsto\mathcal{M}$. 

In particular, when $\mathcal{M}$ is a Lie group, the exogenous velocity typically lies in the tangent plane that also defines $\boxplus \backslash \boxminus$ operations as in (\ref{e:def_group}), hence the operation $\oplus_{\mathcal{M}}$ coincides with the $\boxplus_{\mathcal{M}}$, i.e.,
\begin{equation}
\setlength{\abovedisplayskip}{0.15cm} 
\setlength{\belowdisplayskip}{0.15cm} 
\begin{aligned}
    \mathbf x \oplus_{\mathcal{M}} \mathbf v &= \mathbf x \boxplus_{\mathcal{M}} \mathbf v =  \mathbf{x} \cdot \rm{Exp}(\mathbf{v})  \quad (i.e., \text{$l = n$})
\end{aligned}
\end{equation}

Otherwise, the definition of $\oplus_{\mathcal{M}}$ should be made according to the specific manifold. For example, the element of the 2-sphere manifold $\mathbb{S}^2(r)$ space is a vector of fixed length $r$, whose transformation under a perturbation is a rotation around an axis-angle that may not lie on the homeomorphic space of $\mathbb{S}^2(r)$. Therefore, we can define $\mathbf{x}\oplus_{\mathcal{M}}\mathbf{v}=\mathbf{R}(\mathbf{v})\mathbf{x}$. Take a more concrete example, in practice, $\mathbb{S}^2(r)$ is usually used to describe a constant vector of fixed length ${r}$. When the vector is represented in a fixed world frame, the state satisfies $\dot{\mathbf x} = \mathbf 0$ (i.e., zero velocity), while when the vector is represented in the body frame, it satisfies $\dot{\mathbf x} = \boldsymbol{\omega} \times \mathbf x$, where $\boldsymbol{\omega}$ is the angular velocity. In either case, the velocity is the complete $\mathbb{R}^3$, while the homeomorphic space of $\boxplus$ is the tangent plane at $\mathbf x$ (see Fig. \ref{fig:S2_boxplus}). Moreover, we can see that in either case, the state is updated as $\mathbf R (\mathbf v) \mathbf x$, where $\mathbf v $ is $\mathbf 0$ or $\boldsymbol{\omega} dt $, respectively.

For the sake of notation simplicity, in the following discussion, we drop the subscript $\mathcal{M}$ in operations $\boxplus$, $\boxminus$ and $\oplus$ when no ambiguity exists. 

%\subsection{Differentiable Manifold}

\subsection{Differentiations}
\iffalse
A differentiable manifold is a topological manifold with a globally defined differential structure.
\fi
In the Kalman filter that will be derived later in Sec.~\ref{esekfom}, the partial differentiation of $\left(\left(\left(\mathbf{x}\boxplus\mathbf{u}\right)\oplus \mathbf{v}\right)\boxminus\mathbf{y}\right)$ with respect to $\mathbf{u}$ and $\mathbf{v}$ will be used, where $\mathbf{x}, \mathbf{y}\in\mathcal{M}, \mathbf{u} \in\mathbb{R}^n$ and $\mathbf{v}\in\mathbb{R}^l$. This can be obtained easily from the chain rule as follows:
\begin{equation}
\setlength{\abovedisplayskip}{0.15cm} 
\setlength{\belowdisplayskip}{0.15cm} 
\begin{aligned}
    &\frac{\partial\left(\left(\left(\mathbf{x}\boxplus\mathbf{u}\right)\oplus\mathbf{v}\right)\boxminus\mathbf{y}\right)}{\partial \mathbf{u}}=\frac{\partial\leftidx{^\mathcal{M}}\!{\bm\varphi}{_{\mathbf{y}}}\left(\mathbf{z}\right)}{\partial\mathbf{z}}|_{\begin{subarray}{l}\mathbf{z}=\left(\mathbf{x}\boxplus\mathbf{u}\right)\oplus \mathbf{v}\end{subarray}} \\
    & \quad \quad \quad \quad \quad \quad \quad \quad \cdot \frac{\partial\left(\mathbf z \oplus \mathbf{v}\right)}{\partial\mathbf{z}} |_{\begin{subarray}{l}\mathbf{z}=\mathbf{x}\boxplus\mathbf{u}\end{subarray}} \cdot \frac{\partial\leftidx{^\mathcal{M}}\!{\bm\varphi}{_{\mathbf{x}}^{-1}}\left(\mathbf{z}\right)}{\partial\mathbf{z}} |_{\begin{subarray}{l}\mathbf{z}=\mathbf{u}\end{subarray}}\\
    &\frac{\partial\left(\left(\left(\mathbf{x}\boxplus\mathbf{u}\right)\oplus\mathbf{v}\right)\boxminus\mathbf{y}\right)}{\partial \mathbf{v}}=\frac{\partial\leftidx{^\mathcal{M}}\!{\bm\varphi}{_{\mathbf{y}}}\left(\mathbf{z}\right)}{\partial\mathbf{z}}|_{\begin{subarray}{l}\mathbf{z}=\left(\mathbf{x}\boxplus\mathbf{u}\right)\oplus \mathbf{v}\end{subarray}}\cdot \frac{\partial\left(\left(\mathbf x\boxplus\mathbf{u}\right) \oplus \mathbf{z}\right)}{\partial\mathbf{z}} |_{\begin{subarray}{l}\mathbf{z}=\mathbf{v}\end{subarray}}
\end{aligned}
\end{equation}

For certain manifolds (e.g., $SO(3)$), it is usually more convenient to compute the differentiations $\frac{\partial\left(\left(\left(\mathbf{x}\boxplus\mathbf{u}\right)\oplus \mathbf{v}\right)\boxminus\mathbf{y}\right)}{\partial \mathbf{u}}$ and $\frac{\partial\left(\left(\left(\mathbf{x}\boxplus\mathbf{u}\right)\oplus \mathbf{v}\right)\boxminus\mathbf{y}\right)}{\partial \mathbf{v}}$ directly instead of using the above chain rule.

\subsection{{\it Compound differentiable manifolds}}

Based on the principles of Cartesian product of manifolds, the $\boxplus\backslash\boxminus$ and $\oplus$ on a {\it compound manifold} of two (and by induction arbitrary numbers of) sub-manifolds are defined as:
\begin{equation}\label{e:compound_plus_minus}
\setlength{\abovedisplayskip}{0.15cm} 
\setlength{\belowdisplayskip}{0.15cm} 
    \begin{aligned}
    \underbrace{\begin{bmatrix}
    \mathbf{x}_1\\\mathbf{x}_2
    \end{bmatrix}}_{\mathbf{x}} \! \boxplus \! \underbrace{\begin{bmatrix}
    \mathbf{u}_1\\\mathbf{u}_2
    \end{bmatrix}}_{\mathbf{u}} \!=\! \begin{bmatrix}
    \mathbf{x}_1 \! \boxplus \!  \mathbf{u}_1\\
    \mathbf{x}_2 \! \boxplus \! \mathbf{u}_2
    \end{bmatrix}, 
    \underbrace{\begin{bmatrix}
    \mathbf{x}_1\\\mathbf{x}_2
    \end{bmatrix}}_{\mathbf{x}} \! \oplus \! \underbrace{\begin{bmatrix}
    \mathbf{v}_1\\\mathbf{v}_2
    \end{bmatrix}}_{\mathbf{v}} \!=\! \begin{bmatrix}
    \mathbf{x}_1 \! \oplus \!  \mathbf{v}_1\\
    \mathbf{x}_2 \! \oplus \! \mathbf{v}_2
    \end{bmatrix}.
    \end{aligned}
\end{equation}

The partial differentiation on the {\it compound manifold} therefore satisfies (see Lemma.~\ref{lemma_comp_mani_diff} in Appx.~\ref{app:partial_diff_Com} for proof):

\begin{equation}\label{diff_compound}
\setlength{\abovedisplayskip}{0.15cm} 
\setlength{\belowdisplayskip}{0.15cm} 
    \begin{aligned}
    \hspace{-0.2cm}\frac{\partial \left( \! \left( \! \left( \mathbf{x}\boxplus\mathbf{u} \! \right)\oplus \mathbf{v} \! \right)\boxminus\mathbf{y} \! \right)}{\partial {\mathbf{u}} }\!\!&=\!\!\begin{bmatrix}
    \frac{\partial\left( \! \left( \! \left( \mathbf{x}_1\boxplus\mathbf{u}_1 \! \right)\oplus \mathbf{v}_1 \! \right)\boxminus\mathbf{y}_1 \! \right)}{\partial {\mathbf{u}}_1} &\!\! \mathbf{0}\\\mathbf{0} &\!\!\frac{\partial\left( \! \left( \! \left( \mathbf{x}_2 \boxplus\mathbf{u}_2 \! \right)\oplus \mathbf{v}_2 \! \right)\boxminus\mathbf{y}_2 \! \right)}{\partial {\mathbf{u}}_2}\end{bmatrix}\\
    \hspace{-0.2cm}\frac{\partial \left( \! \left( \! \left( \mathbf{x}\boxplus\mathbf{u} \! \right)\oplus \mathbf{v} \! \right)\boxminus\mathbf{y} \! \right)}{\partial {\mathbf{v}} }\!\!&=\!\!\begin{bmatrix}
    \frac{\partial\left( \! \left( \! \left( \mathbf{x}_1\boxplus\mathbf{u}_1 \! \right)\oplus \mathbf{v}_1 \! \right)\boxminus\mathbf{y}_1 \! \right)}{\partial {\mathbf{v}}_1} &\!\! \mathbf{0}\\\mathbf{0} &\!\!\frac{\partial\left( \! \left( \! \left( \mathbf{x}_2 \boxplus\mathbf{u}_2 \! \right)\oplus \mathbf{v}_2 \! \right)\boxminus\mathbf{y}_2 \! \right)}{\partial {\mathbf{v}}_2}\end{bmatrix}
    \end{aligned}
\end{equation}

The $\boxplus\backslash\boxminus$ and $\oplus$ operations and their partial differentiation  on a {\it compound manifold} are extremely useful, enabling us to define the $\boxplus\backslash\boxminus$ and $\oplus$ operations and their derivatives for {\it primitive manifolds} (e.g., $\mathbb{R}^n, SO(3), \mathbb{S}^2(r)$) only and then extend these definitions to more complicated {\it compound manifolds}. 

\begin{table*}[b]
\vspace{-0.3cm}
\setlength{\abovecaptionskip}{-0cm} 
\setlength{\belowcaptionskip}{-0.2cm} 
\centering
\caption{Operation and differentiation of important manifolds in practice}\label{tab:imp_mani}
\begin{threeparttable}
    \begin{tabular}{llllll}
    \toprule
    $\mathcal{M}$&$\mathbf{x}\boxplus\mathbf{u}$&$\mathbf{y}\boxminus\mathbf{x}$&$\mathbf{x}\oplus\mathbf{v}$&$\frac{\partial\left(\left(\left(\mathbf{x}\boxplus\mathbf{u}\right)\oplus \mathbf{v}\right)\boxminus\mathbf{y}\right)}{\partial \mathbf{u}}$ &$\frac{\partial\left(\left(\left(\mathbf{x}\boxplus\mathbf{u}\right)\oplus \mathbf{v}\right)\boxminus\mathbf{y}\right)}{\partial \mathbf{v}}$ \\
    \midrule
    $\mathbb{R}^n$&$\mathbf{x}+\mathbf{u}$&$\mathbf{y}-\mathbf{x}$&$\mathbf{x}+\mathbf{v}$&$\mathbf{I}_{n\times n}$&$\mathbf{I}_{n\times n}$\\ 
    $SO(3)$&$\mathbf{x}\!\cdot\!{\rm Exp}(\mathbf{u})$&${\rm Log}(\mathbf{x}^{-1}\cdot\mathbf{y})$&$\mathbf{x}\!\cdot\!{\rm Exp}(\mathbf{v})$&$\mathbf{A}(((\mathbf{x}\!\boxplus\!\mathbf{u})\!\oplus \!\mathbf{v})\!\boxminus\!\mathbf{y})^{-\!T}{\rm Exp}(-\mathbf{v})\mathbf{A}(\mathbf{u})^T$&$\mathbf{A}(((\mathbf{x}\!\boxplus\!\mathbf{u})\!\oplus \!\mathbf{v})\!\boxminus\!\mathbf{y})^{-\!T}\!\mathbf{A}(\mathbf{v})^T$\\
    $\mathbb{S}^2(r)$&$\mathbf{R}(\mathbf{B}(\mathbf{x})\mathbf{u})\mathbf{x}$&$\mathbf{B}(\mathbf{x})^T\! \left(
    \theta \frac{\lfloor\mathbf{x}\rfloor\mathbf{y}}{\|\lfloor\mathbf{x}\rfloor\mathbf{y}\|}\! \right)$&$ \mathbf{R}(\mathbf{v})\mathbf{x}$&$\mathbf{N} ((\mathbf{x}\! \boxplus \! \mathbf{u}) \! \oplus \! \mathbf{v},\mathbf{y} ) \mathbf{R}(\mathbf{v})\mathbf{M}(\mathbf{x},\mathbf{u})$&$-\mathbf{N} ((\mathbf{x}\! \boxplus \! \mathbf{u}) \! \oplus \! \mathbf{v},\mathbf{y}) \mathbf{R}(\mathbf{v}) \lfloor\mathbf{x}\!\boxplus\!\mathbf{u}\rfloor \mathbf{A}(\mathbf{v})^T$\\
    \bottomrule
    \end{tabular}
    \end{threeparttable}
    \vspace{-0.3cm}
\end{table*}
For example, the definitions of $\boxplus\backslash\boxminus$ and $\oplus$ operations for a few important manifolds, including $\mathbb{R}^n, SO(3)$ and $\mathbb{S}^2(r)$, according to the previous discussions and their partial differentiations are summarized in Tab.~\ref{tab:imp_mani}, where
\begin{equation}~\label{eq:A}
\setlength{\abovedisplayskip}{0.15cm} 
\setlength{\belowdisplayskip}{0.15cm} 
\begin{aligned}
\mathbf{A}\!\left(\mathbf{u}\right)\!&=\!\mathbf{I}\!+\!\left(\frac{1\!-\!{\rm cos}\left(\Vert\mathbf{u}\Vert\right)}{\Vert\mathbf{u}\Vert}\right)\!\frac{\lfloor\mathbf{u}\rfloor}{\Vert\mathbf{u}\Vert}\!+\!\left(1\!-\!\frac{{\rm sin}\left(\Vert\mathbf{u}\Vert\right)}{\Vert\mathbf{u}\Vert}\right)\!\frac{\lfloor\mathbf{u}\rfloor^2}{\Vert\mathbf{u}\Vert^2}\\
\mathbf{N}\!\left(\mathbf{x},\!\mathbf{y}\right)&\!=\!\mathbf{B}\!\left(\mathbf{y}\right)^{T}\!\!\left(\frac{\theta}{\| \! \lfloor\mathbf{y} \! \rfloor\mathbf{x}\|}\lfloor  \mathbf{y}  \rfloor\!+\!\lfloor  \mathbf{y}  \rfloor \cdot \mathbf{x}\!\cdot\!\mathbf{P}\!\left(\! \mathbf{x}, \! \mathbf{y}\right)\right)\\
     \mathbf{M}\!\left(\mathbf{x},\!\mathbf{u}\right)&\! =\! - \mathbf R (\mathbf B ( \mathbf x  ) \mathbf u) \cdot \lfloor \mathbf x  \rfloor  \cdot \mathbf A \! \left(\mathbf B( \mathbf x  ) \mathbf u \right)^{T} \cdot \mathbf B( \mathbf x  )
\end{aligned}
\end{equation} 
where $\lfloor\mathbf{u}\rfloor$ denotes the skew-symmetric matrix that maps the cross product of $\mathbf{u}\in\mathbb{R}^3$ and
\begin{equation}
\setlength{\abovedisplayskip}{0.15cm} 
\setlength{\belowdisplayskip}{0.15cm} 
    \label{eq:S2}
    \begin{aligned}
      \alpha\left(\|\mathbf{u}\|\right)&=\frac{\|\mathbf{u}\|}{2}{\rm cot}\left(\frac{\|\mathbf{u}\|}{2}\right)=\frac{\|\mathbf{u}\|}{2}\frac{{\rm cos}\left({\|\mathbf{u}\|}/2\right)}{{\rm sin}\left(\|\mathbf{u}\|/2\right)}\\
\mathbf{P}\!\left(\mathbf{x},\mathbf{y}\right) &\!=\!\frac{1}{{r}^4}\left(\frac{-\mathbf{y}^T\mathbf{x}\|\lfloor\mathbf{y}\rfloor\mathbf{x}\|+{r}^4\theta}{\|\lfloor\mathbf{y}\rfloor\mathbf{x}\|^3}\mathbf x^T \lfloor \mathbf y \rfloor^2 \!-\!\mathbf{y}^T \right).
\end{aligned}
\end{equation}

Detailed derivations of the results in Tab.~\ref{tab:imp_mani} can be seen in Appx.~\ref{app:important_manifolds}.

For a new primitive manifold, we can define its respective operations and derive the associate partial differentiation similar to the procedure above. Notice that these procedures need only to be performed once for a certain primitive manifold and they do not depend on the particular system evolving on the manifold.  

\section{Canonical Representation of On-manifold Systems}~\label{cano_repres}

Consider a robotic system in discrete time with sampling period $\Delta t$. Using a zero holder discretization where the inputs (and hence the first order derivative of the state) are constant during one sampling period, we can cast it into the following canonical form:
\begin{equation}\label{cano}
\setlength{\abovedisplayskip}{0.15cm} 
\setlength{\belowdisplayskip}{0.15cm} 
\begin{aligned}
    \mathbf{x}_{k+1}&=\mathbf{x}_k\oplus_{\mathcal{M}_s}\left(\Delta t\mathbf{f}\left(\mathbf{x}_k, \mathbf{u}_k, \mathbf{w}_{k}\right)\right), \mathbf{x}_k \in\mathcal{M}_s, \\
    \mathbf{z}_k &= \mathbf{h}\left(\mathbf{x}_k, \mathbf v_k \right), \mathbf{z}_k\in\mathcal{M}_m,\\
    \mathbf{w}_k&\sim\mathcal{N}\left(\mathbf{0},  \mathcal{Q}_k\right), \mathbf{v}_k\sim\mathcal{N}\left(\mathbf{0},  \mathcal{R}_k\right).
    \end{aligned}
\end{equation}
where the state $\mathbf{x}_k$ is assumed to be on the manifold $\mathcal{M}_s$ of dimension $n$ and the measurement $\mathbf{z}_k$ is assumed to be on the manifold $\mathcal{M}_m$ of dimension $m$. When compared to higher-order discretization methods (e.g., Runge-Kutta integration) used in prior work \cite{Mourikis2007ICRA, huai2018robocentric}, the zero-order hold discretization is usually less accurate. However, such difference is negligible when the sampling period is small.

In the following, we show how to cast different state components into the canonical form in (\ref{cano}). Then with the composition property (\ref{e:compound_plus_minus}), the complete state equation can be obtained by concatenating all components.

{\it Example 1: } Vectors in Euclidean space (e.g., position and velocity). Assume $\mathbf x \in \mathbb{R}^n$ subject to $\dot{\mathbf x} = \mathbf f(\mathbf x, \mathbf u, \mathbf w)$. Using zero-order hold discretization, $\mathbf f(\mathbf x, \mathbf u, \mathbf w)$ is assumed constant during the sampling period $\Delta t$, hence 
\begin{equation}
    \begin{aligned}
     \mathbf{x}_{k+1} &=\mathbf{x}_k + \left(\Delta t\mathbf{f}\left(\mathbf{x}_k, \mathbf{u}_k, \mathbf{w}_{k}\right)\right) \\
     &=\mathbf{x}_k\oplus_{\mathbb{R}^n}\left(\Delta t\mathbf{f}\left(\mathbf{x}_k, \mathbf{u}_k, \mathbf{w}_{k}\right)\right).
    \end{aligned}
\end{equation}

{\it Example 2:} Attitude kinematics in a global reference frame (e.g., the earth-frame). Let $\mathbf x \in SO(3)$ be the body attitude relative to the global frame and $^G\boldsymbol{\omega}$ be the global angular velocity which holds constant for one sampling period $\Delta t$, then
\begin{equation}
    \begin{aligned}
     & \dot{\mathbf x} \!=\! \lfloor ^G\boldsymbol{\omega} \rfloor \cdot \mathbf x \! \! \implies \! \! \mathbf x_{k+1} \!=\! {\rm Exp}(\Delta t \ ^G\boldsymbol{\omega}_k) \cdot \mathbf x_{k} = \mathbf x_{k} \\
     & \cdot {\rm Exp}\left(\Delta t (\mathbf x_k^T \cdot ^G\!\boldsymbol{\omega}_k ) \right) \!= \! \mathbf{x}_k \oplus_{SO(3)} \left(\Delta t\mathbf{f}\left(\mathbf x_k,  ^G\!\boldsymbol{\omega}_k \right)\right), \\ 
     & \quad \quad \quad \quad \quad \quad \quad \quad \quad \quad \quad \quad \mathbf{f}\left(\mathbf x_k, ^G\!\boldsymbol{\omega}_k \right) = \mathbf x_k^T \cdot ^G\!\boldsymbol{\omega}_k.
    \end{aligned}
\end{equation}

{\it Example 3:} Attitude kinematics in body frame. Let $\mathbf x \in SO(3)$ be the body attitude relative to the global frame and $^B\boldsymbol{\omega}$ be the body angular velocity which holds constant for one sampling period $\Delta t$, then
\begin{equation}
    \begin{aligned}
     \dot{\mathbf x} &= \mathbf x \cdot \lfloor ^B\boldsymbol{\omega} \rfloor \! \implies \! \mathbf x_{k+1} = \mathbf x_{k} \cdot {\rm Exp}(\Delta t \ ^B\boldsymbol{\omega}_k) \\
     & = \mathbf{x}_k\oplus_{SO(3)}\left(\Delta t\mathbf{f}\left( ^B\boldsymbol{\omega}_k \right)\right), \ \mathbf{f}\left( ^B\boldsymbol{\omega}_k \right) =\  ^B\boldsymbol{\omega}_k.
    \end{aligned}
\end{equation}

{\it Example 4:} Vectors of known magnitude (e.g., gravity) in the global frame. Let $\mathbf x \in \mathbb{S}^2(g)$ be the gravity vector in the global frame with known magnitude $g$. Then,
\begin{equation}
    \begin{aligned}
     \dot{\mathbf x} \! &= \! \mathbf 0 \!\! \implies \!\! \mathbf x_{k+1} \! = \! \mathbf x_{k}  \! = \! \mathbf{x}_k \! \oplus_{\mathbb{S}^2 \! (g)}\! \! \left(\Delta t\mathbf{f}\!\left(\mathbf x_k \right) \! \right), \mathbf{f}\!\left(\mathbf x_k \right) \!=\! \mathbf 0.
    \end{aligned}
\end{equation}

{\it Example 5:} Vectors of known magnitude (e.g., gravity) in body frame. Let $\mathbf x \in \mathbb{S}^2(g)$ be the gravity vector in the body frame and $^B\boldsymbol{\omega}$ be the body angular velocity which holds constant for one sampling period $\Delta t$. Then,
\begin{equation}
    \begin{aligned}
     \dot{\mathbf x} \! &= \! - \lfloor ^B\boldsymbol{\omega} \rfloor \mathbf x \!\! \implies \!\! \mathbf x_{k+1} \! = \! {\rm Exp}(- \Delta t\  ^B\boldsymbol{\omega}_k) \mathbf x_{k}  \\
     \! &= \! \mathbf{x}_k\oplus_{\mathbb{S}^2(g)}\left(\Delta t\mathbf{f}\left(^B\boldsymbol{\omega}_k \right)\right), \mathbf{f}\left(^B\boldsymbol{\omega}_k \right) \!=\! - ^B\boldsymbol{\omega}_k.
    \end{aligned}
\end{equation}

{\it Example 6:} Bearing-distance parameterization of visual landmarks \cite{bloesch2017iteratedsupp}. Let $\mathbf x \in \mathbb{S}^2(1)$ and $d(\rho) \in \mathbb{R}$ be the bearing vector and depth (with parameter $\rho$), respectively, of a visual landmark, and $^G\mathbf R_C, ^G\mathbf p_C$ be the attitude and position of the camera. Then the visual landmark in the global frame is $^G\mathbf R_C (\mathbf x d(\rho)) + ^G\!\mathbf p_C$, which is constant over time:
\begin{equation}
    \label{bearing-depth_sup}
    \begin{aligned}
     &\frac{d (^G\mathbf R_C (\mathbf x d(\rho)) + ^G\mathbf p_C)}{dt} \!=\! \mathbf 0  \implies \\
     & \lfloor ^C\boldsymbol{\omega} \rfloor (\mathbf x d(\rho)) + \dot{\mathbf x} d(\rho) + {\mathbf x} d'(\rho) \dot{\rho} + ^C\!\mathbf v = \mathbf 0.  
    \end{aligned}
\end{equation}

Left multiplying (\ref{bearing-depth_sup}) by $\mathbf x^T$ and using $\mathbf x^T \dot{\mathbf x} = 0$ yield $\dot{\rho} = - \mathbf x^T \cdot ^C\!\!\mathbf v / d'(\rho)$. Substituting this to (\ref{bearing-depth_sup}) leads to
\begin{equation}
    \begin{aligned}
     & \dot{\mathbf x} \! = \! - \lfloor  ^C\!\boldsymbol{\omega}  +  \frac{1}{d(\rho)} \lfloor \mathbf x \rfloor \cdot ^C\!\mathbf v   \rfloor \cdot \mathbf x \implies \\
     & \mathbf x_{k+1} \!=\! {\rm Exp}\left(- \Delta t \left(  ^C\boldsymbol{\omega}_k  +  \frac{1}{d(\rho)} \lfloor \mathbf x_k \rfloor \cdot ^C\!\mathbf v_k   \right) \right) \mathbf x_k \\
     & \quad \quad  = \! \mathbf{x}_k\oplus_{\mathbb{S}^2(1)} \! \left(\Delta t\mathbf{f}\left(\mathbf x_k, ^C\!\boldsymbol{\omega}_k, ^C\!\mathbf v_k \right)\right), \\
     & \quad \quad \ \ \mathbf{f}\left(\mathbf x_k, ^C\!\boldsymbol{\omega}_k, ^C\!\mathbf v_k \right) \!=\! - ^C\!\boldsymbol{\omega}_k \! - \!  \frac{1}{d(\rho)} \lfloor \mathbf x_k \rfloor \cdot ^C\!\mathbf v_k. 
    \end{aligned}
\end{equation}
where $^C\!\boldsymbol{\omega}  +  \frac{1}{d(\rho)} \lfloor \mathbf x \rfloor \cdot ^C\!\mathbf v$ is assumed constant for one sampling period $\Delta t$ due to the zero-order hold assumption.

\section{Kalman Filters on differentiable Manifolds}~\label{esekfom}
In this chapter, we derive a symbolic Kalman filter based on the canonical system representation (\ref{cano}). To avoid singularity of the minimal parameterization of the system original state which lies on manifolds, we employ the error-state idea that has been previously studied in prior work such as~\cite{sola2017quaternion} and~\cite{Lu2019IMU}. The presented derivation is very abstract, although being more concise, compact and generic. Moreover, for a complete treatment, we derive the full multi-rate iterated Kalman filter. Readers may refer to~\cite{sola2017quaternion} for more detailed derivations/explanations or~\cite{Lu2019IMU} for a brief derivation on a concrete example. We contribute to present an easy way to deploy the iterated Extended Kalman filter on arbitrary robotic systems with states on differentiable manifolds.

In the following presentations, we use the below notations:
\begin{enumerate}[label=(\roman*)]
    \item $\mathcal{M}_s$ denotes the manifold that the state $\mathbf{x}$ lies on. And $\mathcal{M}_m$ denotes the manifold that the measurement $\mathbf{z}$ lies on. For sake of notation simplification, we drop the subscripts $\mathcal{M}_s, \mathcal{M}_m$ for $\boxplus\backslash\boxminus$ and $\oplus$ when the context is made clear.
    \item The subscript $k$ denotes the time index, e.g., $\mathbf{x}_k$ is the ground truth of the state $\mathbf x$ at step $k$.
    \item The subscript $\tau|k$ denotes the estimation of a quantity at step $\tau$ based on all the measurements up to step $k$, e.g., $\mathbf{x}_{\tau|k}$ means the estimation of state $\mathbf{x}_{\tau}$ based on measurements up to step $k$. For filtering problem, it requires $\tau \geq k$. More specifically, we have $\tau > k$ for state predict (i.e., prior estimate) and $\tau = k$ for state update (i.e., posteriori estimate). 
    \item $\delta \mathbf{x}_{\tau|k} = \mathbf{x}_{\tau} \boxminus\mathbf{x}_{\tau|k}$ denotes the estimation error in the local homeomorphic linear space of $\mathbf{x}_{\tau|k}$. It is a random vector in $\mathbb{R}^n$ since the ground true state $\mathbf x$ is random. 
    \item $\mathbf{P}_{\tau|k}$ denotes the covariance of the estimation error $\delta\mathbf{x}_{\tau|k}$.
    \item superscript $j$ denotes the $j$-th iteration of the iterated Kalman filter, e.g. ${\mathbf{x}}_{k|k}^j$ denotes the estimate of state $\mathbf{x}_{k}$ at the $j$-th iteration based on measurements up to step $k$.
\end{enumerate}
\begin{figure}[t]
\vspace{-0.3cm}
\setlength{\abovecaptionskip}{-0.1cm} 
\setlength{\belowcaptionskip}{-0.2cm} 
\centering
        \includegraphics[width=0.6\columnwidth]{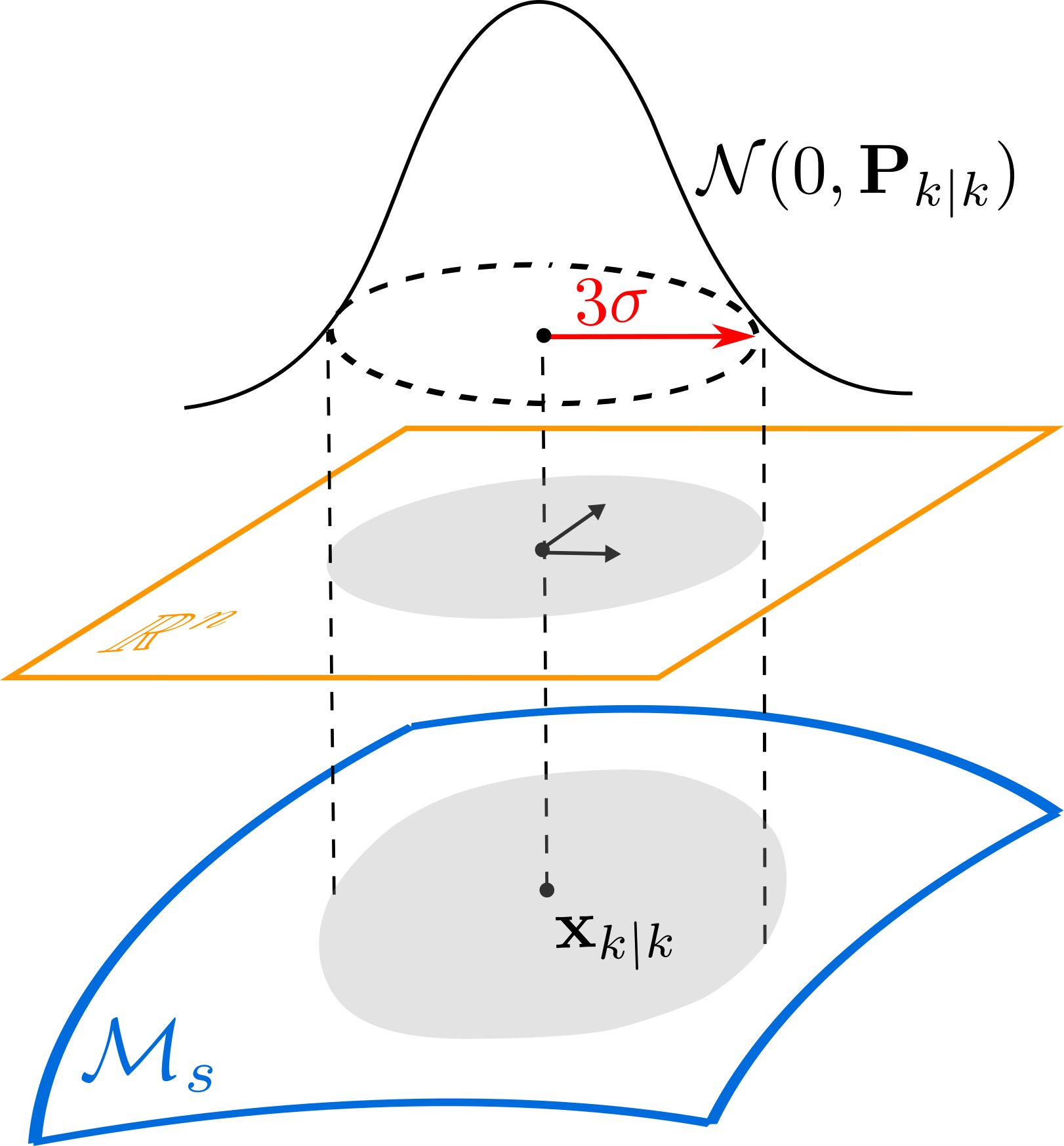}
\centering
\caption{The covariance matrix ($\mathbf{P}_{k|k}$) of the error state ($\delta\mathbf{x}_{k|k})$. The yellow $\mathbb{R}^n$ space is locally homeomorphic to $\mathcal{M}_s$.}
\vspace{-0.3cm}
\label{fig:cov_relation}
\end{figure}
\subsection{Initialization}\label{sec:initialization}
Assume we have received measurements up to step $k$ and updated the state at that time step as $\mathbf{x}_{k|k}$ along with the updated covariance matrix $\mathbf{P}_{k|k}$. According to the notation conventions above, $\mathbf{P}_{k|k}$ denotes the covariance of $\delta \mathbf{x}_{k|k}$, an error in the local homeomorphic space of the state update $\mathbf{x}_{k|k}$. The relation between $\delta \mathbf{x}_{k|k}$ and $\mathbf{P}_{k|k}$ is shown in Fig.~\ref{fig:cov_relation}.
\subsection{State predict}~\label{sec:esekf_state_propagation}
The state predict from step $k$ follows directly from the system model in equation (\ref{cano}) by setting $\mathbf{w} = \mathbf{0}$: 
\begin{equation}~\label{pred0}
\setlength{\abovedisplayskip}{0.15cm} 
\setlength{\belowdisplayskip}{0.15cm} 
    \mathbf{x}_{\tau+1|k}=\mathbf{x}_{\tau|k}\oplus \left(\Delta t\mathbf{f}\left(\mathbf{x}_{\tau|k}, \mathbf{u}_{\tau|k}, \mathbf{0}\right)\right);\tau\geq k
\end{equation}

If only one step needs to be predicted, which is usually the case for measurements being the same sampling rate as that of the input, then $\tau=k$. Otherwise, the predict proceeds at each input and stops when a measurement comes.
\subsection{The error-state system}~\label{algorithm_ekf}
The error-state Kalman filter propagates the covariance matrix in the error state in order to avoid the over-parameterization in $\mathbf{x}$. The error state is defined for all future times $\tau\geq k$ as follows
\begin{equation}~\label{error_state}
\setlength{\abovedisplayskip}{0.15cm} 
\setlength{\belowdisplayskip}{0.15cm} 
    \delta\mathbf{x}_{\tau|k} = \mathbf{x}_{\tau}\boxminus\mathbf{x}_{\tau|k}, \tau\geq k.
\end{equation}

Substituting (\ref{cano}) and (\ref{pred0}) into (\ref{error_state}) leads to 
\begin{equation}
\setlength{\abovedisplayskip}{0.15cm} 
\setlength{\belowdisplayskip}{0.15cm} 
\begin{aligned}
    \delta\mathbf{x}_{\tau+1|k} &= \mathbf{x}_{\tau+1}\boxminus\mathbf{x}_{\tau+1|k} =\!\left( \mathbf{x}_{\tau}\!\oplus \!\left(\Delta t\mathbf{f}\!\left(
    \mathbf{x}_{\tau}, \mathbf{u}_{\tau}, \mathbf{w}_{\tau}\right)\right)\right)\\
    &\quad \quad \quad \quad \quad \quad \boxminus\!\left(\mathbf{x}_{\tau|k}\!\oplus\!\left(\Delta t\mathbf{f}\!\left(
    \mathbf{x}_{\tau|k}, \mathbf{u}_{\tau}, \mathbf{0}\right)\right)\right).
\end{aligned}
\end{equation}

Then substituting (\ref{error_state}) into the above equation leads to
\begin{equation}~\label{pred_ite}
\setlength{\abovedisplayskip}{0.15cm} 
\setlength{\belowdisplayskip}{0.15cm} 
\begin{aligned}
    &\hspace{-0.3cm} \delta \mathbf{x}_{\tau+1|k} \! = \!\left( \! \left(\mathbf{x}_{\tau|k}\!\boxplus\!\delta\mathbf{x}_{\tau|k}\right)\!\oplus \!\left(\Delta t\mathbf{f} \! \left(\mathbf{x}_{\tau|k}\!\boxplus\!\delta\mathbf{x}_{\tau|k}, \! \mathbf{u}_{\tau}, \! \mathbf{w}_\tau\right)\! \right) \! \right)\\
    &\quad\quad\quad\quad\quad \boxminus\!\left(\mathbf{x}_{\tau|k}\!\oplus \!\left(\Delta t\mathbf{f}\left(\mathbf{x}_{\tau|k}, \mathbf{u}_{\tau}, \mathbf{0}\right)\! \right) \! \right).
\end{aligned}
\end{equation}
which defines a {\it new} system starting from $\tau=k$. This system describes the time evolvement of the error state $\delta\mathbf{x}_{\tau|k}$ and hence is referred to as the error-state system. Since the new error-state system originates from the current measurement time $k$, it is re-defined once a new measurement is received to update the state estimate. Such a repeating process effectively restricts the error trajectory within a neighbor of zero, validating the minimal parameterization in $\delta\mathbf{x}_{\tau|k}$.  In case $\mathcal{M}_s$ is a Lie group, the error state in local tangent space of $\mathbf{x}_{\tau | k}$ is $\delta\mathbf{x}_{\tau|k} = \text{Log} ( \mathbf{x}_{\tau | k}^{-1} \cdot \mathbf{x}_{\tau} )$. Define $\widetilde{\mathbf{x}}_{\tau | k} = \mathbf{x}_{\tau}^{-1} \cdot \mathbf{x}_{\tau}$ the error state on the original manifold $\mathcal{M}_s$, the relation between the two trajectories  $\delta\mathbf{x}_{\tau|k}$ and $\widetilde{\mathbf{x}}_{\tau|k}$ is shown in  Fig.~\ref{fig:error_state}.

Since the error system (\ref{pred_ite}) has minimal parameterization, the standard Kalman filter variants could be employed. Accordingly, the two Kalman filter steps, predict and update, are referred to as “error-state predict” and “error-state update”, respectively, in order to distinguish from the original state space (\ref{cano}). In the following, we show in detail the error-state predict and error-state update.

The step of error-state prediction is simply applying standard Kalman prediction to the new error system (\ref{pred_ite}), which is in minimal parameterization and can be viewed as a usual nonlinear system. 

\subsubsection{Initial condition}
The error system (\ref{pred_ite}) starts from $\tau=k$, with initial estimation as below 
\begin{equation}~\label{initial_error}
\setlength{\abovedisplayskip}{0.15cm} 
\setlength{\belowdisplayskip}{0.15cm} 
    \delta\mathbf{x}_{(k|k)|k}=\left(\mathbf{x}_k\boxminus\mathbf{x}_{k|k}\right)_{|k} = \mathbf{x}_{k|k}\boxminus\mathbf{x}_{k|k}=\mathbf{0}
\end{equation}
here, the notation $\delta \mathbf{x}_{(k|k)|k}$ denotes the estimation of the random vector $\delta\mathbf{x}_{k|k}$ (recall that this is indeed random due to its definition in (\ref{error_state}) and that the ground truth state $\mathbf{x}_k$ is random) based on measurements up to $k$. The result in (\ref{initial_error}) is not surprising as $\delta \mathbf{x}_{k|k}$ is the error after conditioning on the measurements (up to $k$) already, so conditioning on  the same measurements again does not give more information.
\subsubsection{Error state predict}
The error state predict follows directly from the error-state system model in (\ref{pred_ite}) by setting the process noise $\mathbf{w}_{\tau}$ to zero:
\begin{equation}
\setlength{\abovedisplayskip}{0.15cm} 
\setlength{\belowdisplayskip}{0.15cm} 
    \begin{aligned}
    \delta\mathbf{x}_{(\tau+1|k)|k}\!\!&=\!\!\left(\left(\mathbf{x}_{\tau|k}\!\boxplus\!\delta\mathbf{x}_{(\tau|k)|k}\right)\right.\\
    &\left.\quad\oplus \!\left(\Delta t\mathbf{f}\!\left(\mathbf{x}_{\tau|k}\!\boxplus\!\delta\mathbf{x}_{(\tau|k)|k}, \mathbf{u}_{\tau}, \mathbf{0}\right)\right)\right)\\
    &\quad\boxminus\!\left(\mathbf{x}_{\tau|k}\!\oplus \!\left(\Delta t\mathbf{f}\left(\mathbf{x}_{\tau|k}, \mathbf{u}_{\tau}, \mathbf{0}\right)\right)\right);\tau\geq k
    \end{aligned}
\end{equation}

Starting from the initial condition in (\ref{initial_error}), we obtain
\begin{equation}
\setlength{\abovedisplayskip}{0.15cm} 
\setlength{\belowdisplayskip}{0.15cm} 
    \delta\mathbf{x}_{(\tau|k)|k}=\mathbf{0}; \forall \tau\geq k.
\end{equation}

Next, to propagate the error covariance, we need to linearize the system (\ref{pred_ite}) as follows 
\begin{equation}\label{e:error_system_lin}
\setlength{\abovedisplayskip}{0.15cm} 
\setlength{\belowdisplayskip}{0.15cm} 
    \delta\mathbf{x}_{\tau+1|k}\approx\mathbf{F}_{\mathbf{x}_{\tau}}\delta\mathbf{x}_{\tau|k}+\mathbf{F}_{\mathbf{w}_{\tau}}\mathbf{w}_\tau
\end{equation}
where $\mathbf{F}_{\mathbf{x}_{\tau}}$ is the partial differention of (\ref{pred_ite}) w.r.t $\delta\mathbf{x}_{\tau|k}$ at point $\delta\mathbf{x}_{(\tau|k) | k}=\mathbf{0}$ and can be computed by the chain rule:
\begin{equation}~\label{F_x}
\setlength{\abovedisplayskip}{0.15cm} 
\setlength{\belowdisplayskip}{0.15cm} 
    \begin{aligned}
    \mathbf{F}_{\mathbf{x}_{\tau}}\!\!\!&=\!\!\frac{\partial\left(\left(\mathbf{x}_{\tau|k}\boxplus\delta\mathbf{x}_{\tau|k}\right)\oplus \left(\!\Delta t\mathbf{f}\left(\mathbf{x}_{\tau|k}, \mathbf{u}_{\tau},\mathbf{0}\right)\!\right)\!\right)\boxminus\left(\mathbf{x}_{\tau|k}\oplus \left(\Delta t\mathbf{f}\left(\mathbf{x}_{\tau|k},\mathbf{u}_{\tau}, \mathbf{0}\right)\right)\right)}{\partial\delta\mathbf{x}_{\tau|k}}\\
    &+\!\!\frac{\partial\left(\mathbf{x}_{\tau|k}\oplus \left(\Delta t\mathbf{f}\left(\mathbf{x}_{\tau|k}\boxplus\delta\mathbf{x}_{\tau|k}, \mathbf{u}_{\tau},\mathbf{0}\right)\right)\right)\boxminus\left(\mathbf{x}_{\tau|k}\oplus \left(\Delta t\mathbf{f}\left(\mathbf{x}_{\tau|k},\mathbf{u}_{\tau}, \mathbf{0}\right)\right)\right)}{\partial\delta\mathbf{x}_{\tau|k}}\\
    &=\!\! \mathbf{G}_{{\mathbf{x}}_{\tau} } +\Delta t \mathbf{G}_{{\mathbf{f}}_{\tau} }  \frac{\partial\mathbf{f}\left(\mathbf{x}_{\tau | k} \boxplus \delta \mathbf{x},\mathbf{u}_{\tau}, \mathbf{0}\right)}{\partial \delta \mathbf{x}} |_{\delta \mathbf{x}=\mathbf{0}}
    \end{aligned}
\end{equation}
and $\mathbf{F}_{\mathbf{w}_{\tau}}$ is the partial differentiation of (\ref{pred_ite}) w.r.t $\mathbf{w}_{\tau}$ at the point $\mathbf{w}_\tau=\mathbf{0}$, as follows
\begin{equation}~\label{F_w}
\setlength{\abovedisplayskip}{0.15cm} 
\setlength{\belowdisplayskip}{0.15cm} 
    \begin{aligned}
     \mathbf{F}_{\mathbf{w}_{\tau}}\!\!&=\!\!\frac{\partial\left(\mathbf{x}_{\tau|k}\oplus \left(\Delta t\mathbf{f}\left(\mathbf{x}_{\tau|k}, \mathbf{u}_{\tau},\mathbf{w}_\tau\right)\right)\right)\boxminus\left(\mathbf{x}_{\tau|k}\oplus \left(\Delta t\mathbf{f}\left(\mathbf{x}_{\tau|k},\mathbf{u}_{\tau}, \mathbf{0}\right)\right)\right)}{\partial\mathbf{w}_\tau}\\
     &=\Delta t \mathbf{G}_{{\mathbf{f}}_{\tau} }  \frac{\partial\mathbf{f}\left(\mathbf{x}_{\tau|k},\mathbf{u}_{\tau}, \mathbf{w} \right)}{\partial \mathbf{w} } |_{\mathbf{w}=\mathbf{0}}
    \end{aligned}
\end{equation}
where
\begin{equation}
\setlength{\abovedisplayskip}{0.15cm} 
\setlength{\belowdisplayskip}{0.15cm} 
    \begin{aligned}
    \mathbf{G}_{{\mathbf{x}}_{\tau} } &=\left. \frac{\partial\left(\left(\left(\mathbf{x}\boxplus\mathbf{u}\right)\oplus \mathbf{v}\right)\boxminus\mathbf{y}\right)}{\partial\mathbf{u}} \right|_{\substack{
    \mathbf{x}=\mathbf{x}_{\tau|k}; \mathbf{u}=\mathbf{0};\mathbf{v}=\Delta t\mathbf{f}\left(\mathbf{x}_{\tau|k}, \mathbf{u}_{\tau}, \mathbf{0}\right); \\ \mathbf{y}=\mathbf{x}_{\tau|k} \oplus \Delta t\mathbf{f}\left(\mathbf{x}_{\tau|k},\mathbf{u}_{\tau}, \mathbf{0}\right)}} \\
    \mathbf{G}_{{\mathbf{f}}_{\tau} } &=\left. \frac{\partial\left(\left(\left(\mathbf{x}\boxplus\mathbf{u}\right)\oplus \mathbf{v}\right)\boxminus\mathbf{y}\right)}{\partial\mathbf{v}} \right|_{\substack{
    \mathbf{x}=\mathbf{x}_{\tau|k}; \mathbf{u}=\mathbf{0}; \mathbf{v}=\Delta t\mathbf{f}\left(\mathbf{x}_{\tau|k}, \mathbf{u}_{\tau}, \mathbf{0}\right);  \\ \mathbf{y}=\mathbf{x}_{\tau|k} \oplus \Delta t\mathbf{f}\left(\mathbf{x}_{\tau|k},\mathbf{u}_{\tau}, \mathbf{0}\right)}}
    \end{aligned}
\end{equation}

Finally, the covariance is propagated as
\begin{equation}\label{e:Cov_prop}
\setlength{\abovedisplayskip}{0.15cm} 
\setlength{\belowdisplayskip}{0.15cm} 
    \mathbf{P}_{\tau+1|k}=\mathbf{F}_{\mathbf{x}_{\tau}}\mathbf{P}_{\tau|k}\mathbf{F}_{\mathbf{x}_{\tau}}^T+\mathbf{F}_{\mathbf{w}_{\tau}}\mathcal{Q}_{\tau} \mathbf{F}_{\mathbf{w}_{\tau}}^T
\end{equation}

\begin{figure}[t]
\vspace{-0.3cm}
\setlength{\abovecaptionskip}{-0.1cm} 
\setlength{\belowcaptionskip}{-0.2cm} 
\centering
        \includegraphics[width=0.8\columnwidth]{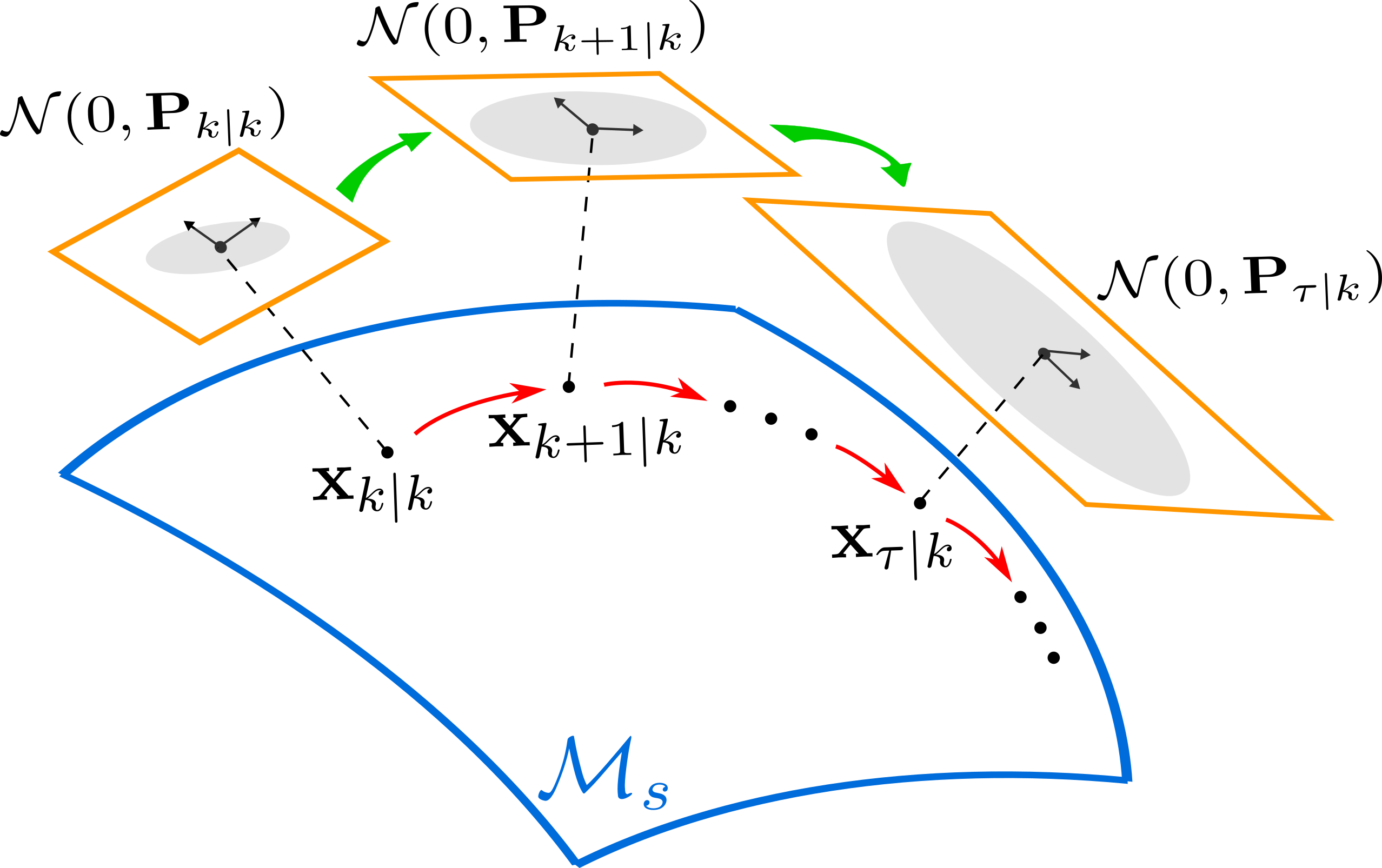}
         \caption{Predict of the state (red arrows on the manifold) and propagation of its covariance (green arrows between local homeomorphic spaces). The yellow $\mathbb{R}^n$ spaces are locally homeomorphic to $\mathcal{M}_s$.}
        \label{fig:cov_prop}
        \vspace{-0.3cm}
\end{figure}

The predict of the state in (\ref{pred0}) and the propagation of respective covariance in (\ref{e:Cov_prop}) are illustrated in Fig.~\ref{fig:cov_prop}. 

\subsubsection{Isolation of manifolds}~\label{sec:isolation_predict}
As shown by (\ref{F_x}) and (\ref{F_w}), the two system matrices $\mathbf{F}_{\mathbf{x}_\tau}, \mathbf{F}_{\mathbf{w}_\tau}$ are well separated into manifold-specific parts $\mathbf{G}_{{\mathbf{x}}_{\tau} }, \mathbf{G}_{{\mathbf{f}}_{\tau} }$ and system-specific parts $ \frac{\partial\mathbf{f}\left(\mathbf{x}_{\tau | k} \boxplus \delta \mathbf{x},\mathbf{u}_{\tau}, \mathbf{0}\right)}{\partial \delta \mathbf{x}}|_{\delta \mathbf x = \mathbf 0} ,  \frac{\partial\mathbf{f}\left(\mathbf{x}_{\tau|k},\mathbf{u}_{\tau}, \mathbf{w} \right)}{\partial \mathbf{w} }|_{\mathbf w = \mathbf 0}$. The manifold-specific parts for commonly used manifolds are listed in Tab.~\ref{tab:mani-spec_Fx}. Moreover, based on (\ref{diff_compound}), the manifold-specific parts for any {\it compound manifold} are the concatenation of that of these {\it primitive manifolds}.
\begin{table}[ht]
\vspace{-0.3cm}
\setlength{\abovecaptionskip}{-0cm} 
\setlength{\belowcaptionskip}{-0.2cm} 
\centering
\caption{Manifold-specific parts for $\mathbf{G}_{{\mathbf{x}}_{\tau} }, \mathbf{G}_{{\mathbf{f}}_{\tau} }$}\label{tab:mani-spec_Fx}
\begin{threeparttable}
    \begin{tabular}{lll}
    \toprule
    $\mathcal{M}_s$&$\!\mathbf{G}_{\mathbf{x}_\tau}$&$\mathbf{G}_{\mathbf{f}_\tau}$ \\
    \midrule
    $\mathbb{R}^n$&$\!\mathbf{I}_{n\times n}$&$\mathbf{I}_{n\times n}$\\ 
    $SO(3)$&$\!{\rm Exp}(-\!\Delta t \mathbf{f}(\mathbf{x}_{\tau|k}, \!\mathbf{u}_{\tau}, \mathbf{0}))$&$\mathbf{A}\!\left(\Delta t \mathbf{f}\!\left(\mathbf{x}_{\tau|k}, \mathbf{u}_{\tau}, \mathbf{0}\right)\right)^T$\\ $\mathbb{S}^2(r)$&$\!-\frac{1}{{r}^2}\mathbf{B}\!\left(\mathbf{x}_{\tau+1|k}\right)^T$&$-\frac{1}{{r}^2}\mathbf{B}\!\left(\mathbf{x}_{\tau+1|k}\right)^T$\\
    &$\!\cdot\mathbf{R}(\Delta t \mathbf{f}(\mathbf{x}_{\tau|k},\mathbf{u}_\tau,\mathbf{0}))$&$\cdot\mathbf{R}(\Delta t \mathbf{f}\!(\mathbf{x}_{\tau|k},\mathbf{u}_\tau,\mathbf{0}))$\\
    &$\cdot \lfloor \mathbf{x}_{\tau|k} \rfloor^2 \mathbf{B}\!\left(\mathbf{x}_{\tau|k}\right)$&$\cdot\lfloor\! \mathbf{x}_{\tau|k}\!\rfloor^2\! \mathbf{A}\!\left(\!\Delta t\mathbf{f}\!\left(\!\mathbf{x}_{\tau|k},\mathbf{u}_\tau,\mathbf{0}\!\right)\!\right)^T$\\
    \bottomrule
    \end{tabular}
    \end{threeparttable}
    \vspace{-0.3cm}
\end{table}
\subsection{State update}~\label{sec:state_update_esekf}

The state prediction in (\ref{pred0}) along with the covariance propagation in (\ref{e:Cov_prop}) proceed at every input $\mathbf u_{\tau}$, until the next measurement arrives. At that time, the predicted state and the propagated covariance give a prior distribution for the state while the next measurement serves as an observation of the state. Combining these two, we obtain a posteriori distribution of the state and subsequently perform a maximum a posteriori estimation (i.e., state update). 

\subsubsection{Prior distribution} 
Assume the next measurement arrives at step $\tau > k$. Without the loss of generality, we assume $\tau = k+1$, i.e., the measurement rate is equal to the input rate. The predicted error state $\delta \mathbf x_{k+1|k}$ and its covariance $\mathbf P_{k+1 | k}$ create a prior distribution for $\mathbf x_{k+1}$:
\begin{equation}
\setlength{\abovedisplayskip}{0.15cm} 
\setlength{\belowdisplayskip}{0.15cm} 
    \label{dist:pred}
    \delta \mathbf x_{k+1|k} = \mathbf x_{k+1} \boxminus \mathbf x_{k+1|k} \sim \mathcal{N} \left(\mathbf 0, \mathbf P_{k+1|k} \right)
\end{equation}
\subsubsection{Iterated update}
Now assume the next measurement at $k+1$ is $\mathbf z_{k+1}$. Since the measurement model could be nonlinear, linear approximations have to be made at each iteration. Assume in the $j$-th iteration, the state estimate is $\mathbf x_{k+1|k+1}^j$, where $\mathbf x_{k+1|k+1}^j = \mathbf x_{k+1|k}$ (i.e., the priori estimate) for $j = 0$ , then define and linearize the residual as below
\begin{equation}
\setlength{\abovedisplayskip}{0.15cm} 
\setlength{\belowdisplayskip}{0.15cm} 
    \label{eq:residual}
    \begin{aligned}
     \mathbf r_{k+1}^j &\triangleq \mathbf z_{k+1} \! \boxminus \! \mathbf h (\mathbf x_{k+1|k+1}^j, \mathbf 0 ) \\ 
     &= \mathbf h(\mathbf x_{k+1}, \mathbf v_{k+1}) \boxminus \mathbf h (\mathbf x_{k+1|k+1}^j, \mathbf 0 ) \\
     &= \mathbf h(\mathbf x_{k+1|k+1}^j \boxplus \delta \mathbf x_j, \mathbf v_{k+1}) \boxminus \mathbf h (\mathbf x_{k+1|k+1}^j, \mathbf 0 ) \\
     &\approx \mathbf D_{k+1}^j \mathbf v_{k+1} + \mathbf H_{k+1}^j \delta \mathbf x_j
    \end{aligned}
\end{equation}
where $\delta \mathbf x_j \triangleq \mathbf x_{k+1} \boxminus \mathbf x_{k+1|k+1}^j$ is the error between the ground true state $\mathbf x_{k+1}$ and its most recent estimate $\mathbf x_{k+1|k+1}^j$, and
\begin{equation}
\setlength{\abovedisplayskip}{0.15cm} 
\setlength{\belowdisplayskip}{0.15cm} 
    \label{e:iterative_error}
    \begin{aligned}
    \mathbf{H}_{k+1}^j &= \frac{\partial \left( \mathbf{h}(\mathbf{x}_{k+1|k+1}^j \boxplus \delta \mathbf{x}, \mathbf 0)\boxminus\mathbf{h}(\mathbf{x}_{k+1|k+1}^j, \mathbf 0) \right)}{\partial \delta \mathbf{x}}|_{ \delta \mathbf{x}=\mathbf{0}} \\
         &= \frac{\partial \mathbf{h}(\mathbf{x}_{k+1|k+1}^j \boxplus \delta \mathbf{x}, \mathbf 0)}{\partial \delta \mathbf{x}}|_{\delta \mathbf{x}=\mathbf{0}}, \text{ for } \mathcal{M}_m = \mathbb{R}^m, \\
    \mathbf D_{k+1}^j &= \frac{\partial \left( \mathbf{h}(\mathbf{x}_{k+1|k+1}^j, \mathbf v)\boxminus\mathbf{h}(\mathbf{x}_{k+1|k+1}^j, \mathbf 0) \right)}{\partial \mathbf{v}}|_{  \mathbf{v}=\mathbf{0}} \\
         &= \frac{\partial \mathbf{h}(\mathbf{x}_{k+1|k+1}^j, \mathbf v)}{\partial \mathbf{v}}|_{ \mathbf{v}=\mathbf{0}}, \text{ for } \mathcal{M}_m = \mathbb{R}^m
    \end{aligned}
\end{equation}

Equation (\ref{eq:residual}) defines a linearized observation model  for $\delta \mathbf x_j$ (and equivalently for $\mathbf x_{k+1}$): 
\begin{equation}\label{e:posteriori}
\setlength{\abovedisplayskip}{0.15cm} 
\setlength{\belowdisplayskip}{0.15cm} 
    \begin{aligned}
      &(\mathbf D_{k+1}^j \mathbf v_{k+1})| \delta \mathbf x_j \!= \! \mathbf r_{k+1}^j \!-\!  \mathbf H_{k+1}^j \delta \mathbf x_j \sim \mathcal{N} \left(\mathbf 0, \bar{\mathcal{R}}_{k+1} \right); \\
      & \bar{\mathcal{R}}_{k+1} = \mathbf D_{k+1}^j \mathcal{R}_{k+1} (\mathbf D_{k+1}^j)^T
    \end{aligned}
\end{equation}

\begin{figure}[tb]
\vspace{-0.3cm}
\setlength{\abovecaptionskip}{-0.1cm} 
\setlength{\belowcaptionskip}{-0.2cm} 
\centering
        \includegraphics[width=0.7\columnwidth]{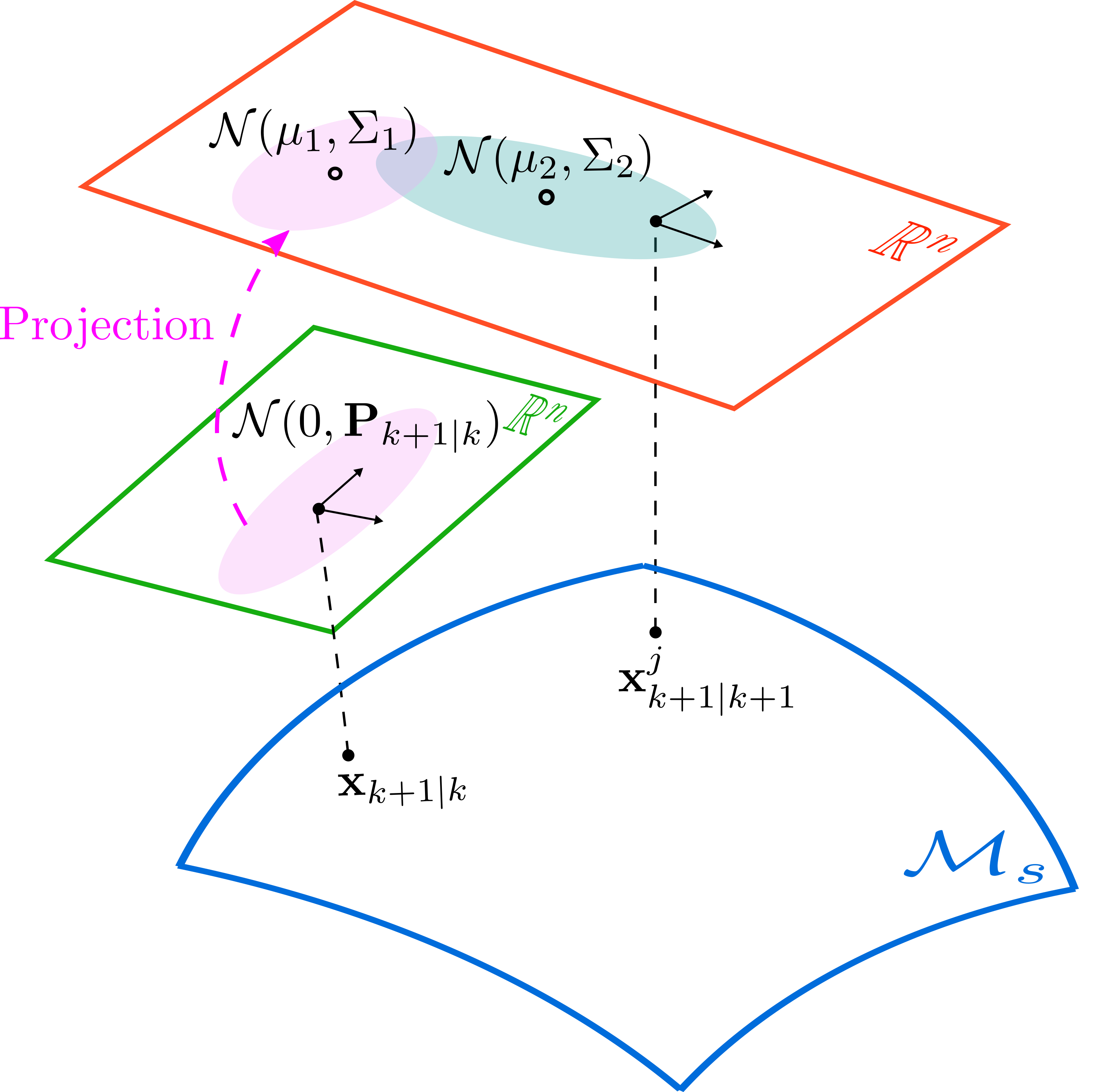}
\centering
\caption{Prior distribution $\mathcal{N}(\mathbf 0, \mathbf P_{k+1|k})$, its projection $\mathcal{N}(\boldsymbol{\mu}_1, \mathbf \Sigma_1)$, and the distribution imposed by measurements $\mathcal{N}(\boldsymbol{\mu}_2, \mathbf \Sigma_2)$, where $\boldsymbol{\mu}_1=-\mathbf{J}_{k+1}^j(\mathbf{x}_{k+1|k+1}^j\boxminus\mathbf{x}_{k+1|k})$, $\mathbf \Sigma_1=\mathbf{J}_{k+1}^j\mathbf{P}_{k+1|k}(\mathbf{J}_{k+1}^j)^{T}$ and $\boldsymbol{\mu}_2=(\mathbf{H}_{k+1}^j)^{-1}\mathbf{r}_{k+1}^j$, $\mathbf \Sigma_2=(\mathbf{H}_{k+1}^j)^{-1}\bar{\mathcal{R}}_{k+1}(\mathbf{H}_{k+1}^j)^{-T}$. The red and green $\mathbb{R}^n$ spaces are locally homeomorphic to the $\mathcal{M}_s$ space.}
\vspace{-0.3cm}
\label{fig:cov_proj}
\end{figure}

Note that the prior distribution for $\mathbf x_{k+1}$ in (\ref{dist:pred}) is defined in terms of $\delta \mathbf x_{k+1|k}$, which lies in the local homeomorphic space at $\mathbf x_{k+1|k}$, while the observation model of $\mathbf x_{k+1}$ in (\ref{e:posteriori}) is defined in terms of $\delta \mathbf x_j$, which lies in the local homeomorphic space at $\mathbf x_{k+1|k+1}^j$. To combine them into a posterior distribution for $\mathbf x_{k+1}$, we need to project them into the same space as shown in Fig.~\ref{fig:cov_proj}. We choose to project $\delta \mathbf x_{k+1|k}$ to the local homeomorphic space at $\mathbf x_{k+1|k+1}^j$:
\begin{equation}
\setlength{\abovedisplayskip}{0.15cm} 
\setlength{\belowdisplayskip}{0.15cm} 
    \label{e:projection}
    \begin{aligned}
     \delta \mathbf x_{k+1|k} \! &= \! \mathbf x_{k+1} \boxminus \mathbf x_{k+1|k} \!=\!  (\mathbf x_{k+1|k+1}^j \! \boxplus \!  \delta \mathbf x_j) \! \boxminus  \!\mathbf x_{k+1|k} \\
     & = (\mathbf x_{k+1|k+1}^j \! \boxminus  \!\mathbf x_{k+1|k})  + (\mathbf J_{k+1}^j)^{-1} \delta \mathbf x_j
    \end{aligned}
\end{equation}
where 
\begin{equation}
\setlength{\abovedisplayskip}{0.15cm} 
\setlength{\belowdisplayskip}{0.15cm} 
    \label{eq:update}
    \begin{aligned}
         \mathbf J_{k+1}^j &\!= \! \left. \frac{\partial\left(\! \left( \! \left(\mathbf{x}\boxplus\mathbf{u}\right)\oplus \mathbf{v}\right)\boxminus\mathbf{y}\right)}{\partial\mathbf{u}} \right|_{\begin{subarray}
          \mathbf \mathbf x =\mathbf{x}_{k+1|k}, \mathbf{u}= \mathbf x_{k+1|k+1}^j\boxminus\mathbf{x}_{k+1|k},\\
          \mathbf{v}=\mathbf{0}, \mathbf{y}= \mathbf{x}_{k+1|k+1}^j\end{subarray}} 
    \end{aligned}
\end{equation}
is the inverse Jacobian of $\delta \mathbf x_{k+1|k}$ with repect to $\delta \mathbf x_j$ evaluated at zero. Then, the equivalent prior distribution for $\mathbf x_{k+1}$ can now be expressed in terms of $\delta \mathbf x_j$ as below: 
\begin{equation}\label{e:prior}
\setlength{\abovedisplayskip}{0.15cm} 
\setlength{\belowdisplayskip}{0.15cm} 
    \begin{aligned}
        &\delta \mathbf x_j \! \sim  \!\mathcal{N}\! (-\mathbf{J}_{k+1}^j(\mathbf x_{k+1|k+1}^j \! \boxminus  \!\mathbf x_{k+1|k}), \mathbf J_{k+1}^j \mathbf P_{k+1|k} (\mathbf J_{k+1}^j)^T )
    \end{aligned}
\end{equation}

Now the prior distribution (\ref{e:prior}) and observation model (\ref{e:posteriori}) are in the same space and can be combined to produce the posterior distribution and finally the maximum a-posteriori estimate (MAP) in terms of $\delta \mathbf x_j$ (see Fig. \ref{fig:cov_proj}):
\begin{equation}
\setlength{\abovedisplayskip}{0.15cm} 
\setlength{\belowdisplayskip}{0.15cm} 
    \label{e:cost}
    \begin{aligned}
     &\arg \underset{ \delta \mathbf x_j}{\max}  \log \left( \mathcal{N} \! \left(\delta \mathbf x_j \right) \! \mathcal{N} \left((\mathbf D_{k+1}^j \mathbf v_{k+1})| \delta \mathbf x_j \right) \right)  \\
     &= \arg \underset{ \delta \mathbf x_j}{\min} \ g \left( \delta \mathbf x_j \right); \ g \! \left( \delta \mathbf x_j \right) \! = \!  \| \mathbf r_{k+1}^j\! -\! \mathbf H_{k+1}^j \delta \mathbf x_j \|^2_{\bar{\mathcal{R}}_{k+1}} \\
    & \quad \quad + \| (\mathbf x_{k+1|k+1}^j \! \boxminus  \!\mathbf x_{k+1|k})  + (\mathbf J_{k+1}^j)^{-1} \delta \mathbf x_j \|^2_{\mathbf P_{k+1|k}} 
    \end{aligned}
\end{equation}
where $\| \mathbf x \|^2_{\mathbf A}  = \mathbf x^T \mathbf A^{-1} \mathbf x$. The optimization problem in (\ref{e:cost}) is a standard quadratic programming and the optimal solution $\delta \mathbf x^{o} $ can be easily obtained, which is the Kalman update \cite{bell1993iterated}: 
\small
\begin{equation}
\setlength{\abovedisplayskip}{0.15cm} 
\setlength{\belowdisplayskip}{0.15cm} 
    \label{e:ekf_update}
    \begin{aligned}
        \delta \mathbf x_j^{o} & \! = \! - \mathbf J_{k+1}^j (\mathbf x_{k+1|k+1}^j \! \boxminus  \!\mathbf x_{k+1|k}) \\
        & +  \mathbf K_{k+1}^j (\mathbf r_{k+1}^j   + \mathbf H_{k+1}^j \mathbf J_{k+1}^j (\mathbf x_{k+1|k+1}^j \! \boxminus  \!\mathbf x_{k+1|k}))\\
        \mathbf K_{k+1}^j & \! = \!  \left( \mathbf Q_{k+1}^j \right)^{-1} (\mathbf H_{k+1}^j)^T \bar{\mathcal{R}}_{k+1}^{-1}\\
        &\! = \!  \mathbf J_{k+1}^j \mathbf P_{k+1|k} (\mathbf J_{k+1}^j)^{T} (\mathbf H_{k+1}^{j})^T (\mathbf S_{k+1}^j)^{-1}  \\
        \mathbf Q_{k+1}^j &= (\mathbf H_{k+1}^j)^T \bar{\mathcal{R}}_{k+1}^{-1}\mathbf H_{k+1}^j \! + \! (\mathbf J_{k+1}^j)^{-T} \mathbf P_{k+1|k}^{-1}\left( \mathbf J_{k+1}^j\right)^{-1} \\
        \mathbf S_{k+1}^j &= \mathbf H_{k+1}^j \mathbf J_{k+1}^j  \mathbf P_{k+1|k} (\mathbf J_{k+1}^j)^{T} (\mathbf H_{k+1}^j)^T +\bar{\mathcal{R}}_{k+1}
    \end{aligned}
\end{equation}
\normalsize
where $\mathbf Q_{k+1}^j$ is the Hessian matrix of (\ref{e:cost}) and its inverse represents the covariance of $\delta \mathbf x_j - \delta \mathbf x_j^o$, which can be furthermore written into the form below \cite{bell1993iterated}
\begin{equation}
\setlength{\abovedisplayskip}{0.15cm} 
\setlength{\belowdisplayskip}{0.15cm} 
    \begin{aligned}
     \mathbf P_{k+1}^j &= (\mathbf Q_{k+1}^j )^{-1} \\
     &= (\mathbf I - \mathbf K_{k+1}^j \mathbf H_{k+1}^j) \mathbf J_{k+1}^j \mathbf P_{k+1|k} (\mathbf J_{k+1}^j)^{T}
    \end{aligned}
\end{equation}

With the optimal $\delta \mathbf x_j^{o}$, the update of $\mathbf x_{k+1}$ estimate is then
\begin{equation}
\setlength{\abovedisplayskip}{0.15cm} 
\setlength{\belowdisplayskip}{0.15cm} 
    \mathbf x_{k+1 | k+1}^{j+1}= \mathbf x_{k+1 | k+1}^{j} \boxplus \delta \mathbf x_j^{o} 
\end{equation}

The above process iterates until convergence or exceeding the maximum steps. 

\begin{figure}[t]
\vspace{-0.3cm}
\setlength{\abovecaptionskip}{-0.1cm} 
\setlength{\belowcaptionskip}{-0.2cm} 
\centering
        \includegraphics[width=0.7\columnwidth]{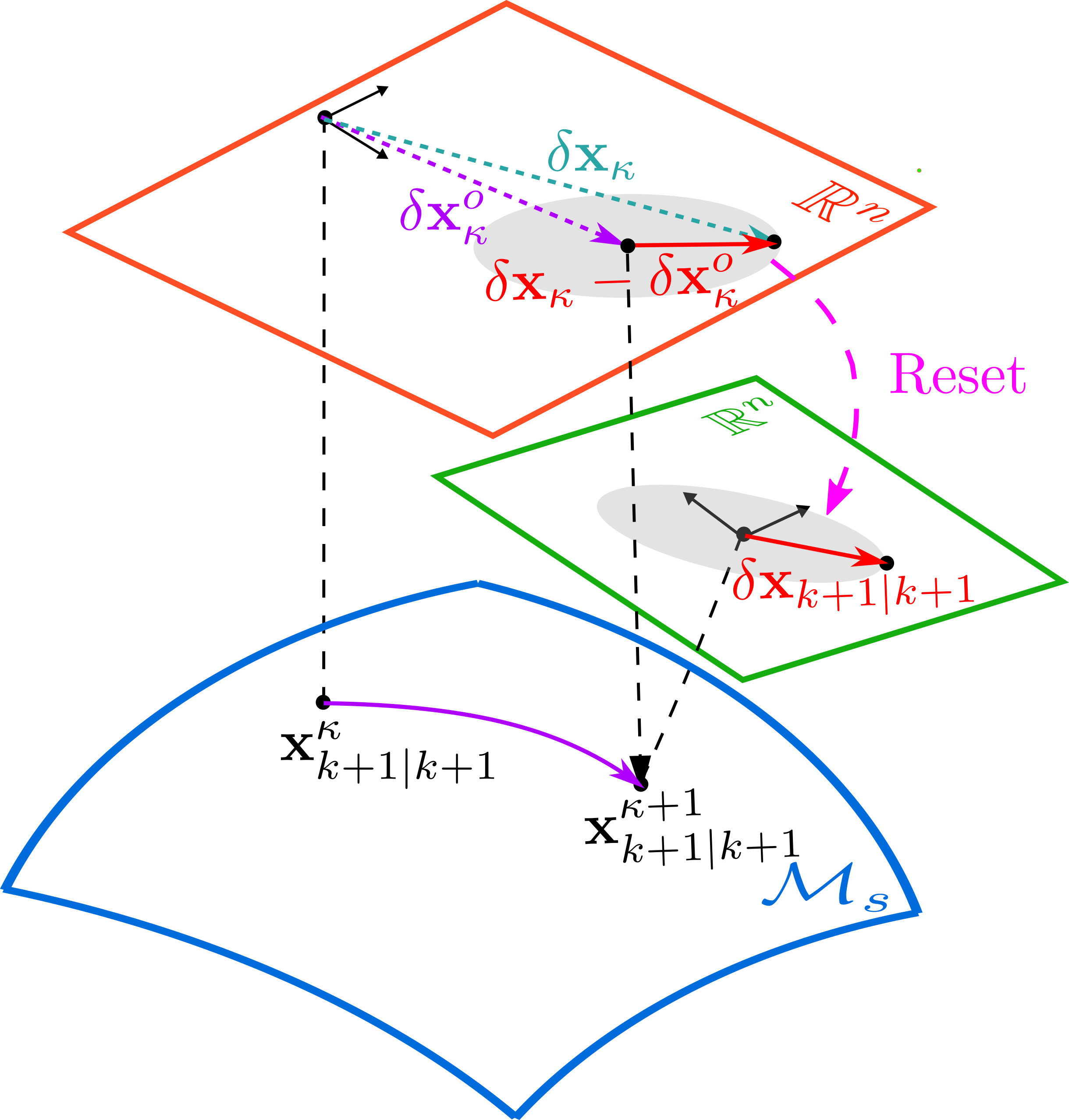}\\
        \centering
    \caption{Reset of covariance. The red and green $\mathbb{R}^n$ spaces are local homeomorphic linear spaces of the $\mathcal{M}_s$ space at points $\mathbf{x}_{k+1|k+1}^\kappa$ and $\mathbf{x}_{k+1|k+1}^{\kappa+1}$ respectively.}
    \label{fig:cov_reset}
    \vspace{-0.3cm}
\end{figure}

\subsubsection{Covariance reset}~\label{sec:reset_cov}
Assume the iterated update stops after $\kappa \geq 0$ iterations, resulting in a MAP estimate $\mathbf x_{k+1 | k+1}^{\kappa+1}$ and covariance matrix $\mathbf P_{k+1}^\kappa$. Then $\mathbf x_{k+1 | k + 1}^{\kappa+1}$ becomes the Kalman update of $\mathbf x_{k+1}$ 
\begin{equation}
\setlength{\abovedisplayskip}{0.15cm} 
\setlength{\belowdisplayskip}{0.15cm} 
     \mathbf x_{k+1|k+1} = \mathbf x_{k+1 | k + 1}^{\kappa+1}
\end{equation}
which is passed to the next step of Kalman filter. For the $\mathbf P_{k+1}^\kappa$, note that it describes the covariance of $\delta \mathbf x_\kappa - \delta \mathbf x_\kappa^o$ in the local homeomorphic space of $\mathbf{x}_{k+1|k+1}^\kappa$, while what required at the next step of Kalman filter should be the covariance $\mathbf{P}_{k+1|k+1}$ describing error $\delta\mathbf{x}_{k+1|k+1}$ in the local homeomorphic space of $\mathbf{x}_{k+1|k+1}$ (see Sec.~\ref{sec:initialization}). This discrepancy necessitates a projection step as shown in Fig. \ref{fig:cov_reset}. According to the definition of the error state in (\ref{error_state}), we have
\begin{equation}
\setlength{\abovedisplayskip}{0.15cm} 
\setlength{\belowdisplayskip}{0.15cm} 
    \begin{aligned}
    \delta\mathbf{x}_{k+1|k+1}&=\mathbf{x}_{k+1}\boxminus\mathbf{x}_{k+1|k+1}=\mathbf{x}_{k+1}\boxminus\mathbf{x}_{k+1|k+1}^{\kappa + 1}\\
    \delta\mathbf{x}_{\kappa}&=\mathbf{x}_{k+1}\boxminus\mathbf{x}_{k+1|k+1}^{\kappa}
    \end{aligned}
\end{equation}
which leads to
\begin{equation}
\setlength{\abovedisplayskip}{0.15cm} 
\setlength{\belowdisplayskip}{0.15cm} 
\begin{aligned}
\delta\mathbf{x}_{k+1|k+1}& \! \!=\!(\mathbf{x}_{k+1|k+1}^{\kappa} \! \!  \boxplus \! \delta\mathbf{x}_{\kappa}) \! \boxminus \! \mathbf{x}_{k+1|k+1}^{\kappa+1} \\
&=\mathbf{L}_{k+1} (\delta\mathbf{x}_{\kappa} - \delta\mathbf{x}_{\kappa}^o )
\end{aligned}    
\end{equation}
where
\begin{equation}~\label{e:reset_esekf}
\setlength{\abovedisplayskip}{0.15cm} 
\setlength{\belowdisplayskip}{0.15cm} 
    \hspace{-0.3cm}\mathbf{L}_{k+1} \!=\! \left. \frac{\partial\left(\! \left( \! \left(\mathbf{x}\boxplus\mathbf{u}\right)\oplus \mathbf{v}\right)\boxminus\mathbf{y}\right)}{\partial\mathbf{u}} \right|_{\begin{subarray}
    \mathbf \mathbf{x}=\mathbf{x}_{k+1|k+1}^{\kappa}, \mathbf{u}= \delta\mathbf{x}_{\kappa}^o,\\
    \mathbf{v}=\mathbf{0}, \mathbf{y}= \mathbf{x}_{k+1|k+1}^{\kappa+1}\end{subarray}}
\end{equation}
is the Jacobian of $\delta\mathbf{x}_{k+1|k+1}$ w.r.t. $\delta\mathbf{x}_{\kappa}$ evaluated at $\delta\mathbf{x}_{\kappa}^o$.

Finally, the covariance for $\delta\mathbf{x}_{k+1|k+1}$ is
\begin{equation}
\setlength{\abovedisplayskip}{0.15cm} 
\setlength{\belowdisplayskip}{0.15cm} 
    \mathbf{P}_{k+1|k+1}=\mathbf{L}_{k+1} {\mathbf{P}}_{k+1}^{\kappa} \mathbf{L}_{k+1}^T
\end{equation}

In particular, for an extended Kalman filter (i.e., $\kappa = 0$), $\mathbf J_{k+1}^{\kappa} = \mathbf I$ while $\mathbf L_{k+1} \neq \mathbf I$; for a fully converged iterated Kalman filter (i.e., $\kappa$ is sufficiently large), $\mathbf J_{k+1}^{\kappa} \neq \mathbf I$ while $\mathbf L_{k+1} = \mathbf I$.

\subsubsection{Isolation of manifolds}~\label{sec:isolation_update}
Notice that the two matrices $\mathbf J_{k+1}^j$and $\mathbf{L}_{k+1}$ required in the Kalman upudate only depend on the manifold $\mathcal{M}_s$ thus being manifold-specific matrices. Their values for commonly used manifolds are summarized in Tab.~\ref{tab:mani-spec_L}. Again, the manifold-specific parts for any {\it compound manifold} are the concatenation of these {\it primitive manifolds}. 
\begin{table}[ht]
\vspace{-0.3cm}
\setlength{\abovecaptionskip}{-0cm} 
\setlength{\belowcaptionskip}{-0.2cm} 
\centering
\caption{Manifold-specific parts for $\mathbf{J}_{k+1}^j$, $\mathbf{L}_{k+1}$}\label{tab:mani-spec_L}
\begin{threeparttable}
    \begin{tabular}{lll}
    \toprule
    $\mathcal{M}_s$&$\mathbf{J}_{k+1}^j$&$\mathbf{L}_{k+1}$\\
    \midrule
    $\mathbb{R}^n$&$\mathbf{I}_{n\times n}$&$\mathbf{I}_{n\times n}$\\
    $SO(3)$&$\mathbf{A}(\delta\mathbf{x}_{k+1|k+1}^j)^T$&$\mathbf{A}(\delta\mathbf{x}_{\kappa}^o)^T$\\
    $\mathbb{S}^2(r)$&$\frac{-1}{{r}^2}\mathbf{B}(\mathbf{x}_{k+1|k+1}^j)^T$&$\frac{-1}{{r}^2}\mathbf{B}(\mathbf{x}_{k+1|k+1}^{\kappa+1})^T$\\
    &$\cdot\mathbf{R}(\mathbf{B}(\mathbf{x}_{k+1|k})\delta\mathbf{x}_{k+1|k+1}^j)$&$\cdot\mathbf{R}(\mathbf{B}(\mathbf{x}_{k+1|k+1}^\kappa)\delta\mathbf{x}_{\kappa}^o)$\\
    &$\cdot\lfloor\!\mathbf{x}_{k+1|k}\rfloor^2$&$\cdot\lfloor\mathbf{x}_{k+1|k+1}^\kappa\rfloor^2$\\
    &$\cdot\mathbf{A}(\mathbf{B}(\mathbf{x}_{k+1|k})\delta\mathbf{x}_{k+1|k+1}^j)^T$&$\cdot\mathbf{A}(\mathbf{B}(\mathbf{x}_{k+1|k+1}^\kappa)\delta\mathbf{x}_{\kappa}^o)^T$\\
    &$\cdot\mathbf{B}(\mathbf{x}_{k+1|k})$&$\cdot\mathbf{B}(\mathbf{x}_{k+1|k+1}^\kappa)$\\
    \bottomrule\\
    \end{tabular}
    \begin{tablenotes}
    \item[1] $\delta\mathbf{x}_{k+1|k+1}^j=\mathbf{x}_{k+1|k+1}^j\boxminus\mathbf{x}_{k+1|k}$.
    \end{tablenotes}
    \end{threeparttable}
    \vspace{-0.3cm}
\end{table}

\subsection{Fully iterated Kalman filter on differentiable Manifolds}
Summarizing all the above procedures in sections \ref{sec:initialization}, \ref{sec:esekf_state_propagation}, \ref{algorithm_ekf}, \ref{sec:state_update_esekf} lead to the full error-state iterated Kalman filter operating on differentiable manifolds (see Algorithm 1). Setting the number of iteration $N_{\text{max}}$ to zero leads to the error-state extended Kalman filter used in \cite{Lu2019IMU, sola2017quaternion}. 

\begin{table}[ht]
\setlength{\abovecaptionskip}{-0cm} 
\setlength{\belowcaptionskip}{-0.2cm} 
    \begin{tabular}{ll}
    \toprule
    \!Algorithm 1:\!\!&\!\!\!\!Iterated error-state Kalman filter on differentiable manifolds  \\
    \midrule
    \!Input:\!\!&\!\!\!\!$\mathbf{x}_{k|k}$, $\mathbf{P}_{k|k}$, $\mathbf{u}_k$, $\mathbf{z}_{k+1}$\\
    \!Output:\!\!&\!\!\!\! State update $\mathbf{x}_{k+1|k+1}$ and covariance $\mathbf{P}_{k+1|k+1}$\\
    \!Predict:\!\!&\!\!\!\!\\
    \!&\!\!\!\!$\mathbf{x}_{k+1|k}=\mathbf{x}_{k|k}\oplus \left(\Delta t\mathbf{f}\left(\mathbf{x}_{k|k}, \mathbf{u}_{k}, \mathbf{0}\right)\right)$;\\
    \!&$\!\!\!\!\mathbf{P}_{k+1|k}=\mathbf{F}_{
    \mathbf{x}_{k}}\mathbf{P}_{k|k}\mathbf{F}_{\mathbf{x}_{k}}^T+\mathbf{F}_{\mathbf{w}_{k}}\mathcal{Q}_k \mathbf{F}_{\mathbf{w}_{k}}^T$;\\
    \!Update:\!\!&\!\!\!\!\\
    \!& \!\!\!\! $j = -1; \ \ {\mathbf{x}}_{k+1|k+1}^{0} = \mathbf x_{k+1|k}$; \\
    \!&\!\!\!\!\textbf{while} Not Converged and $j \leq N_{\text{max}} - 1$ \textbf{do}\\
    \!&$\ j = j + 1$; \\
    \!&$\ $Calculate $\mathbf r_{k+1}^j$, $\mathbf{D}_{k+1}^j$, ${\mathbf{H}}_{k+1}^{j}$ as in (\ref{eq:residual}) and (\ref{e:iterative_error}); \\
    \!&$\ $Calculate ${\mathbf{J}}_{k+1}^{j}$ as in (\ref{eq:update}); \\
    \!&$\ $Calculate $\mathbf K_{k+1}^j$ and $\delta \mathbf x_j^o $ as in (\ref{e:ekf_update}); \\
    \!&$\ {\mathbf{x}}_{k+1|k+1}^{j+1}\!=\!{\mathbf{x}}_{k+1|k+1}^{j}\boxplus\! \delta \mathbf x_j^o$;\\
    \!&\!\!\!\!\textbf{end} \textbf{while}\\
    \!&\!\!\!\!$\mathbf P_{k+1}^j \! = \! (\mathbf I \! - \!  \mathbf K_{k+1}^j \mathbf H_{k+1}^j) \mathbf J_{k+1}^j \mathbf P_{k+1|k} (\mathbf J_{k+1}^j)^{T}$; \\
    \!&\!\!\!\!$\mathbf{x}_{k+1|k+1}={\mathbf{x}}_{k+1|k+1}^{j+1}$;\\
    \!&\!\!\!\!$\mathbf{P}_{k+1|k+1}=\mathbf{L}_{k+1} {\mathbf{P}}_{k+1}^{j} \mathbf{L}_{k+1}^T$; \\
    \bottomrule
    \end{tabular}
    \label{tab:algorithm_2}
    \vspace{-0.3cm}
\end{table}

\section{Integrating Manifolds into Kalman Filters and Toolkit Development}~\label{C++_imp}

% \subsection{Embedding manifold structures into Kalman filters}~\label{exa_matrix}

Shown in Sec.~\ref{esekfom}, the derived Kalman filter is formulated in symbolic representations and it is seen  
\iffalse
From the derivations in Sec.~\ref{esekfom}, it is seen
\fi
that each step of the Kalman filter is nicely separated into manifold constraints and system-specific behaviors. More specifically, state predict (\ref{pred0}) breaks into the manifold-specific operation $\oplus$ and system-specific part $\Delta t\mathbf{f}\left(\mathbf{x}, \mathbf{u}, \mathbf{w}\right)$, the two matrices $\mathbf{F}_{\mathbf{x}}$ and $\mathbf{F}_{\mathbf{w}}$ used in the covariance propagation (\ref{e:Cov_prop}) breaks into the manifold-specific parts $\mathbf{G}_{\mathbf{x}}$, $\mathbf{G}_{\mathbf{f}}$ and system-specific parts $ \frac{\partial\mathbf{f}\left(\mathbf{x} \boxplus \delta \mathbf{x},\mathbf{u}, \mathbf{0}\right)}{\partial \delta \mathbf{x}}|_{\delta \mathbf x= \mathbf 0}, \frac{\partial\mathbf{f}\left(\mathbf{x},\mathbf{u}, \mathbf{w} \right)}{\partial \mathbf{w} }|_{ \mathbf w= \mathbf 0}$. State update (\ref{e:ekf_update}) breaks into the manifold-specific operation $\boxplus$, manifold-specific part $\mathbf{J}_{k+1}^j$ and system-specific parts, i.e.,  $\mathbf{h}(\mathbf{x}, \mathbf v)$, $ \frac{\partial\left(\mathbf{h}(\mathbf{x}\boxplus \delta \mathbf{x}, \mathbf 0)\boxminus\mathbf{h}(\mathbf{x}, \mathbf 0)\right)}{\partial \delta \mathbf{x}}|_{\delta \mathbf x= \mathbf 0}$, and $ \frac{\partial\left(\mathbf{h}(\mathbf{x}, \mathbf v)\boxminus\mathbf{h}(\mathbf{x}, \mathbf 0)\right)}{\partial \mathbf{v}}|_{ \mathbf v= \mathbf 0}$.
And covariance reset only involves the manifold-specific part $\mathbf{L}_{k+1}$. Note that these system-specific descriptions are often easy to be derived even for robotic systems of high dimension (see Sec.~\ref{experiments}). 

%\subsection{Implementation of A $C++$ package}
The nice separation property between the manifold constraints and system-specific behaviors allows the integration of manifolds into the Kalman filter framework, and only leaves system-specific parts to be filled for specific systems. Moreover, enabled by the manifold composition in (\ref{e:compound_plus_minus}) and (\ref{diff_compound}), we only need to do so for simple {\it primitive manifolds} while those for larger {\it compound manifolds} can be automatically constructed. These two properties enabled us to develop a $C$++ toolkit that integrates the manifold-specific operations with a Kalman filter. With this toolkit, users need only to specify the manifold of state $\mathcal{M}_s$, measurement $\mathcal{M}_m$, and system-specific descriptions (i.e., function $\mathbf f, \mathbf h$ and their derivatives), and call the respective Kalman filter operations (i.e., predict and update) according to the current event (e.g., reception of an input or a measurement).

The current toolkit implementation is a full multi-rate iterated Kalman filter naturally operating on differentiable manifolds and is thus termed as {\it IKFoM}. Furthermore, it supports three {\it primitive manifolds}: $\mathbb{R}^n$, $SO(3)$ and $\mathbb{S}^2(r)$, but extendable to other types of {\it primitive manifolds} with proper definition of the operation $\boxplus\backslash\boxminus, \oplus$, and differentiations $\frac{\partial\left(\left(\left(\mathbf{x}\boxplus\mathbf{u}\right)\oplus \mathbf{v}\right)\boxminus\mathbf{y}\right)}{\partial \mathbf{u}}$, $\frac{\partial\left(\left(\left(\mathbf{x}\boxplus\mathbf{u}\right)\oplus \mathbf{v}\right)\boxminus\mathbf{y}\right)}{\partial \mathbf{v}}$. The toolkit is open sourced and more details about the implementation can be found at~\url{https://github.com/hku-mars/IKFoM}.

\section{Experiments}~\label{experiments}

In this section, we apply our developed Kalman filter framework and toolkit implementations to a tightly-coupled lidar-inertial navigation system taken from \cite{xu2020fast}. The overall system, shown in Fig. \ref{fig:conf_lid}, consists of a solid-state lidar (Livox AVIA) with a built-in IMU and an onboard computer. The lidar provides a typical scan rate of $10Hz-100Hz$ and $200Hz$ gyro and accelerometer measurements. Unlike conventional spinning lidars (e.g., Velodyne lidars), the Livox AVIA has only 70$^\circ$ Field of View (FoV), making the lidar-inertial odometry rather challenging. The onboard computer is configured with a $1.8 GHz$ quad-core Intel i7-8550U CPU and $8 GB$ RAM. Besides the original state estimation problem considered in \cite{xu2020fast}, we further consider the online estimation of the extrinsic between the lidar and IMU. 
\subsection{System modeling}
The global frame is denoted as $G$ (i.e. the initial IMU frame), the IMU frame is taken as the body frame (denoted as $I$), and the lidar frame is denoted as $L$. Assuming the lidar is rigidly attached to the IMU with an unknown extrinsic $^I\mathbf{T}_L=\left({}^I\mathbf{R}_L, {}^I\mathbf{p}_L\right)$, the objective of this system is to 1) estimate kinematics states of the IMU including its position ($^G\mathbf{p}_I$), velocity ($^G\mathbf{v}_I$), and rotation ($^G\mathbf{R}_I\in SO(3)$) in the global frame; 2) estimate the biases of the IMU (i.e., $\mathbf{b}_\mathbf{a}$ and $\mathbf{b}_{\bm \omega}$; 3) estimate the gravity vector ($^G\mathbf{g}$) in the global frame; 4) estimate the extrinsic $^I\mathbf{T}_L=\left({}^I\mathbf{R}_L, {}^I\mathbf{p}_L\right)$ online; and 5) build a global point cloud map of the observed environment. 
\begin{figure}[t]
\setlength{\abovecaptionskip}{-0.1cm} 
\setlength{\belowcaptionskip}{-0.2cm} 
	\centering
	\includegraphics[width=1.0\columnwidth]{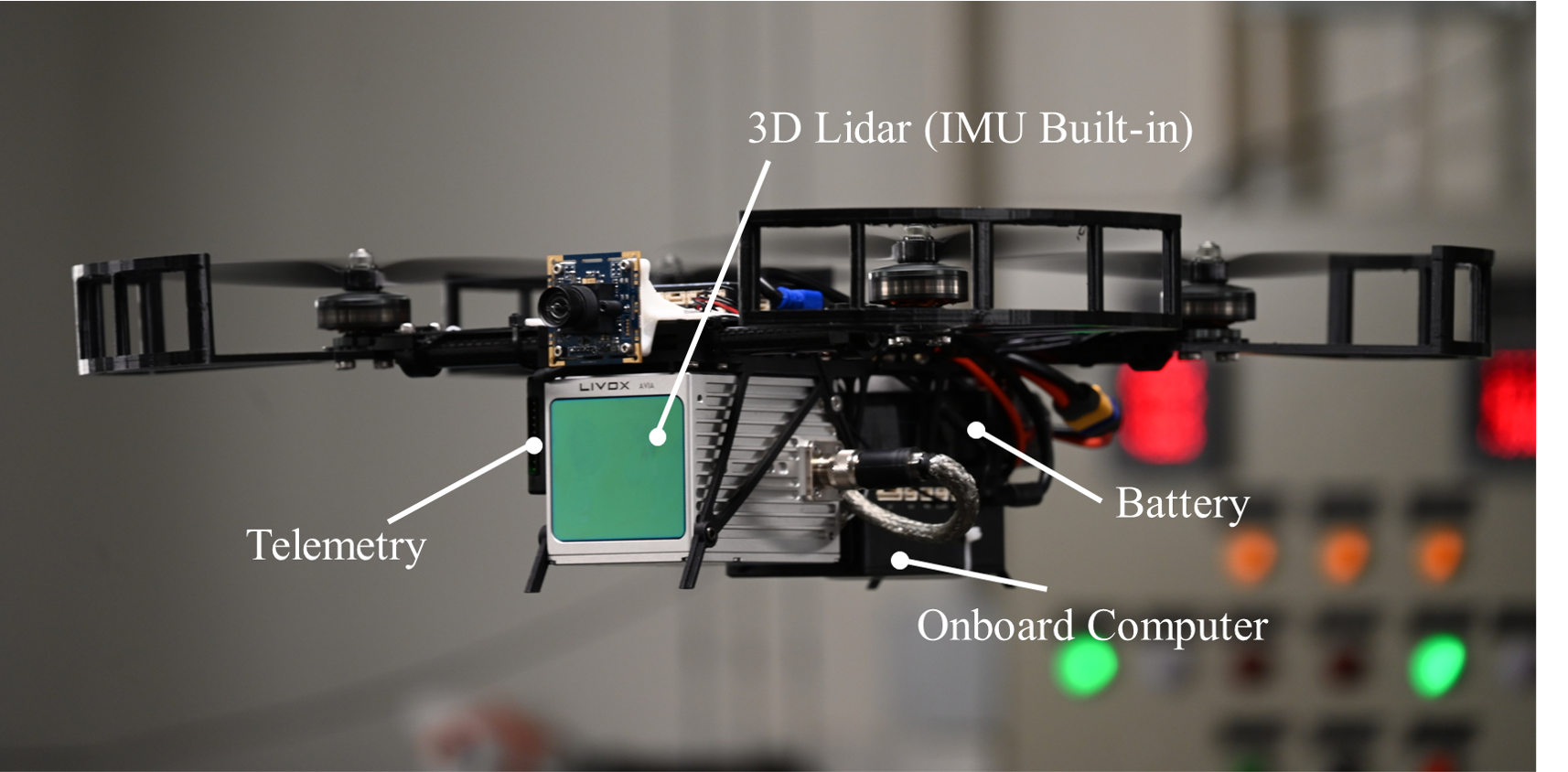}
	\caption{Configuration of the lidar-inertial system from \cite{xu2020fast}: A small scale ($280mm$ wheelbase) quadrotor UAV carrying a Livox AVIA lidar and a DJI Manifold 2C computer. The onboard camera is for visualization only.
		\label{fig:conf_lid}}
	\vspace{-0.3cm}
\end{figure}

Augmenting the state formulation in \cite{xu2020fast} with the lidar-IMU extrinsic, we have:
\begin{equation}~\label{e:system_model_2}
\setlength{\abovedisplayskip}{0.15cm} 
\setlength{\belowdisplayskip}{0.15cm} 
\begin{aligned}
&^G\dot{\mathbf{p}}_I = {}^G\mathbf{v}_I, \ ^G\dot{\mathbf{v}}_I = {}^G\mathbf{R}_I\left(\mathbf{a}_m-\mathbf{b}_\mathbf{a}-\mathbf{n}_\mathbf{a}\right)+{}^G\mathbf{g}\\
&^G\dot{\mathbf{R}}_I = {}^G\mathbf R_I \lfloor \bm\omega_m\!-\!\mathbf{b}_{\bm\omega}\!-\!\mathbf{n}_{\bm\omega} \rfloor,\ \dot{\mathbf{b}}_{\bm\omega} = \mathbf{n}_{\mathbf{b}_{\bm\omega}}, \ \dot{\mathbf{b}}_\mathbf{a} = \mathbf{n}_{\mathbf{b}_\mathbf{a}}\\
&{}^G\dot{\mathbf{g}}=\mathbf{0},\ {}^I\dot{\mathbf{R}}_L=\mathbf{0},\ {}^I\dot{\mathbf{p}}_L=\mathbf{0}
\end{aligned}
\end{equation}
where $\mathbf{a}_m$, $\bm\omega_m$ are measurements of IMU, $\mathbf{n}_\mathbf{a}$,
$\mathbf{n}_{\bm\omega}$ are IMU noises, $\mathbf{n}_{\mathbf{b}_{\bm\omega}}$ and $\mathbf{n}_{\mathbf{b}_\mathbf{a}}$ are zero mean Gaussian white noises that drive the IMU biases $\mathbf{b}_{\bm\omega}$ and $\mathbf{b}_{\mathbf{a}}$ respectively. The gravity vector $^G\mathbf{g}$ is of fixed length $9.81 m/s^2$. 

The measurement model is identical to \cite{xu2020fast}: for a new scan of lidar raw points, we extract the plane and edge points (i.e., feature points) based on the local curvature~\cite{lin2020loam}. Then for a measured feature point $^L\mathbf{p}_{f_i}, i=1,...,m$, its true location in the global frame should lie on the corresponding plane (or edge) in the map built so far. More specifically, we represent the corresponding plane (or edge) in the map by its normal direction (or direction of the edge) $\mathbf u_i$ and a point ${}^G\mathbf{q}_{i}$ lying on the plane (or edge). Since the point $^L\mathbf{p}_{f_i}, i=1,...,m$ is measured in the lidar local frame (thus denoted as $L$) and contaminated by measurement noise $\mathbf n_i$, the true point location in the global frame is ${}^G\mathbf{T}_I{}^I\mathbf{T}_L \left({}^L\mathbf{p}_{f_i} - \mathbf n_i \right)$. Since this true location lies on the plane (or edge) defined by $\mathbf u_i$ and ${}^G\mathbf{q}_{i}$, its distance to the plane (or edge) should be zero, i.e.,
\begin{equation}~\label{e:lidar-inertial-output}
\setlength{\abovedisplayskip}{0.15cm} 
\setlength{\belowdisplayskip}{0.15cm} 
\mathbf{G}_i\left({}^G\mathbf{T}_I{}^I\mathbf{T}_L \left(  {}^L\mathbf{p}_{f_i} - \mathbf n_i \right)-{}^G\mathbf{q}_{i}\right) \!=\! \mathbf 0, \ i = 1, \cdots, m
\end{equation}
where $\mathbf{G}_i=\mathbf{u}_i^T$ for a planar feature and $\mathbf{G}_i=\lfloor\mathbf{u}_i\rfloor$ for an edge feature. This equation defines an implicit measurement model which relates the measurement ${}^L\mathbf{p}_{f_i}$, measurement noise $\mathbf n_i$, and the ground-truth state ${}^G\mathbf{T}_I$ and ${}^I\mathbf{T}_L$. 

To obtain $\mathbf u_i, {}^G\mathbf{q}_{i}$ of the corresponding plane (or edge) in the map, we use the state estimated at the current iteration to project the feature point $^L\mathbf{p}_{f_i}$ to the global frame and find the closest five feature points (of the same type) in the map built so far. After convergence of the iterated Kalman filter, the optimal state update is used to project the feature point $^L\mathbf{p}_{f_i}$ to the global frame and append it to the map. The updated map is finally used in the next step to register new scans. 

\subsection{Canonical representation:}
Using the zero-order hold discretization described in Sec.~\ref{cano_repres}, which is based on the operation $\oplus$, the system with state model (\ref{e:system_model_2}) and measurement model (\ref{e:lidar-inertial-output}) can be discretized and cast into the canonical form (\ref{cano}) as follows:
\begin{equation}\label{e:example_system_description}
\setlength{\abovedisplayskip}{0.15cm} 
\setlength{\belowdisplayskip}{0.15cm} 
    \begin{aligned}
    & \mathcal{M}_s = \mathbb{R}^3 \! \times \! \mathbb{R}^3 \! \times \! SO(3) \! \times \! \mathbb{R}^3  \! \times \! \mathbb{R}^3 \! \times \! \mathbb{S}^2 \! \times \! SO(3) \times \! \mathbb{R}^3, \\
    &\mathcal{M}_m=\underbrace{\mathbb{R}^{1}\times\cdots\times\mathbb{R}^1 \times \mathbb{R}^1 \times \cdots \times \mathbb{R}^1}_{m},\\
    & \mathbf{x}^T = \begin{bmatrix}^G\mathbf{p}_{I}\!\!
    &\!\!^G\mathbf{v}_{I}\!\!
    &\!\!^G\mathbf{R}_{I}\!\!
    &\!\!\mathbf{b}_{{\mathbf{a}}}\!\!
    &\!\!\mathbf{b}_{{\bm\omega}}\!\!
    &\!\!^G\mathbf{g}\!\!&\!\!^I\mathbf{R}_L\!\!&\!\!^I\mathbf{p}_L
    \end{bmatrix}\!\!,\\
    &\mathbf{u}^T=\begin{bmatrix}\mathbf{a}_{m}&\bm\omega_{m}\end{bmatrix}\\
    &\mathbf{f}\left(\mathbf{x},\mathbf{u}, \mathbf{w}\right)^T =\left[\begin{matrix}^G\mathbf{v}_{I}\!\!&\!\!^G\mathbf{R}_{I}  \!\left(\mathbf{a}_{m}\!-\!\mathbf{b}_{\mathbf{a}}\!-\!\mathbf{n}_{\mathbf{a}} \right)\!+\!{}^G\mathbf{g}\!\!\end{matrix}\right.\\
    &\quad\quad\quad\quad\left.\begin{matrix}\!\!\bm\omega_{m}\!-\!\mathbf{b}_{{\bm\omega}}\!-\!\mathbf{n}_{\bm\omega}&\mathbf{n}_{\mathbf{b}_{{\mathbf{a}}}}&\mathbf{n}_{\mathbf{b}_{{\bm\omega}}}&\mathbf{0}&\mathbf{0}&\mathbf{0}\end{matrix}\right],\\
    &\mathbf{h}_i \!\left(\mathbf{x},\! \mathbf{v}\right)^T = \mathbf{G}_i\left({}^G\mathbf{T}_I{}^I\mathbf{T}_L \left(  {}^L\mathbf{p}_{f_i} - \mathbf n_i \right)-{}^G\mathbf{q}_{i}\right),\\
    &\mathbf{w}^T =\begin{bmatrix}\mathbf{n}_{\mathbf{a}}&\mathbf{n}_{{\bm\omega}}&\mathbf{n}_{\mathbf{b}_{\mathbf{a}}} &\mathbf{n}_{\mathbf{b}_{{\bm\omega}}}\end{bmatrix},\\
    & \mathbf{v}^T = \begin{bmatrix}\cdots&\mathbf{n}_i &\cdots\end{bmatrix}, i=1,...,m.
    \end{aligned}
\end{equation}
with equivalent measurement $\mathbf z$ being constantly zero. 

Accordingly, the system-specific partial differentiations, including $\left.\frac{\partial\mathbf{f}\left(\mathbf{x} \boxplus \delta\mathbf{x},\mathbf{u},\mathbf{0}\right)}{\partial \delta \mathbf{x}}\right|_{\delta \mathbf{x} = \mathbf{0}}$, $\left.\frac{\partial\mathbf{f}\left(\mathbf{x},\mathbf{u},\mathbf{w}\right)}{\partial\mathbf{w}}\right|_{\mathbf{w}=\mathbf{0}}$ and $\frac{\partial \left(\mathbf{h}\left(\mathbf{x} \boxplus \delta \mathbf{x}, \mathbf 0\right)\boxminus\mathbf{h}\left(\mathbf{x}, \mathbf 0\right)\right)}{ \partial \delta \mathbf{x}}|_{\delta \mathbf{x}=\mathbf{0}}$, $\frac{\partial\left(\mathbf{h}(\mathbf{x}, \mathbf v)\boxminus\mathbf{h}(\mathbf{x}, \mathbf 0)\right)}{\partial \mathbf{v}}|_{ \mathbf v= \mathbf 0}$ can be easily calculated as follows:

Partial differentiations for $\mathbf f(\mathbf x, \mathbf u, \mathbf w)$: % in (\ref{e:example_system_description}):

\begin{equation}\label{e:partial_F_sup}
\setlength{\abovedisplayskip}{0.15cm} 
\setlength{\belowdisplayskip}{0.15cm} 
\begin{aligned}
    \hspace{-0.3cm}\left.\frac{\partial\mathbf{f}\left(\mathbf{x} \boxplus \delta \mathbf{x},\mathbf{u},\mathbf{0}\right)}{\partial \delta \mathbf{x}}\right|_{\delta \mathbf{x} = \mathbf{0}}\!\!&=\!\!\begin{bmatrix}\mathbf{0}&\mathbf{I}&\mathbf{0}&\mathbf{0}&\mathbf{0}&\mathbf{0}&\mathbf{0}&\mathbf{0}\\
    \mathbf{0}&\mathbf{0}&\mathbf{U}^F_{23}&-^G\mathbf{R}_{I}&\mathbf{0}&\mathbf{U}_{26}^F&\mathbf{0}&\mathbf{0}\\
    \mathbf{0}&\mathbf{0}&\mathbf{0}&\mathbf{0}&-\mathbf{I}&\mathbf{0}&\mathbf{0}&\mathbf{0}\\
    \mathbf{0}&\mathbf{0}&\mathbf{0}&\mathbf{0}&\mathbf{0}&\mathbf{0}&\mathbf{0}&\mathbf{0}\\
    \mathbf{0}&\mathbf{0}&\mathbf{0}&\mathbf{0}&\mathbf{0}&\mathbf{0}&\mathbf{0}&\mathbf{0}\\
    \mathbf{0}&\mathbf{0}&\mathbf{0}&\mathbf{0}&\mathbf{0}&\mathbf{0}&\mathbf{0}&\mathbf{0}\\
    \mathbf{0}&\mathbf{0}&\mathbf{0}&\mathbf{0}&\mathbf{0}&\mathbf{0}&\mathbf{0}&\mathbf{0}\\
    \mathbf{0}&\mathbf{0}&\mathbf{0}&\mathbf{0}&\mathbf{0}&\mathbf{0}&\mathbf{0}&\mathbf{0}\end{bmatrix}\\
    \left.\frac{\partial\mathbf{f}\left(\mathbf{x},\mathbf{u},\mathbf{w}\right)}{\partial\mathbf{w}}\right|_{\mathbf{w}=\mathbf{0}}&=\begin{bmatrix}\mathbf{0}&\mathbf{0}&\mathbf{0}&\mathbf{0}\\
-{}^G\mathbf{R}_{I}&\mathbf{0}&\mathbf{0}&\mathbf{0}\\
    \mathbf{0}&-\mathbf{I}&\mathbf{0}&\mathbf{0}\\
    \mathbf{0}&\mathbf{0}&\mathbf{I}&\mathbf{0}\\
    \mathbf{0}&\mathbf{0}&\mathbf{0}&\mathbf{I}\\
    \mathbf{0}&\mathbf{0}&\mathbf{0}&\mathbf{0}\\
    \mathbf{0}&\mathbf{0}&\mathbf{0}&\mathbf{0}\\
    \mathbf{0}&\mathbf{0}&\mathbf{0}&\mathbf{0}\end{bmatrix}
    \end{aligned}
\end{equation}
where $\mathbf{U}^F_{23}=-^G\mathbf{R}_{I}\lfloor\mathbf{a}_{m}-\mathbf{b}_{\mathbf{a}}\rfloor$ and $\mathbf{U}_{26}^F=-\lfloor{}^G\mathbf{g}\rfloor\mathbf{B}({}^G\mathbf{g})$, $\mathbf{B}(\cdot)$ is defined in the equation (\ref{e:Rx_sup}). 

And partial differentiations for $\mathbf h(\mathbf x, \mathbf v)$: %in (\ref{e:example_system_description}):
\begin{equation}~\label{e:H_k+1_2_sup}
\setlength{\abovedisplayskip}{0.15cm} 
\setlength{\belowdisplayskip}{0.15cm} 
\begin{aligned}
    &\frac{\partial \left(\mathbf{h}\left(\mathbf{x} \boxplus \delta \mathbf{x}, \mathbf 0\right)\boxminus\mathbf{h}\left(\mathbf{x}, \mathbf 0\right)\right)}{ \partial \delta \mathbf{x}}|_{\delta \mathbf{x}=\mathbf{0}} = \\
    &\quad \quad \quad  \begin{bmatrix}\vdots&\vdots&\vdots&\vdots&\vdots&\vdots&\vdots&\vdots\\
    \mathbf{G}_i&\mathbf{0}&\mathbf{U}^H_{i3}&\mathbf{0}&\mathbf{0}&\mathbf{0}&\mathbf{U}_{i7}^H&\mathbf{G}_i{}^G\mathbf{R}_I\\
    \vdots&\vdots&\vdots&\vdots&\vdots&\vdots&\vdots&\vdots  \end{bmatrix}, \\
    & \frac{\partial\left(\mathbf{h}(\mathbf{x}, \mathbf v)\boxminus\mathbf{h}(\mathbf{x}, \mathbf 0)\right)}{\partial \mathbf{v}}|_{ \mathbf v= \mathbf 0} = \text{diag}(\cdots, - \mathbf G_i {}^G\mathbf{R}_I {}^I\mathbf{R}_L, \cdots)
\end{aligned}
\end{equation}
where $\mathbf{U}_{i3}^H=-\mathbf{G}_i {}^G\mathbf{R}_I\lfloor{}^I\mathbf{T}_L{}^L\mathbf{p}_{f_{i}} \rfloor$, and $\mathbf{U}_{i7}^H=-\mathbf{G}_i{}^G\mathbf{R}_I{}^I\mathbf{R}_L\lfloor{}^L\mathbf{p}_{f_i}\rfloor$.

The partial differential of $\left.\frac{(\mathbf{x}\boxplus\mathbf{u})\cdot\mathbf{a}}{\partial\mathbf{u}}\right|_{\mathbf{u}=\mathbf{0}}$ with $\mathbf{x}\in SO(3)$ and $\mathbf{a}\in\mathbb{R}^3$ is shown in Lemma~\ref{lemma:pd_SO3} in Appx.~\ref{app:pdiff_on_SO3}, which is utilized in the above procedure.

Supplying the canonical representation of the system (\ref{e:example_system_description}) and the respective partial differentiations (\ref{e:partial_F_sup}, \ref{e:H_k+1_2_sup}) to our toolkit leads to a tightly-coupled lidar-inertial navigation system.

\subsection{Experiment results}

We verify the tightly-coupled lidar-inertial navigation system implemented by our toolkit in three different scenarios, i.e., indoor UAV flight, indoor quick-shake experiment, and outdoor random walk experiment. They are denoted as V1, V2 and V3 respectively. For each scenario, we test the implementation on two trials of data, one collected by ourselves and the other from the original paper \cite{xu2020fast}. The six datasets are denoted as V1-01, V1-02, V2-01, V2-02, V3-01 and V3-02, respectively. In all experiments, the maximal number of iteration in the iterated Kalman filter (see Algorithm 1) is set to 4, i.e., $N_{\rm max} = 4$. We compare our on-manifold Kalman filter {\it IKFoM} with a quaternion-based iterated Kalman filter. In the quanterion-based filter, a rotation is parameteirzed by a quaternion viewed as a vector in $\mathbb{R}^4$, and the gravity is parameterized as a vector in $\mathbb{R}^3$. The quaternion and gravity estimates are normalized after each predict or update, and additional measurement equations due to the normalization constraints are added. For the sake of space limit, we only show the results of the data trial collected in this work, i.e., V1-01, V2-01 and V3-01.

\begin{figure}[t]
\vspace{-0.3cm}
\setlength{\abovecaptionskip}{-0.1cm} 
\setlength{\belowcaptionskip}{-0.2cm} 
	\centering
	\includegraphics[width=1\columnwidth]{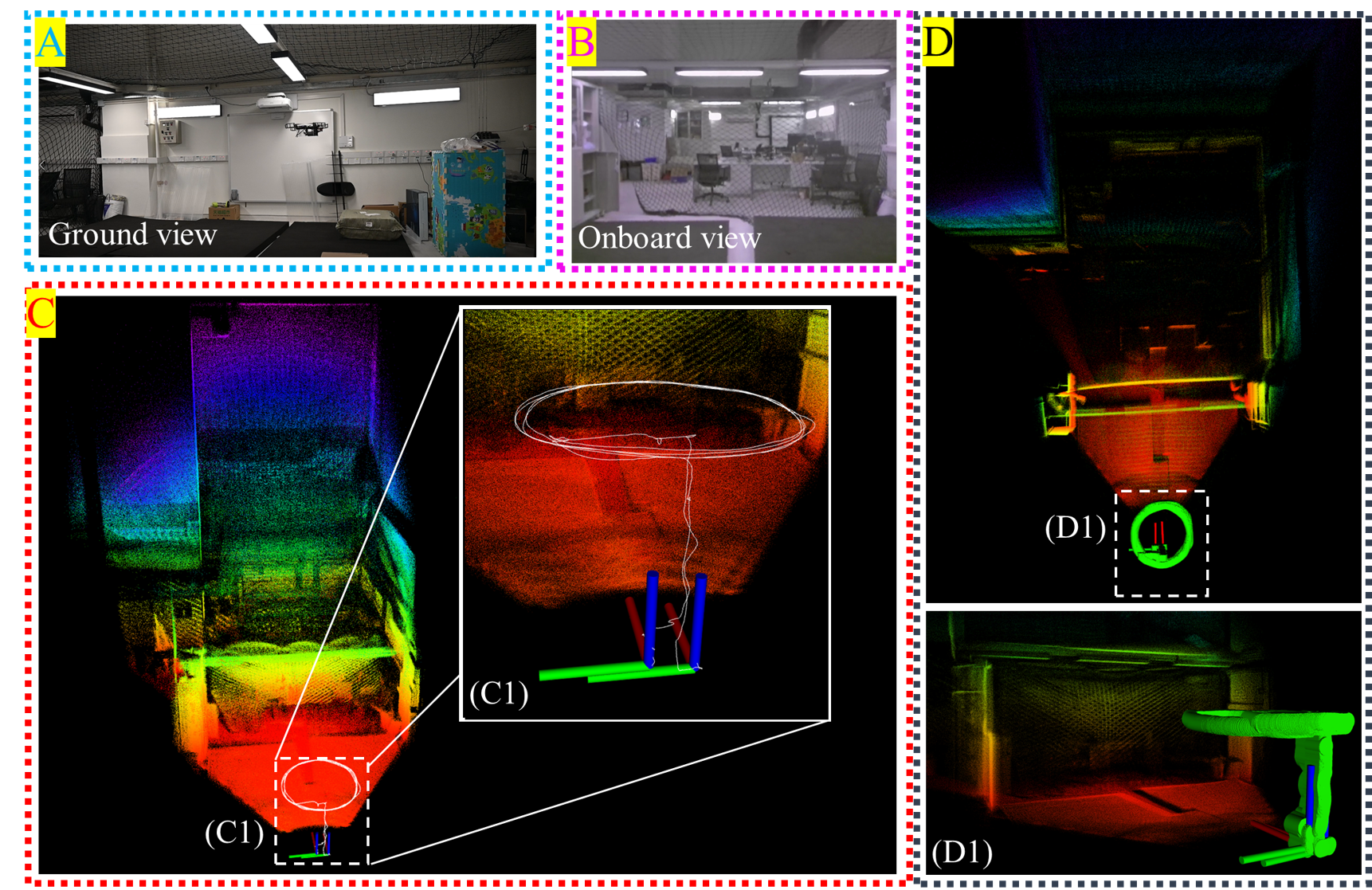}
	\caption{Real-time mapping results of dataset V1-01. A: Photo from ground view; B: Snapshot of onboard FPV video; C: Map result of {\it IKFoM}, (C1) Trajectory and poses of the UAV at beginning and end of the experiment; D: Map result of quaternion-based iterated Kalman filter, (D1) Trajectory and poses of the UAV at beginning and end of the experiment.	\label{fig:lid_map_circle_sup}}
	\vspace{-0.3cm}
\end{figure}

\begin{figure}[b]
\vspace{-0.3cm}
\setlength{\abovecaptionskip}{-0.1cm} 
\setlength{\belowcaptionskip}{-0.2cm} 
	\centering
	\includegraphics[width=1\columnwidth]{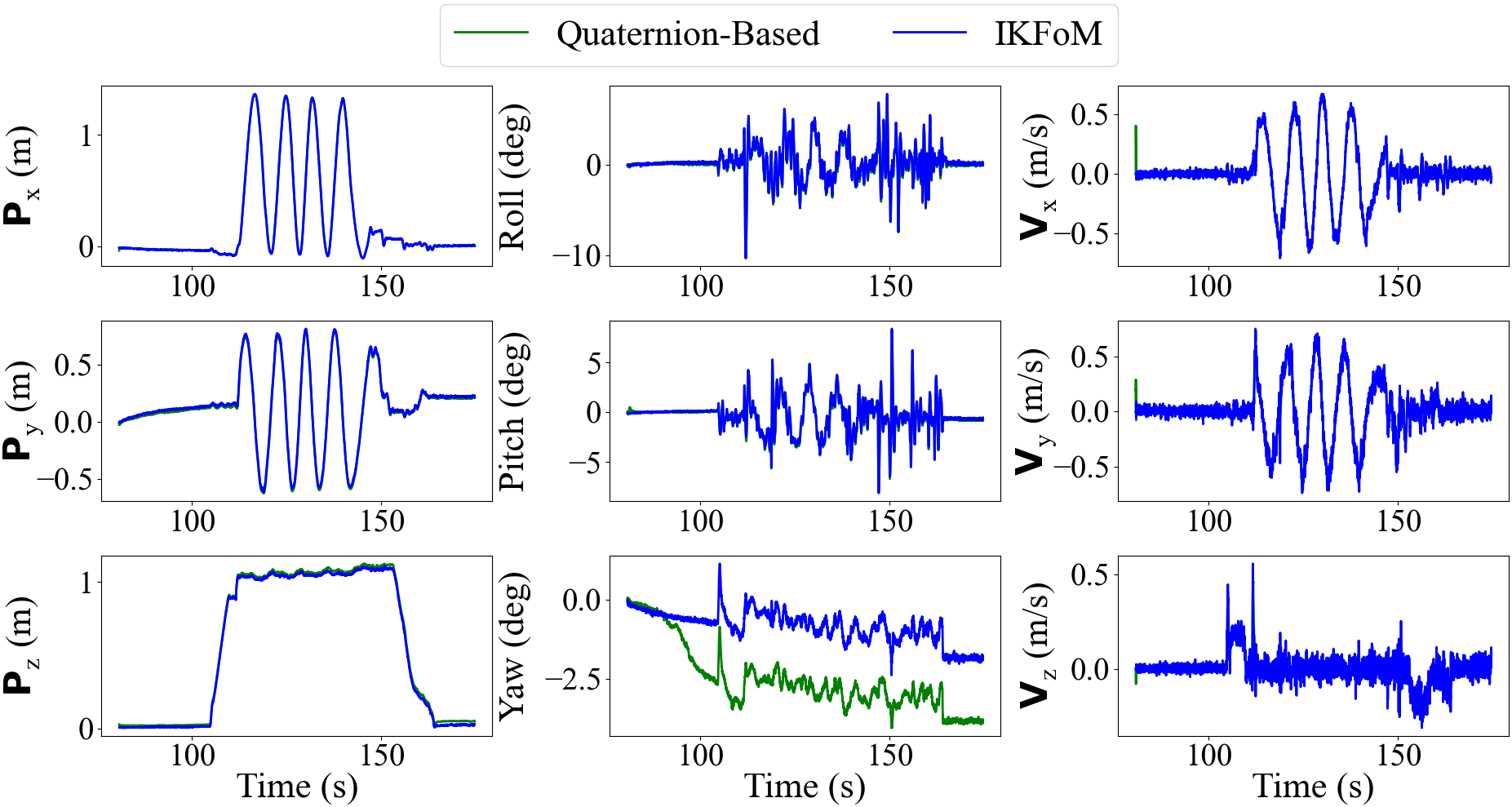}
	\caption{Comparison of estimates of quaternion-based iterated Kalman filter and {\it IKFoM} for the position, rotation in Euler angles and velocity of the IMU on dataset V1-01.
		\label{fig:lid_kin_circle_sup}}
	\vspace{-0.3cm}
\end{figure}

\subsubsection{Indoor UAV flight}

For the dataset V1-01, the experiment is conducted in an indoor environment (see Fig.~\ref{fig:lid_map_circle_sup} (A)) where the UAV took off from the ground and flied in a circle path. During the path following, the UAV is constantly facing at a cluttered office area behind a safety net (see Fig.~\ref{fig:lid_map_circle_sup} (B)).  After the path following, a human pilot took over the UAV and landed it manually to ensure that the landing position coincides with the take-off point.

Fig.~\ref{fig:lid_map_circle_sup} (C) shows the real-time mapping results overlaid with the 3D trajectory estimated by our toolkit. It can be seen that our toolkit achieves consistent mapping even in the cluttered indoor environment. The position drift is less than $0.7825\%$ (i.e., $0.2309m$ drift over the $29.5168m$ path, see Fig. \ref{fig:lid_map_circle_sup} (C1)). This drift is caused, in part by the accumulation of odometry error, which is common in SLAM systems, and in part by inaccurate manual landing. In addition, the lidar-inertial navigation system is implemented using quaternion-based iterated Kalman filter as comparison, the mapping results are shown in Fig.~\ref{fig:lid_map_circle_sup} (D). The results indicate that the point cloud map of the observed scene is constructed properly in this scenario and the estimated pose at the end of the experiment coincides with the starting pose greatly. And the position drift is $0.2382m$ versus $0.2309m$ of {\it IKFoM}.

We show the estimated trajectory of position (${}^G{\mathbf{p}}_I$), rotation (${}^G{\mathbf{R}}_I$), and velocity ($^G{\mathbf{v}}_I$) obtained by {\it IKFoM} as well as the quaternion-based iterated Kalman filter in Fig.~\ref{fig:lid_kin_circle_sup}, where the experiment starts from $80.7994s$ and ends at $174.6590s$. Our method achieves smooth state estimation that is suitable for onboard feedback control. All the estimated state variables agree well with the actual motions. The estimates of the quaternion-based iterated Kalman filter have a good coincidence with our method excepting the estimated Yaw angle. For further experiment demonstration, we refer readers to the videos at~\url{https://youtu.be/sz\_ZlDkl6fA}.

\begin{figure}[b]
\vspace{-0.3cm}
\setlength{\abovecaptionskip}{-0.1cm} 
\setlength{\belowcaptionskip}{-0.2cm} 
	\centering
	\includegraphics[width=1\columnwidth]{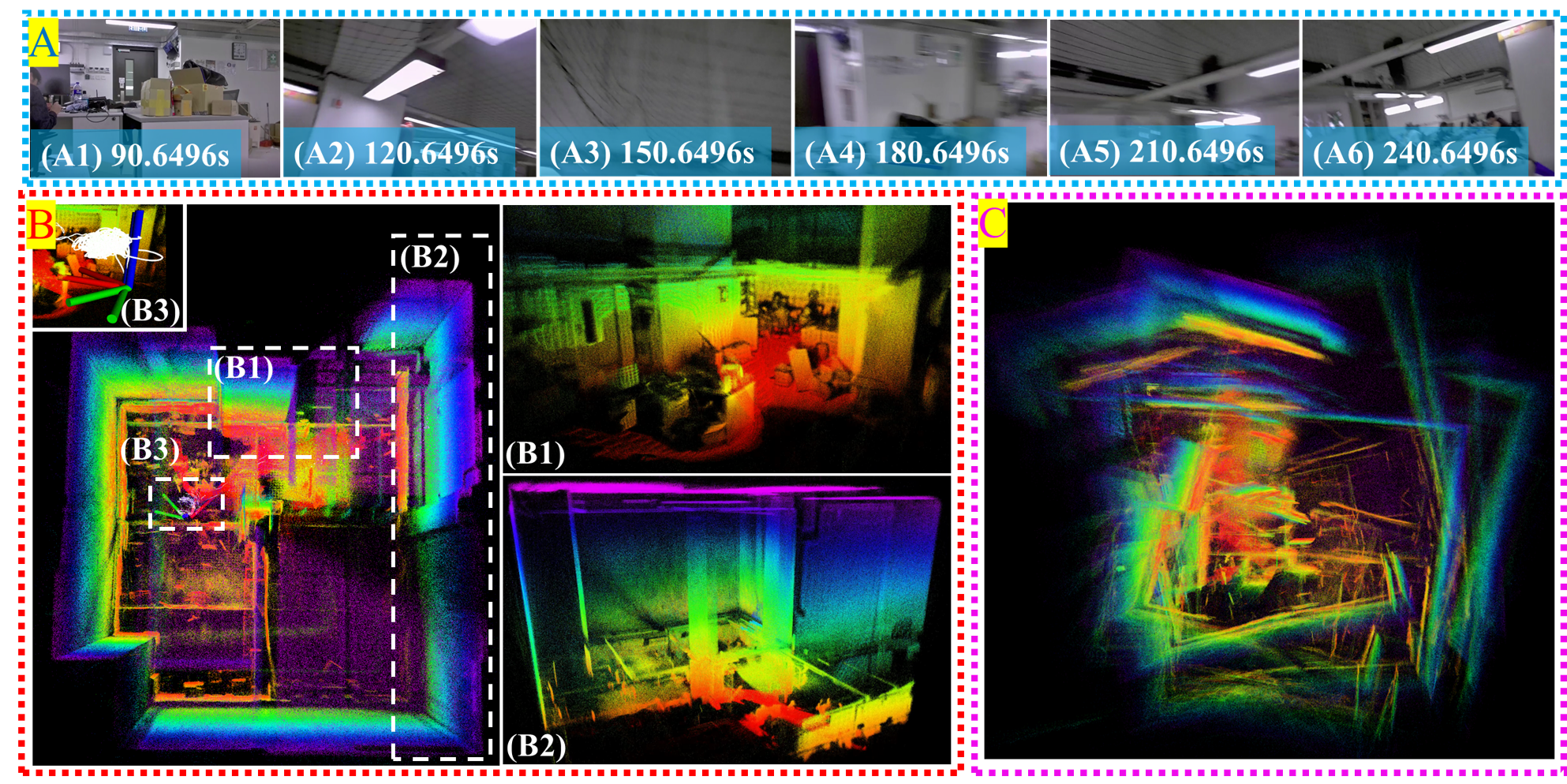}
	\caption{Real-time mapping results of dataset V2-01. A: Snapshots of onboard FPV video; B: Map result of {\it IKFoM}, (B1) Local zoom-in map, (B2) Side view of the map, (B3) Poses of the UAV at the beginning and end of the experiment; C: Map result of quaternion-based iterated Kalman filter.	\label{fig:lid_map_shake}}
	\vspace{-0.3cm}
\end{figure}

\begin{figure}[t]
\vspace{-0.3cm}
\setlength{\abovecaptionskip}{-0.1cm} 
\setlength{\belowcaptionskip}{-0.2cm} 
\centering
	\includegraphics[width=1.0\columnwidth]{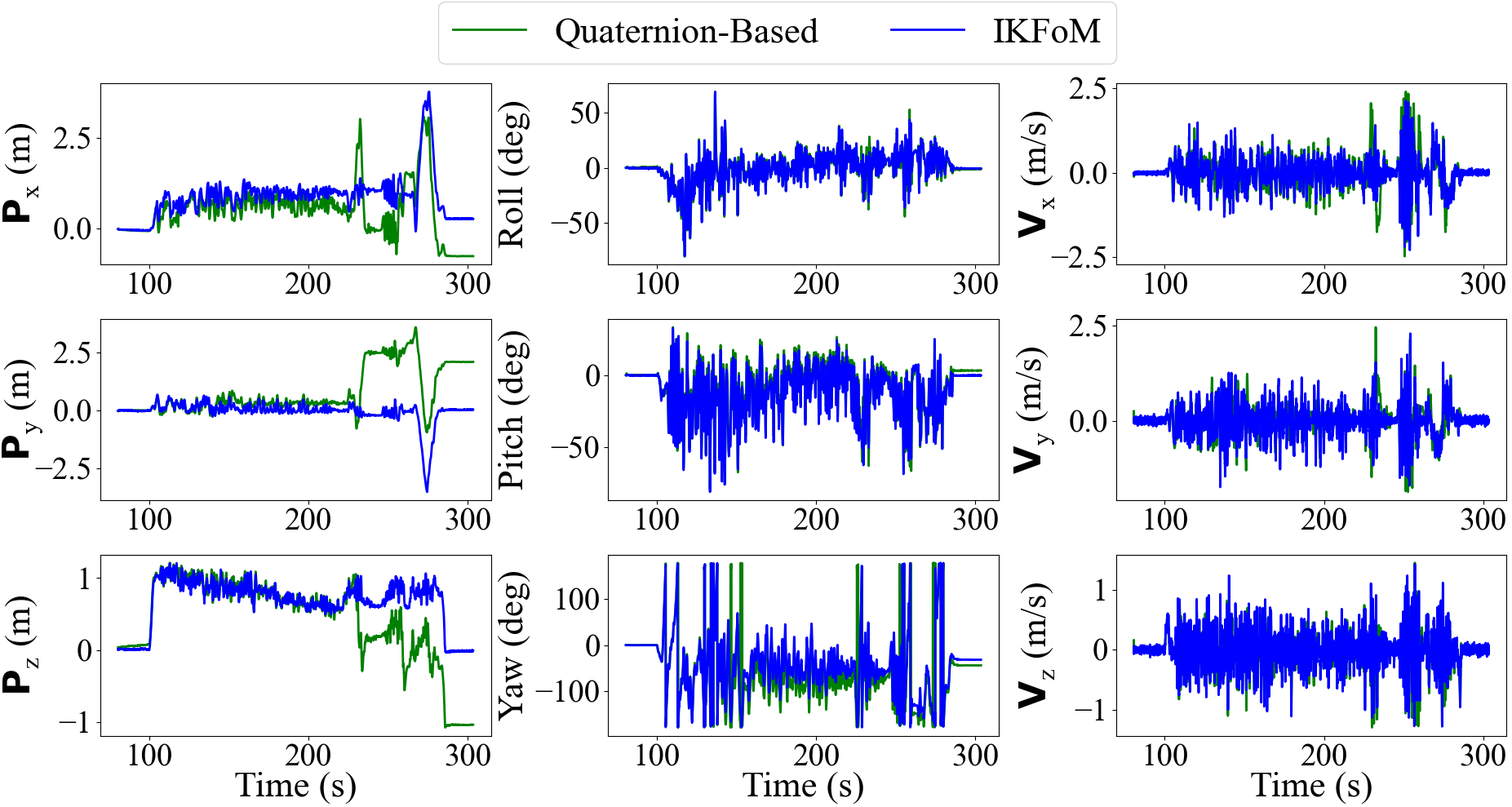}
	\caption{Comparison of estimations of quaternion-based iterated Kalman filter and {\it IKFoM} for position, rotation in Euler angles and velocity of the IMU on dataset V2-01, an indoor quick shake experiment.
		\label{fig:lid_kin_shake}}
	\vspace{-0.3cm}
\end{figure}

\subsubsection{Indoor quick shake}
The quick-shake experiment in V2-01 is conducted in a cluttered office area (see Fig.~\ref{fig:lid_map_shake} (A)). In the experiment, the UAV shown in Fig.~\ref{fig:conf_lid} is handheld and quickly shaken, creating a large rotation up to $356.85 deg/s$. The UAV ends at the starting position to enable the computation of odometry drift. 
\iffalse
During the experiment, the Vicon is used to provide the ground truth of position, velocity and rotation. 
\fi

Fig.~\ref{fig:lid_map_shake} (B) shows the real-time mapping result of {\it IKFoM} on dataset V2-01. It is seen that our system achieves consistent mapping even in fast rotational movements that are usually challenging for visual-inertial odometry due to image defocus and/or motion blur (see Fig.~\ref{fig:lid_map_shake} (A4) and (A5)). As shown in Fig.~\ref{fig:lid_map_shake} (B3), the estimated final position of the UAV coincides with the beginning position, leading to a position drift less than $0.2364\%$ (i.e., $0.2827m$ drift over $119.58m$ path). As comparison, the mapping results of the quaternion-based iterated Kalman filter are shown in Fig~\ref{fig:lid_map_shake} (C). This method fails to construct the map due to the improper propagation of the state covariance, for which, fast rotational movements amplify the phenomenon. 

Blue trajectories in Fig.~\ref{fig:lid_kin_shake} show {\it IKFoM} estimates of position ($^G{\mathbf{p}}_I$), rotation ($^G{\mathbf{R}}_I$) in Euler angles and velocity ($^G{\mathbf{v}}_I$) of the IMU, where the experiment starts from $80.5993s$ and ends at $303.499s$. Those estimates are changing in a high frequency, which is consistent with the actual motions of the UAV. The noticeable translation around $275s$ is the actual UAV motion. As comparison, the estimations of the quaternion-based iterated Kalman filter for those parameters are depicted in green lines. As can be seen, the position estimate of this method does not return to the initial point and leads to a much larger drift ($2.4837m$ versus $0.2827m$ of {\it IKFoM}) finally. We refer readers to the video at~\url{https://youtu.be/sz\_ZlDkl6fA} for further experiment demonstration. 

\begin{figure}[tb]
\vspace{-0.3cm}
\setlength{\abovecaptionskip}{-0.1cm} 
\setlength{\belowcaptionskip}{-0.2cm} 
	\centering
	\includegraphics[width=1\columnwidth]{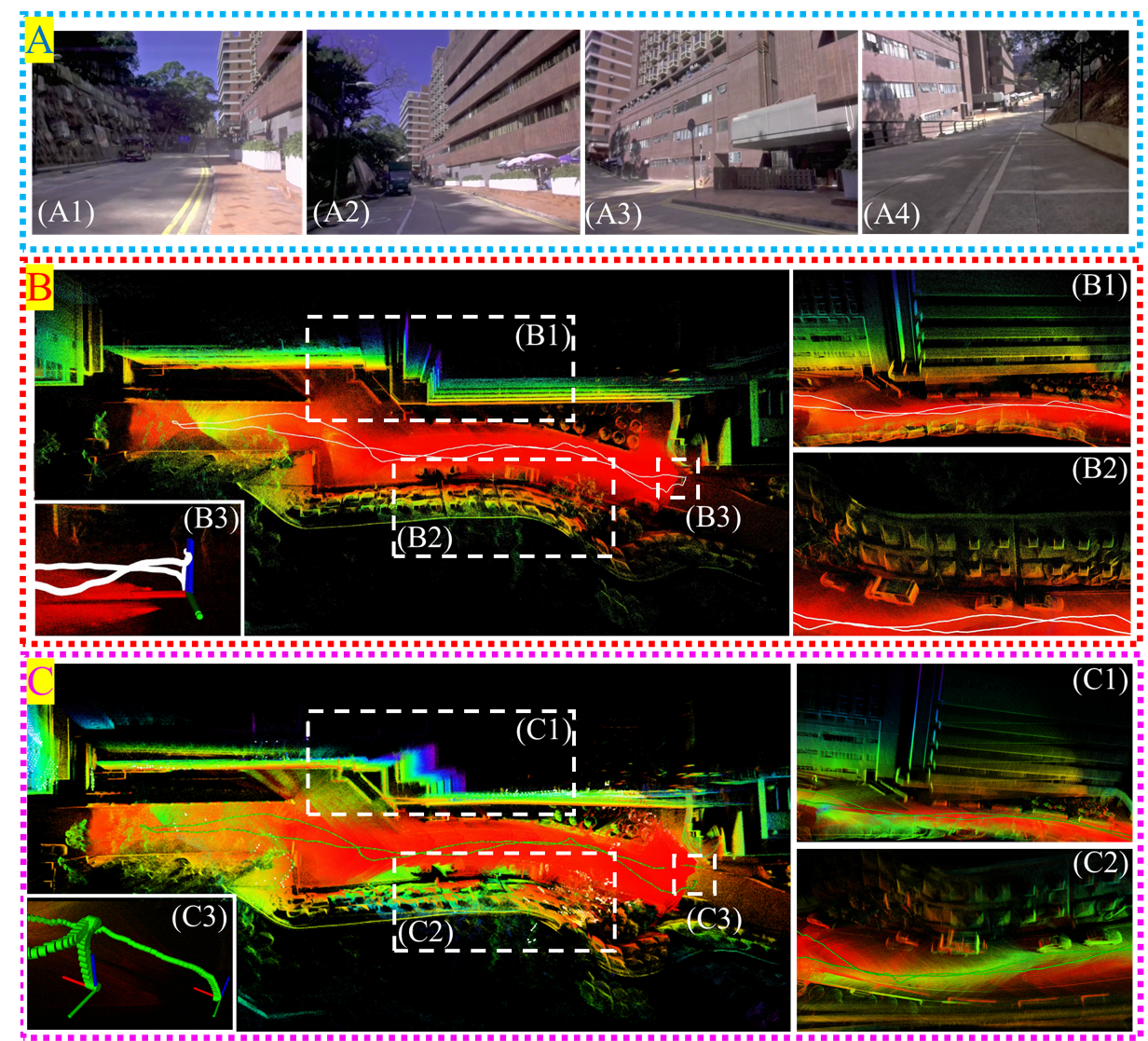}
	\caption{Real-time mapping results of dataset V3-01. A: Photos of the environment of this experiment; B: Map result of {\it IKFoM}, (B1) Local zoom-in map result of one side of the road, (B2) Local zoom-in map result of the other side of the road, (B3) Poses of the UAV at the beginning and end of the experiment; C: Map result of quaternion-based iterated Kalman filter, (C1) Local zoom-in map result of one side of the road, (C2) Local zoom-in map result of the other side of the road, (C3) Poses of the UAV at the beginning and end of the experiment.
		\label{fig:lid_map_outdoor_sup}}
	\vspace{-0.1cm}
\end{figure}

\begin{figure}[tb]
\vspace{-0.1cm}
\setlength{\abovecaptionskip}{-0.1cm} 
\setlength{\belowcaptionskip}{-0.2cm} 
	\centering
	\includegraphics[width=1\columnwidth]{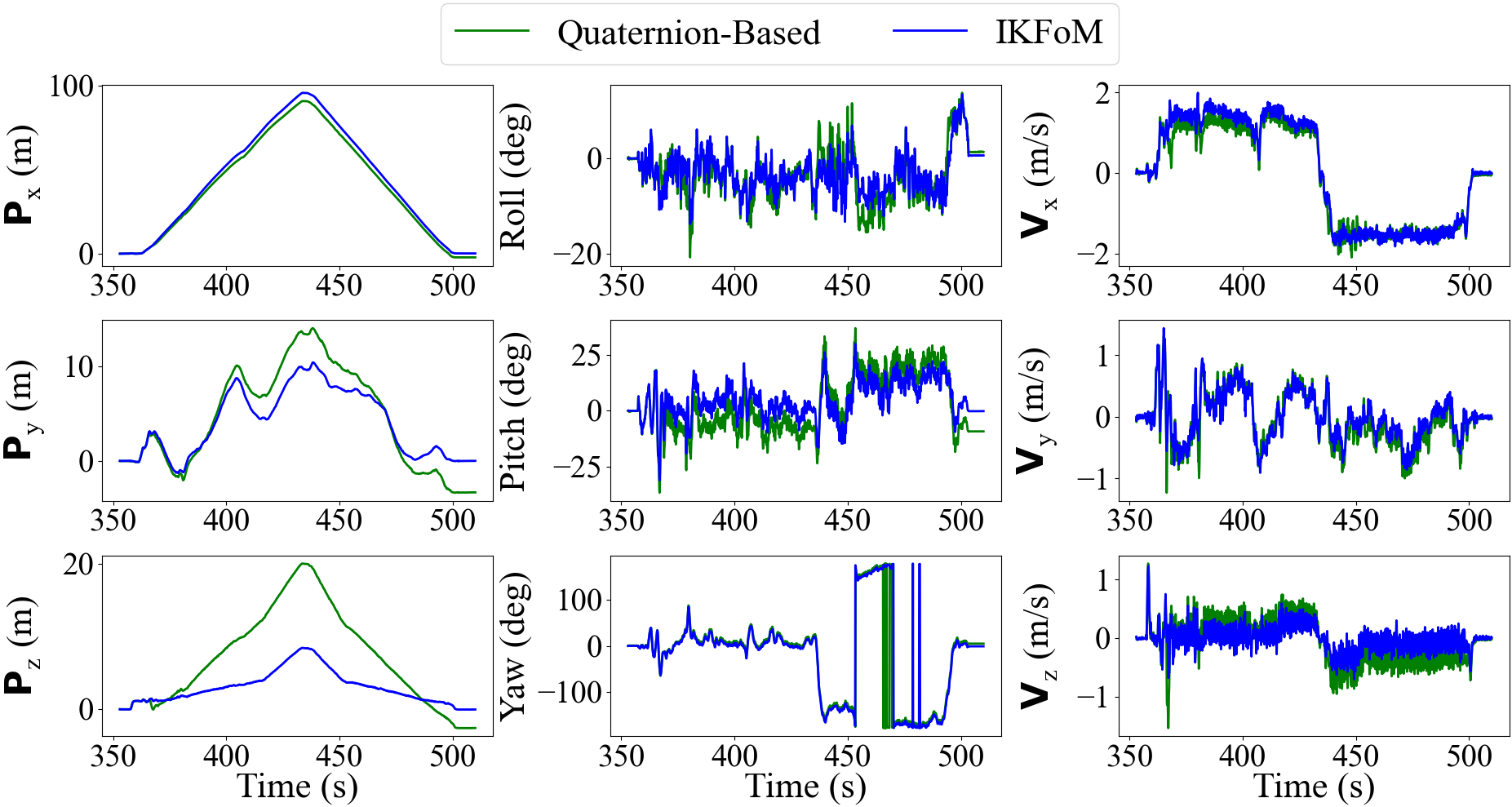}
	\caption{Comparison of estimates of quaternion-based iterated Kalman filter and {\it IKFoM} for the position, rotation in Euler angles and velocity of the IMU on dataset V3-01.
		\label{fig:lid_kin_outdoor_sup}}
	\vspace{-0.3cm}
\end{figure}

\subsubsection{Outdoor random walk}
The experiment of V3-01 is conducted in a structured outdoor environment which is a corridor between a slope and the Hawking Wong building of the University of Hong Kong. In the experiment, the UAV is handheld to move along the road and then return to the beginning position (see Fig.~\ref{fig:lid_map_outdoor_sup} (A)).

The real-time mapping results of dataset V3-01 estimated by our toolkit is shown in Fig.~\ref{fig:lid_map_outdoor_sup} (B), which clearly shows the building on one side and the cars and bricks on the slope. The position drift is less than $0.07458\%$ (i.e., $0.1536m$ drift over $205.95m$ path, see Fig.~\ref{fig:lid_map_outdoor_sup} (B3)). This small drift supports the efficacy of our system. As comparison, the mapping results of the lidar-inertial navigation implemented by the quaternion-based iterated Kalman filter are shown in Fig.~\ref{fig:lid_map_outdoor_sup} (C). Because of the improper propagation of the state covariance of this Kalman filter, the map of the building is layered and that of the cars is fringed. What's more, the drift of the pose at the end of experiment is obviously too large to coincide with the pose at the beginning of the experiment, which is calculated as around $3.9672m$ versus $0.1536m$ of {\it IKFoM}.

The estimations of the kinematics parameters are shown in Fig.~\ref{fig:lid_kin_outdoor_sup}, where the experiment starts from $353.000s$ and ends at $509.999s$. The trajectory estimated by {\it IKFoM} is approximately symmetric about the middle time in X and Z direction, which agrees with the actual motion profile where the sensor is moved back on the same road. The estimates of quaternion-based iterated Kalman filter have a drift from our method, which is inaccurate as indicated by the pose drift. For further experiment demonstration, we refer readers to the videos at~\url{https://youtu.be/sz\_ZlDkl6fA}.

\subsubsection{Online estimation of extrinsic, gravity, and IMU bias}
\begin{figure}[tb]
\vspace{-0.3cm}
\setlength{\abovecaptionskip}{-0.1cm} 
\setlength{\belowcaptionskip}{-0.2cm} 
	\centering
	\includegraphics[width=1.0\columnwidth]{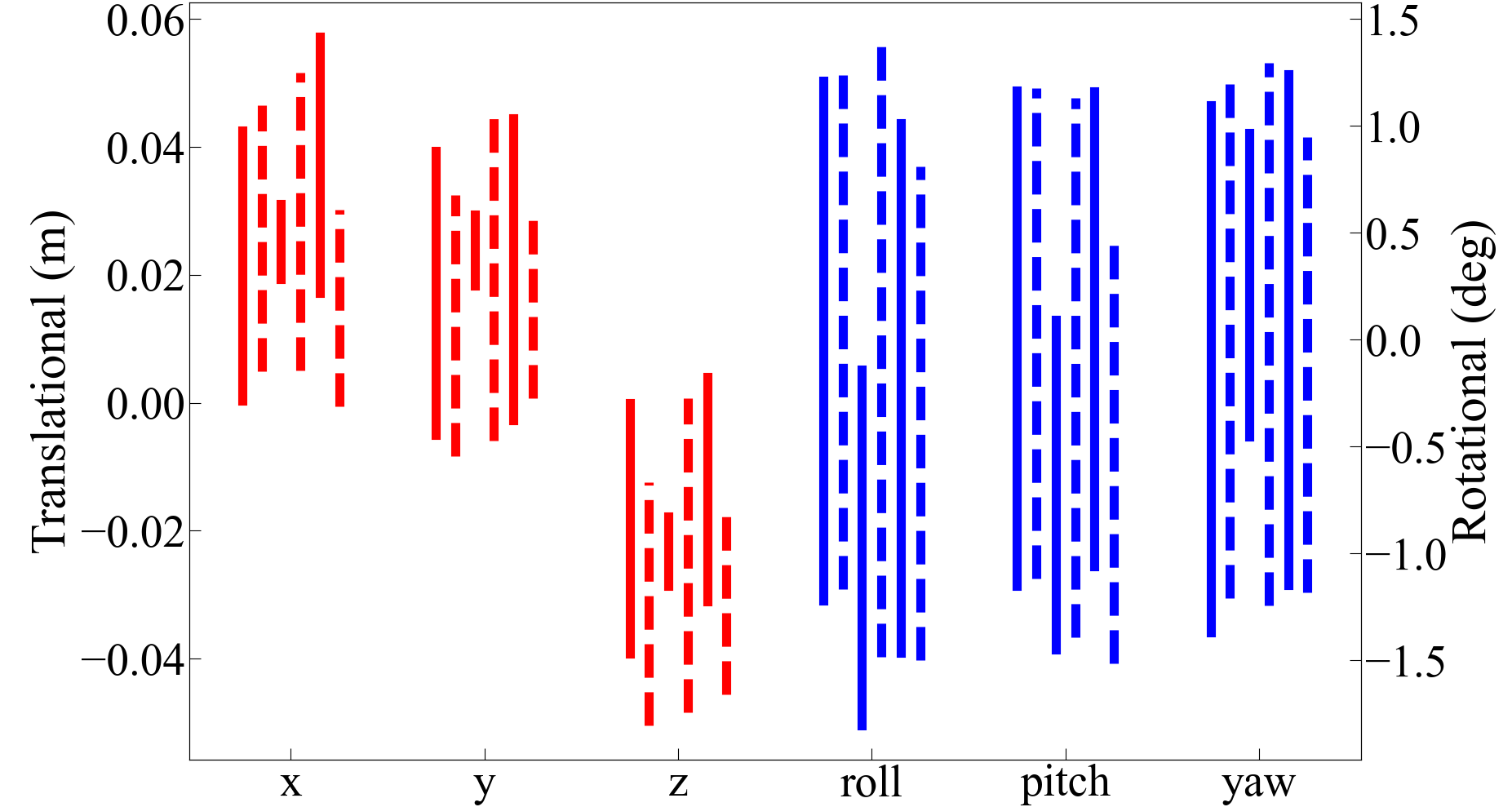}
	\caption{Estimation of the extrinsics (red for translational and blue for rotational) between lidar and IMU for $6$ datasets. The length of each line corresponds to the $3\sigma$ bounds the {\it IKFoM} converged to. 
	\iffalse
	Red, blue, and yellow lines
	\fi
	The lines from left to right indicate datasets V1 (i.e., indoor UAV flight), V2 (i.e., indoor quick-shake), and V3 (i.e., outdoor random walk), in sequence. For each dataset, the solid and the dashed lines respectively indicate the trial 01 (i.e., data collected in this work) and trial 02 (i.e., data from~\cite{xu2020fast}).
		\label{fig:extrinsics_lid}}
	\vspace{-0.1cm}
\end{figure}
To verify our developed method being a properly functioning filter, the online calibration parameters, which are composed of lidar-IMU extrinsics, gravity in the global frame and IMU biases, have to converge. Moreover, the extrinsic estimate should be close across different datasets with the same sensor setup, and we can thus evaluate the extrinsics on multiple datasets and compare the values they have converged to. Fig.~\ref{fig:extrinsics_lid} shows the final estimate of the translational and rotational parts of the extrinsics by running the proposed toolkit on all the six datasets. The initial values of the extrinsics were read from the manufacturer datasheet. As seen in Fig.~\ref{fig:extrinsics_lid}, the extrinsic estimates (both rotation and translation) over different datasets show great agreement. The uncertainty in translation is $1cm-5cm$ while that in rotation is less than $3^{\circ}$. In particular, as indicated by the second solid line counted from left in Fig.~\ref{fig:extrinsics_lid}, we notice a smaller variance in V2-01 over other datasets. This is resulted from the fact that V2-01 has been excited sufficiently and constantly for over $222.85 s$, which is much longer compared with other indoor datasets. And repeated scenes come along during the experiment of V2-01, while the outdoor datasets always deal with new scenes. 
\iffalse
while V2-02 only ran for $48.001 s$ where the Kalman filter has not fully converged (e.g., see Fig. \ref{fig:lid_grav_error}). 
\fi

We further inspect the convergence of the gravity estimation by {\it IKFoM}. Due to the space limit, we show the result on dataset V2-01 only. Fig.~\ref{fig:lid_grav_error} shows the gravity estimation error $\mathbf u\!=\! {}^G\widehat{\mathbf{g}}_k\!\boxminus\! {}^G\bar{\mathbf{g}} \in \mathbb{R}^2$, where  ${}^G\bar{\mathbf{g}}$ is the ground true gravity and ${}^G\widehat{\mathbf{g}}_k$ is the estimate at step $k$. Since the ground true gravity vector is unknown, we use the converged gravity estimation as ${}^G\bar{\mathbf{g}}$. Fig.~\ref{fig:lid_grav_error} further
shows the $3\sigma$ bounds of $\mathbf u$ estimated by the proposed {\it IKFoM}. It is shown that the error constantly falls within the $3\sigma$ bounds, which indicates the consistency of the proposed method. 

\begin{figure}[t]
\vspace{-0.1cm}
\setlength{\abovecaptionskip}{-0.2cm} 
\setlength{\belowcaptionskip}{-0.2cm} 
	\centering
	\includegraphics[width=1.0\columnwidth]{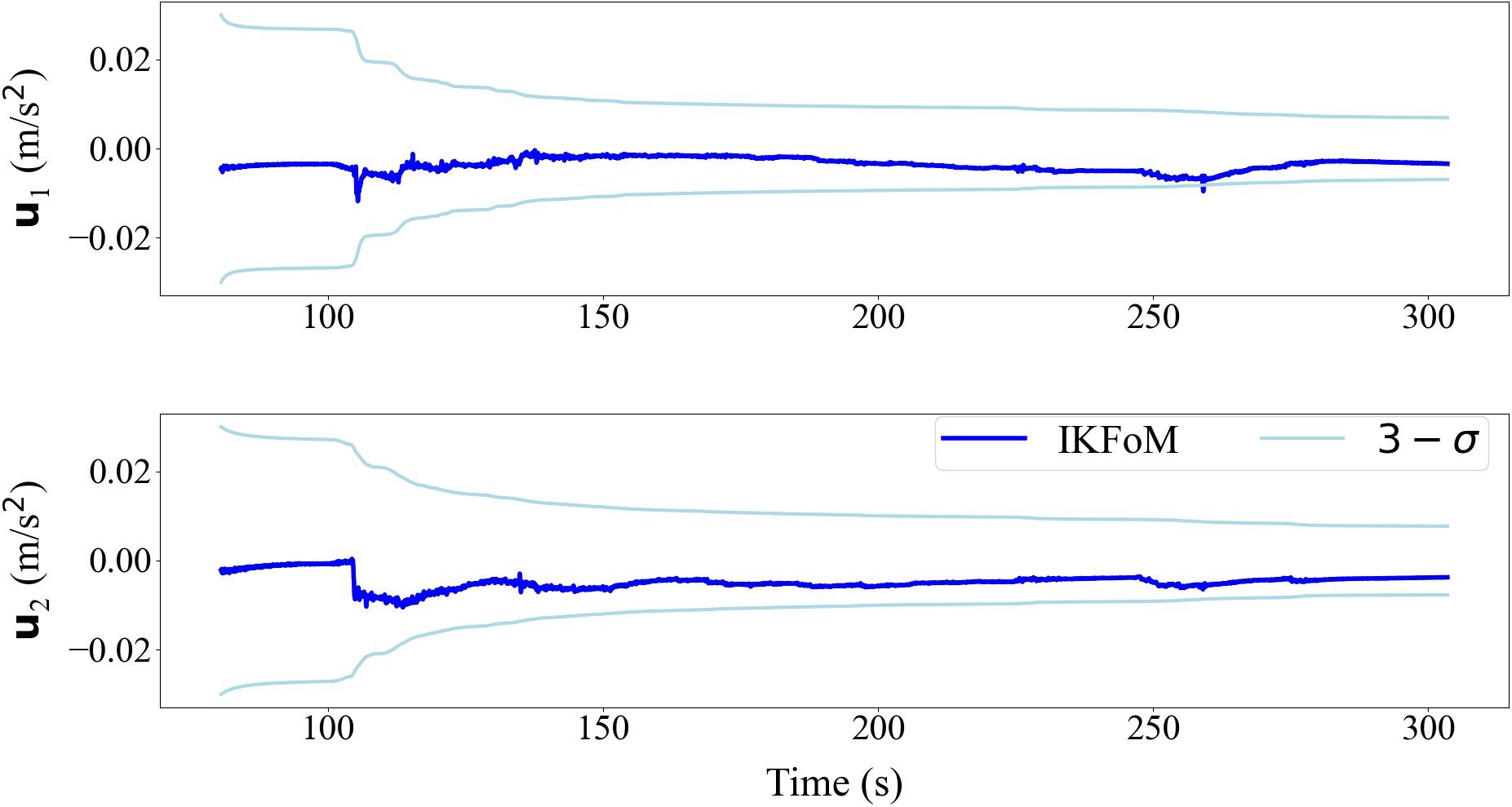}
	\caption{Gravity estimation error (dark blue lines) and its $3\sigma$ bounds (bright green lines) on dataset V2-01. 
		\label{fig:lid_grav_error}}
	\vspace{-0.3cm}
\end{figure}

Finally, we investigate the convergence of
the IMU bias estimation. We show the results on
dataset V2-01 only, which are depicted in Fig.~\ref{fig:lid_bias}. The estimates of IMU biases over time are plotted together with the $3\sigma$ bounds. In particular, the accelerometer biases converge after sufficient excitation and it typically converges faster along the gravity direction due to the large vertical movement at the beginning of the dataset (see Fig. \ref{fig:lid_map_shake}). Also the gyroscope biases converge rapidly due to the large rotational movement.
\iffalse
They typically converge faster along the gravity direction due to the large vertical movement at the beginning of the dataset (see Fig. \ref{fig:lid_map_shake}).
\fi
\begin{figure}[t]
\setlength{\abovecaptionskip}{-0.1cm} 
\setlength{\belowcaptionskip}{-0.2cm} 
\vspace{-0.3cm}
	\centering
	\includegraphics[width=1.0\columnwidth]{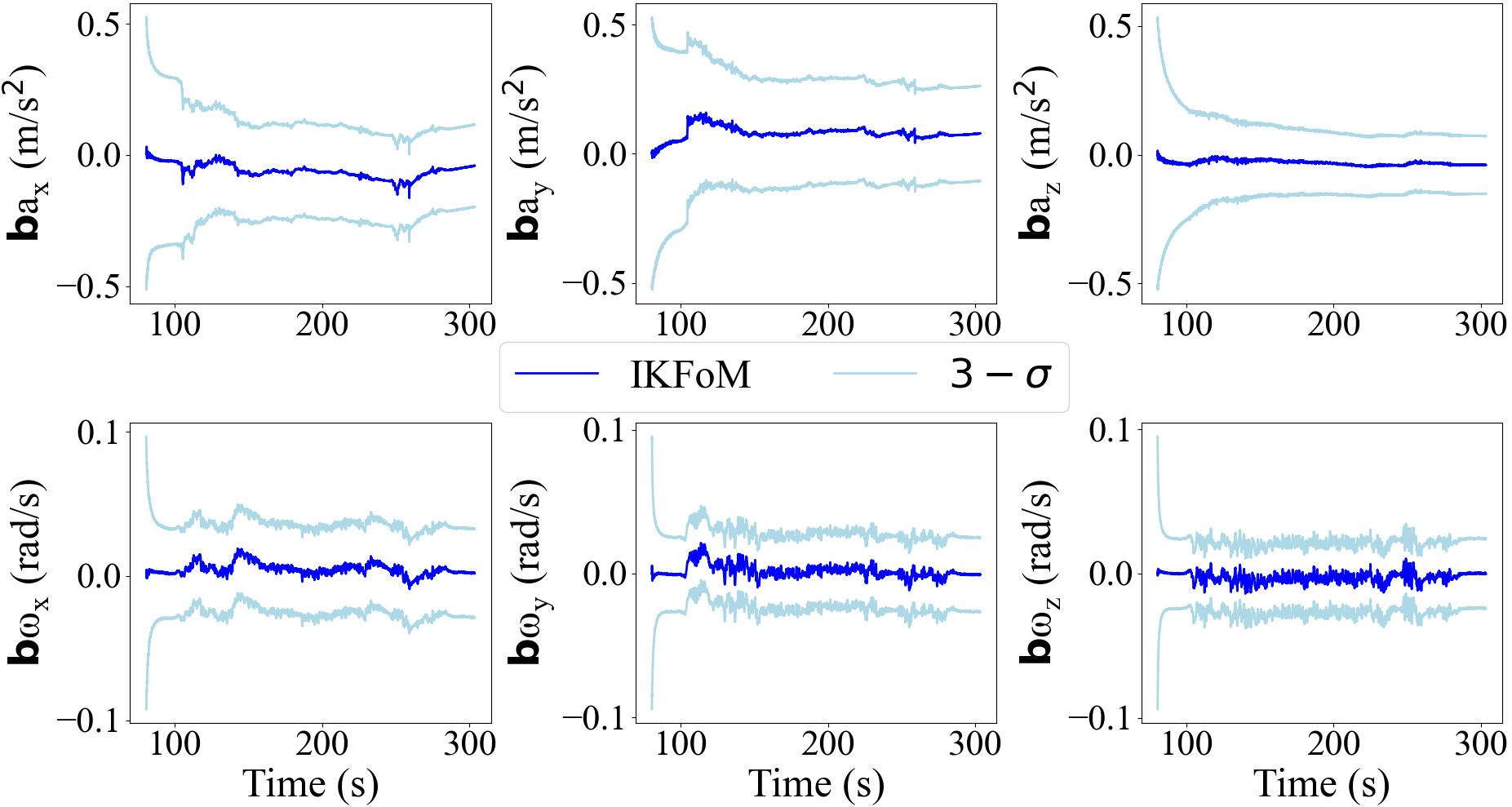}
	\caption{Estimated accelerometer and gyroscope biases on dataset V2-01. Estimates (dark blue lines) together with the $3\sigma$ bounds (bright green lines) are depicted. The estimated accelerometer bias converges quicker along the gravity direction which is mostly along the z-axis.
		\label{fig:lid_bias}}
	\vspace{-0.1cm}
\end{figure}

\subsubsection{Running time}

To further evaluate the practicability of the developed toolkit, its running time on all the six datasets are evaluated and compared against the hand-derived iterated ESEKF in~\cite{xu2020fast}. Note that the work in \cite{xu2020fast} also used an iterated Kalman filter but differs with our implementations in two aspects: (1) The iterated Kalman filter in \cite{xu2020fast} is manually derived and the respective matrices (e.g., $\mathbf F_{\mathbf x_{\tau}}, \mathbf F_{\mathbf w_{\tau}}$ in (\ref{e:Cov_prop})) used in the Kalman filter are directly coded. Matrices sparsity are carefully exploited for computation efficiency. In contrast, our implementation directly uses the toolkit which separates the computation of manifold constraints and system-specific behaviors; (2) The original work in \cite{xu2020fast} did not consider the estimation of extrinsics between lidar and IMU, hence has six fewer state variables. Other than these two aspects, the rest implementations are identical. Both implementations are tested on the UAV onboard computer (see Fig. \ref{fig:conf_lid}). 

The running time comparison is shown in Tab.~\ref{tab:run-time-comparison}, which shows the average time for completing one iteration of the lidar-inertial navigation. As expected, the toolkit-based implementation takes a little more computation time due to the higher state dimension and the toolkit overhead. However, this time overhead is acceptable and both implementations run sufficiently fast in real-time. 
\begin{table}[bt]
\vspace{-0.1cm}
\setlength{\abovecaptionskip}{-0cm} 
\setlength{\belowcaptionskip}{-0.2cm} 
\centering
\caption{Comparison of the average running time}\label{tab:run-time-comparison}
\begin{threeparttable}
    \begin{tabular}{lll}
    \toprule
    &{\it IKFoM}-based & Hand-derived\\
    &implementation&implementation in~\cite{xu2020fast}\\
    \midrule
    V1:&$7.764$$ms$ (01), $8.024$$ms$ (02) &$7.186$$ms$ (01), $7.357$$ms$ (02)\\
    \midrule
    V2:&$31.57$$ms$ (01), $33.88$$ms$ (02)&$21.23$$ms$ (01), $30.17$$ms$ (02)\\ 
    \midrule
    V3:&$60.05$$ms$ (01), $56.75$$ms$ (02)&$55.69$$ms$ (01), $54.56$$ms$ (02)\\
    \bottomrule
    \end{tabular}
    \begin{tablenotes}
    \item[1] Running time is recorded as the time consumed by one complete iteration of lidar-inertial navigation.
    \end{tablenotes}
    \end{threeparttable}
    \vspace{-0.3cm}
\end{table}

\section{Conclusion}~\label{conclusion}
This paper proposed a canonical representation of robot systems and developed a symbolic error-state iterated Kalman filter. The canonical representation employs manifolds to represent the system states and uses $\boxplus\backslash\boxminus$ and $\oplus$ operations to describe the system model. Based on the canonical representation of a robotic system, we showed the separation principle between the manifold constraints and the system-specific behaviors in a Kalman filter framework. This separation enables us to integrate differentiable manifolds into the Kalman filter by developing a $C$++ toolkit, facilitating the quick deployment of Kalman filters to generic robotic systems operating on differentiable manifolds. The proposed method and the developed toolkit are verified on a tightly-coupled lidar-inertial navigation system in three different scenarios, and result in superior filtering performance.

\appendix

\subsection{Partial differentiation of {\it compound manifolds} }~\label{app:partial_diff_Com}
\begin{lemma}~\label{lemma_comp_mani_diff}
If $\mathbf{x}_1, \mathbf{y}_1 \in\mathcal{S}_1$; $\mathbf{x}_2, \mathbf{y}_2 \in\mathcal{S}_2$; $\mathbf{u}_1 \in\mathbb{R}^{n_1}$ and $\mathbf{u}_2\in\mathbb{R}^{n_2}$; where $n_1$, $n_2$ are dimensions of $\mathcal{S}_1$, $\mathcal{S}_2$ respectively, define compound manifold $\mathcal{S}=\mathcal{S}_1\times\mathcal{S}_2$, and its elements $\mathbf{x}=\begin{bmatrix}\mathbf{x}_1 &\mathbf{x}_2\end{bmatrix}^T \in\mathcal{S} $; $\mathbf{y}=\begin{bmatrix}\mathbf{y}_1 &\mathbf{y}_2\end{bmatrix}^T \in\mathcal{S} $; $\mathbf{u}=\begin{bmatrix}\mathbf{u}_1 &\mathbf{u}_2\end{bmatrix}^T \in\mathbb{R}^{n_1+n_2} $ and $\mathbf{v}_1\in\mathbb{R}^{l_1}$, $\mathbf{v}_2\in\mathbb{R}^{l_2}$, $\mathbf{v}=\begin{bmatrix}\mathbf{v}_1 &\mathbf{v}_2\end{bmatrix}^T\in\mathbb{R}^{l_1+l_2}$, then 
\begin{equation}
\setlength{\abovedisplayskip}{0.15cm} 
\setlength{\belowdisplayskip}{0.15cm} 
\begin{aligned}
    &\frac{\partial\left(\!\left(\!\left(\!\mathbf{x}\boxplus_{\mathcal{S}} \mathbf{u}\right)\oplus_{\mathcal{S}} \mathbf{v} \! \right)\boxminus_{\mathcal{S}} \mathbf{y} \! \right)}{\partial\mathbf{u}}\!\!\\
    &=\begin{bmatrix}\!\!\frac{\partial\left(\!\left(\!\left(\!\mathbf{x}_1\boxplus_{\mathcal{S}_1} \! \mathbf{u}_1 \! \right)\oplus_{\mathcal{S}_1} \! \mathbf{v}_1 \! \right)\boxminus_{\mathcal{S}_1}  \mathbf{y}_1 \!\right)}{\partial\mathbf{u}_1}\!\!\!\!&\!\!\!\!\mathbf{0}\!\!\\
    \!\!\mathbf{0}\!\!\!\!&\!\!\!\!\frac{\partial\left(\!\left(\!\left(\!\mathbf{x}_2\boxplus_{\mathcal{S}_2} \! \mathbf{u}_2\!\right)\oplus_{\mathcal{S}_2} \! \mathbf{v}_2\!\right)\boxminus_{\mathcal{S}_2} \mathbf{y}_2\!\right)}{\partial\mathbf{u}_2}\!\!\end{bmatrix}
\end{aligned} \nonumber
\end{equation}
\end{lemma}

\begin{proof}
Define $\mathbf{w}=\left(\left(\mathbf{x}\boxplus_{\mathcal{S}}\mathbf{u}\right)\oplus_{\mathcal{S}}\mathbf{v}\right)\boxminus_{\mathcal{S}}\mathbf{y}$, then according to the composition of operation $\boxplus$, $\boxminus$ and $\oplus$ in the paper, we have
\begin{equation}
\setlength{\abovedisplayskip}{0.15cm} 
\setlength{\belowdisplayskip}{0.15cm} 
\begin{aligned}
    \mathbf{w}&\!=\left(\left(\mathbf{x}\boxplus_{\mathcal{S}}\mathbf{u}\right)\oplus_{\mathcal{S}}\mathbf{v}\right)\boxminus_{\mathcal{S}}\mathbf{y}\\
    &\! = \left(\left(\begin{bmatrix}
    \mathbf{x}_1 \\ \mathbf{x}_2
    \end{bmatrix} \boxplus_{\mathcal{S}} \begin{bmatrix}
    \mathbf{u}_1 \\ \mathbf{u}_2
    \end{bmatrix} \right)\oplus_{\mathcal{S}}\mathbf{v}\right)\boxminus_{\mathcal{S}}\mathbf{y}\\
    &\!=\left(\begin{bmatrix}\mathbf{x}_1\boxplus_{\mathcal{S}_1}\mathbf{u}_1\\
    \mathbf{x}_2\boxplus_{\mathcal{S}_2}\mathbf{u}_2\end{bmatrix}\oplus_{\mathcal{S}}\begin{bmatrix}
    \mathbf{v}_1 \\ \mathbf{v}_2
    \end{bmatrix}\right)\boxminus_{\mathcal{S}}\mathbf{y}\\
    &\!=\begin{bmatrix}\left(\mathbf{x}_1\boxplus_{\mathcal{S}_1}\mathbf{u}_1\right)\oplus_{\mathcal{S}_1}\mathbf{v}_1\\
    \left(\mathbf{x}_2\boxplus_{\mathcal{S}_2}\mathbf{u}_2\right)\oplus_{\mathcal{S}_2}\mathbf{v}_2\end{bmatrix}\boxminus_{\mathcal{S}}\begin{bmatrix}
    \mathbf{y}_1 \\ \mathbf{y}_2
    \end{bmatrix}\\
    &\!=\begin{bmatrix}\left(\left(\mathbf{x}_1\boxplus_{\mathcal{S}_1}\mathbf{u}_1\right)\oplus_{\mathcal{S}_1}\mathbf{v}_1\right)\boxminus_{\mathcal{S}_1}\mathbf{y}_1\\
    \left(\left(\mathbf{x}_2\boxplus_{\mathcal{S}_2}\mathbf{u}_2\right)\oplus_{\mathcal{S}_2}\mathbf{v}_2\right)\boxminus_{\mathcal{S}_2}\mathbf{y}_2\end{bmatrix} \triangleq \begin{bmatrix}
    \mathbf{w}_1 \\ \mathbf{w}_2
    \end{bmatrix}
\end{aligned} \nonumber
\end{equation}

As a result, the differentiation is
\begin{equation}\notag
\setlength{\abovedisplayskip}{0.15cm} 
\setlength{\belowdisplayskip}{0.15cm} 
    \begin{aligned}
 &\frac{\partial\mathbf{w}}{\partial\mathbf{u}}=\begin{bmatrix}\frac{\partial\mathbf{w}_1}{\partial\mathbf{u}_1}&\frac{\partial\mathbf{w}_1}{\partial\mathbf{u}_2}\\
    \frac{\partial\mathbf{w}_2}{\partial\mathbf{u}_1}&\frac{\partial\mathbf{w}_2}{\partial\mathbf{u}_2}\end{bmatrix}=\begin{bmatrix}\frac{\partial\mathbf{w}_1}{\partial\mathbf{u}_1}&\mathbf{0}\\
    \mathbf{0}&\frac{\partial\mathbf{w}_2}{\partial\mathbf{u}_2}\end{bmatrix}\\
    &=\begin{bmatrix}\!\!\frac{\partial\left(\!\left(\!\left(\!\mathbf{x}_1\boxplus_{\mathcal{S}_1} \! \mathbf{u}_1 \! \right)\oplus_{\mathcal{S}_1} \! \mathbf{v}_1 \! \right)\boxminus_{\mathcal{S}_1} \mathbf{y}_1 \!\right)}{\partial\mathbf{u}_1}\!\!\!\!&\!\!\!\!\mathbf{0}\!\!\\
    \!\!\mathbf{0}\!\!\!\!&\!\!\!\!\frac{\partial\left(\!\left(\!\left(\!\mathbf{x}_2\boxplus_{\mathcal{S}_2} \! \mathbf{u}_2\!\right)\oplus_{\mathcal{S}_2} \! \mathbf{v}_2\!\right)\boxminus_{\mathcal{S}_2} \mathbf{y}_2\!\right)}{\partial\mathbf{u}_2}\!\!\end{bmatrix}
    \end{aligned}
\end{equation}
\end{proof}
The partial differentiation of $(((\mathbf{x}\boxplus_{\mathcal{S}})\oplus_{\mathcal{S}})\mathbf{v})\boxminus_{\mathcal{S}}\mathbf{y})$ with respect to the $\mathbf{v}$ on the {\it compound manifold} $\mathcal{S}$ is able to be proved in the same way as above.

\subsection{Important manifolds in practice and their derivatives}\label{app:important_manifolds}

\textbf{\textit{ 1. Euclidean space $\mathcal{S} = \mathbb{R}^n$:}}

\begin{equation}
\setlength{\abovedisplayskip}{0.15cm} 
\setlength{\belowdisplayskip}{0.15cm} 
    \begin{aligned}
    \mathbf{x}\boxplus\mathbf{u}&=\mathbf{x}+\mathbf{u}\\
    \mathbf{y}\boxminus\mathbf{x}&=\mathbf{y}-\mathbf{x}\\
    \mathbf{x}\oplus \mathbf{v}&=\mathbf{x}+\mathbf{v}\\
    \frac{\partial\left(\left(\left(\mathbf{x}\boxplus\mathbf{u}\right)\oplus \mathbf{v}\right)\boxminus\mathbf{y}\right)}{\partial \mathbf{u}}&=\mathbf{I}_{n\times n}\\
    \frac{\partial\left(\left(\left(\mathbf{x}\boxplus\mathbf{u}\right)\oplus \mathbf{v}\right)\boxminus\mathbf{y}\right)}{\partial \mathbf{v}}&=\mathbf{I}_{n\times n}
    \end{aligned}
\end{equation}

\textbf{\textit{2. Special orthogonal group $\mathcal{S}\!\!=\!SO(3)$:}} 

\begin{equation}
\setlength{\abovedisplayskip}{0.15cm} 
\setlength{\belowdisplayskip}{0.15cm}  \label{SO3:boxplus_sup}
    \mathbf{x}\boxplus\mathbf{u}=\mathbf{x}\cdot{\rm Exp}\left(\mathbf{u}\right) 
    \end{equation}
    \begin{equation} \label{SO3:boxminus_sup}
    \mathbf{y}\boxminus\mathbf{x}={\rm Log}\left(\mathbf{x}^{-1}\cdot\mathbf{y}\right)
    \end{equation}
    \begin{equation} \label{SO3:oplus_sup}
    \mathbf{x}\oplus \mathbf{v}=\mathbf{x}\cdot{\rm Exp}\left(\mathbf{v}\right) 
    \end{equation}
    \begin{equation} \label{SO3:partial1_sup}
    \frac{\partial\left(\!\left(\!\left(\!\mathbf{x}\boxplus\mathbf{u}\right)\oplus \mathbf{v}\right)\boxminus\mathbf{y}\right)}{\partial \mathbf{u}}\!\!=\!\!\mathbf{A}\!\left(\!\left(\!\left(\!\mathbf{x}\!\boxplus\!\mathbf{u}\right)\!\oplus \!\mathbf{v}\right)\!\boxminus\!\mathbf{y}\right)\!^{-T}{\rm Exp}\!\left(\!-\mathbf{v}\right)\!\mathbf{A}\!\left(\!\mathbf{u}\right)\!^T
    \end{equation}
    \begin{equation} \label{SO3:partial2_sup}
    \frac{\partial\left(\!\left(\!\left(\!\mathbf{x}\boxplus\mathbf{u}\right)\oplus \mathbf{v}\right)\boxminus\mathbf{y}\right)}{\partial \mathbf{v}}\!\!=\!\!\mathbf{A}\!\left(\!\left(\!\left(\!\mathbf{x}\!\boxplus\!\mathbf{u}\right)\!\oplus \!\mathbf{v}\right)\!\boxminus\!\mathbf{y}\right)\!^{-T}\!\mathbf{A}\!\left(\mathbf{v}\right)\!^T
    \end{equation}

where
\begin{equation}~\label{eq:A_sup}
\setlength{\abovedisplayskip}{0.15cm} 
\setlength{\belowdisplayskip}{0.15cm} 
\begin{aligned}
\rm{Exp}\left(\mathbf{u} \right)\!&=\!\rm{exp}\left(\lfloor\mathbf{u} \rfloor\right) \\
\mathbf{A}\!\left(\mathbf{u}\right)\!&=\!\mathbf{I}\!+\!\left(\frac{1\!-\!{\rm cos}\left(\Vert\mathbf{u}\Vert\right)}{\Vert\mathbf{u}\Vert}\right)\!\frac{\lfloor\mathbf{u}\rfloor}{\Vert\mathbf{u}\Vert}\!+\!\left(1\!-\!\frac{{\rm sin}\left(\Vert\mathbf{u}\Vert\right)}{\Vert\mathbf{u}\Vert}\right)\!\frac{\lfloor\mathbf{u}\rfloor^2}{\Vert\mathbf{u}\Vert^2}\\
\mathbf{A}\left(\mathbf{u}\right)^{-1}&=\mathbf{I}-\frac{1}{2}\lfloor\mathbf{u}\rfloor+\left(1-\alpha\left(\Vert\mathbf{u}\Vert\right)\right)\frac{\lfloor\mathbf{u}\rfloor^2}{\Vert\mathbf{u}\Vert^2}\\
\alpha\left(\|\mathbf{u}\|\right)&=\frac{\|\mathbf{u}\|}{2}{\rm cot}\left(\frac{\|\mathbf{u}\|}{2}\right)=\frac{\|\mathbf{u}\|}{2}\frac{{\rm cos}\left({\|\mathbf{u}\|}/2\right)}{{\rm sin}\left(\|\mathbf{u}\|/2\right)}
\end{aligned}
\end{equation}

In the above equations, $\mathbf{u}\in\mathbb{R}^3$ is an axis-angle and $\lfloor\mathbf{u}\rfloor$ denotes the skew-symmetric matrix that maps the cross product of $\mathbf{u}$. 

As explained in Sec.~\ref{boxplus_method}, $SO(3)$ is a Lie group, for which, its tangent space is a Lie algebra with $\mathfrak{m}=\mathfrak{so}(3)=\{\lfloor\mathbf{u}\rfloor|\mathbf{u}\in\mathbb{R}^3\}$ and its minimal parameterization space is $\mathbf{R}^3$. Therefore, the $\rm{exp}$ and $\mathfrak{f}$ of $SO(3)$ are $\mathfrak{so}(3)\mapsto SO(3):\rm{exp}(\lfloor\mathbf{u}\rfloor)$ and $\mathbb{R}^3\mapsto\mathfrak{so}(3):\lfloor\mathbf{u}\rfloor$ respectively, which finally results in $\rm{Exp}=\rm{exp}\circ\mathfrak{f}=\rm{exp}(\lfloor\mathbf{u}\rfloor), \mathbf{u}\in\mathbb{R}^3$. What's more, its $\oplus$ operation is defined the same as its $\boxplus$ due to the fact that the exogenous velocity typically lies in the tangent plane for a Lie group, where the $\boxplus$ operation is defined.

The partial differentiations of $\left(\left(\left(\mathbf{x}\boxplus\mathbf{u}\right)\oplus\mathbf{v}\right)\boxminus\mathbf{y}\right)$ w.r.t. $\mathbf{u}$ (\ref{SO3:partial1_sup}) and $\mathbf{v}$ (\ref{SO3:partial2_sup}) are directly calculated without using the chain rule. (\ref{SO3:partial1_sup}) is proved as follows:

Denote $\mathbf{w}\!=\!\left(\left(\mathbf{x}\!\boxplus\!\mathbf{u}\right)\!\oplus \!\mathbf{v}\right)\!\boxminus\!\mathbf{y}$, we have
\begin{equation}
\setlength{\abovedisplayskip}{0.15cm} 
\setlength{\belowdisplayskip}{0.15cm} 
    {\rm Exp}\left(\mathbf{w}\right)=\mathbf{y}^{-1}\cdot\mathbf{x}\cdot{\rm Exp}\left(\mathbf{u}\right)\cdot{\rm Exp}\left(\mathbf{v}\right) \label{eq:w_sup}
\end{equation}

Hence a small variation $\Delta \mathbf{u}$ in $\mathbf{u}$ causes a small variation $\Delta \mathbf{w}$ in $\mathbf{w}$, which is subject to
\begin{equation}
\setlength{\abovedisplayskip}{0.15cm} 
\setlength{\belowdisplayskip}{0.15cm} 
    {\rm Exp}\left(\mathbf{w}+\Delta\mathbf{w}\right)=\mathbf{y}^{-1}\cdot\mathbf{x}\cdot{\rm Exp}\left(\mathbf{u}+\Delta\mathbf{u}\right)\cdot{\rm Exp}\left(\mathbf{v}\right) \label{lemma1_aft_sup}
\end{equation}

Using the fact ${\rm Exp}\!\left(\mathbf{u}\!+\!\Delta\mathbf{u}\right)\!=\!{\rm Exp}\!\left(\mathbf{u}\right)\!\cdot\!\left(\mathbf{I}\!+\!\lfloor\mathbf{A}\!\left(\!\mathbf{u}\!\right)^T\!\Delta\mathbf{u}\rfloor\!\right)$ as shown in~\cite{bullo1995proportionalsupp}, it is derived that the left hand side of (\ref{lemma1_aft_sup}) 
\begin{equation}
\setlength{\abovedisplayskip}{0.15cm} 
\setlength{\belowdisplayskip}{0.15cm} 
\begin{aligned}
&{\rm Exp}\left(\mathbf{w}\!+\!\Delta\mathbf{w}\right)\!={\rm Exp}\!\left(\mathbf{w}\right)\!\cdot\!\left(\mathbf{I}\!+\!\lfloor\mathbf{A}\!\left(\mathbf{w}\right)^T\!\Delta\!\mathbf{w}\rfloor\!\right) \label{eq:lside_sup}
\end{aligned} 
\end{equation}
and the right hand side of (\ref{lemma1_aft_sup})
\begin{equation}
\setlength{\abovedisplayskip}{0.15cm} 
\setlength{\belowdisplayskip}{0.15cm} 
\begin{aligned}
&\mathbf{y}^{-1}\cdot\mathbf{x}\cdot{\rm Exp}\left(\mathbf{u}+\Delta\mathbf{u}\right)\cdot{\rm Exp}\left(\mathbf{v}\right)\\
&\ =\mathbf{y}^{-1}\cdot\mathbf{x}\cdot{\rm Exp}\left(\mathbf{u}\right)\cdot\left(\mathbf{I}+\lfloor\mathbf{A}\left(\mathbf{u}\right)^T\Delta\mathbf{u}\rfloor\right)\cdot{\rm Exp}\left(\mathbf{v}\right)\\
&\ ={\rm Exp}\left(\mathbf{w}\right){\rm Exp}\left(-\mathbf{v}\right)\cdot\left(\mathbf{I}+\lfloor\mathbf{A}\left(\mathbf{u}\right)^T\Delta\mathbf{u}\rfloor\right)\cdot{\rm Exp}\left(\mathbf{v}\right)
\end{aligned} \nonumber
\end{equation}

Equating the two sides of (\ref{lemma1_aft_sup}) leads to
\begin{equation}
\setlength{\abovedisplayskip}{0.15cm} 
\setlength{\belowdisplayskip}{0.15cm} 
\begin{aligned}
    \mathbf{A}\left(\mathbf{w}\right)^T\Delta\mathbf{w}={\rm Exp}\left(-\mathbf{v}\right)\cdot\mathbf{A}\left(\mathbf{u}\right)^T \Delta\mathbf{u} \nonumber
\end{aligned}
\end{equation}
and as a result,
\begin{equation}
\setlength{\abovedisplayskip}{0.15cm} 
\setlength{\belowdisplayskip}{0.15cm} 
\begin{aligned}
    \frac{\partial \left( \left(\left(\mathbf{x}\boxplus\mathbf{u}\right)\oplus \mathbf{v}\right)\boxminus\mathbf{y} \right)}{\partial \mathbf{u} }=\frac{\Delta\mathbf{w}}{\Delta\mathbf{u}}=\mathbf{A}\left(\mathbf{w}\right)^{-T}{\rm Exp}\left(-\mathbf{v}\right)\mathbf{A}\left(\mathbf{u}\right)^T \nonumber
\end{aligned}
\end{equation}

The equation (\ref{SO3:partial1_sup}) is obtained.

(\ref{SO3:partial2_sup}) is able to be calculated in the same way:

Using the same denotation of $\mathbf{w}$ and based on the equation (\ref{eq:w_sup}), a small variation $\Delta \mathbf{v}$ in $\mathbf{v}$ causes a small variation $\Delta \mathbf{w}$ in $\mathbf{w}$, which is subject to
\begin{equation}
\setlength{\abovedisplayskip}{0.15cm} 
\setlength{\belowdisplayskip}{0.15cm} 
    {\rm Exp}\left(\mathbf{w}+\Delta\mathbf{w}\right)=\mathbf{y}^{-1}\cdot\mathbf{x}\cdot{\rm Exp}\left(\mathbf{u}\right)\cdot{\rm Exp}\left(\mathbf{v}+\Delta\mathbf{v}\right) \label{prove:partial2_sup}
\end{equation}

Then the left side of (\ref{prove:partial2_sup}) is the same as shown in (\ref{eq:lside_sup}), and the right side of it is:
\begin{equation}
\setlength{\abovedisplayskip}{0.15cm} 
\setlength{\belowdisplayskip}{0.15cm} 
\begin{aligned}
&\mathbf{y}^{-1}\cdot\mathbf{x}\cdot{\rm Exp}\left(\mathbf{u}\right)\cdot{\rm Exp}\left(\mathbf{v}+\Delta\mathbf{v}\right)\\
&\ =\mathbf{y}^{-1}\cdot\mathbf{x}\cdot{\rm Exp}\left(\mathbf{u}\right)\cdot{\rm Exp}\left(\mathbf{v}\right)\cdot\left(\mathbf{I}+\lfloor\mathbf{A}\left(\mathbf{v}\right)^T\Delta\mathbf{v}\rfloor\right)\\
&\ ={\rm Exp}\left(\mathbf{w}\right)\cdot\left(\mathbf{I}+\lfloor\mathbf{A}\left(\mathbf{v}\right)^T\Delta\mathbf{v}\rfloor\right)
\end{aligned} \nonumber
\end{equation}

Equating the two sides of (\ref{prove:partial2_sup}) leads to
\begin{equation}
\setlength{\abovedisplayskip}{0.15cm} 
\setlength{\belowdisplayskip}{0.15cm} 
\mathbf{A}\left(\mathbf{w}\right)^T\Delta\mathbf{w}=\mathbf{A}\left(\mathbf{v}\right)^T\Delta\mathbf{v}
\end{equation}
and as a result:
\begin{equation}
\setlength{\abovedisplayskip}{0.15cm} 
\setlength{\belowdisplayskip}{0.15cm} 
    \frac{\partial(((\mathbf{x}\boxplus\mathbf{u})\oplus\mathbf{v})\boxminus\mathbf{y})}{\partial\mathbf{v}}=\frac{\Delta\mathbf{w}}{\Delta\mathbf{v}}=\mathbf{A}(\mathbf{w})^{-T}\mathbf{A}(\mathbf{v})
\end{equation}

The equation (\ref{SO3:partial2_sup}) is proved.

\textbf{\textit{3. 2-sphere, $\mathcal{S} = \mathbb{S}^2({r}) \triangleq \{\mathbf{x}\in\mathbb{R}^3|\Vert\mathbf{x}\Vert=  r, r > 0\}$:}} 
\begin{equation}
\setlength{\abovedisplayskip}{0.15cm} 
\setlength{\belowdisplayskip}{0.15cm} 
    \label{S2:def_sup1}
    \begin{aligned}
     \mathbf{x}\boxplus\mathbf{u}&=\mathbf{R}(\mathbf{B}(\mathbf{x})\mathbf{u})\cdot\mathbf{x}\\
     \mathbf{y}\boxminus\mathbf{x}&=\mathbf{B}\left(\mathbf{x}\right)^T \left(\theta \frac{\lfloor\mathbf{x}\rfloor\mathbf{y}}{\|\lfloor\mathbf{x}\rfloor\mathbf{y}\|}\right), \theta \!= \! \rm{atan2}\left(\|\lfloor\mathbf{x}\rfloor\mathbf{y}\|,\mathbf{x}^T\mathbf{y}\right)\\
     \mathbf{x}\oplus\mathbf{v}&=\mathbf{R}(\mathbf{v})\cdot\mathbf{x}\\
     \frac{\partial\left(\!\left(\!\left(\!\mathbf{x}\boxplus\mathbf{u}\right)\oplus \mathbf{v}\right)\boxminus\mathbf{y}\right)}{\partial\mathbf{u}}\!&=\!\mathbf{N}\! \left(\!(\mathbf{x}\! \boxplus \! \mathbf{u}) \! \oplus \! \mathbf{v},\!\mathbf{y} \! \right) \! {\mathbf R}(\mathbf{v})  \mathbf{M}\left(\mathbf{x},\!\mathbf{u}\right)\\
     \frac{\partial\left(\!\left(\!\left(\!\mathbf{x}\boxplus\mathbf{u}\right)\oplus \mathbf{v}\right)\boxminus\mathbf{y}\right)}{\partial\mathbf{v}}\!&=\!-\mathbf{N}\! \left(\!(\mathbf{x}\! \boxplus \! \mathbf{u}) \! \oplus \! \mathbf{v},\!\mathbf{y} \! \right) \! {\mathbf R}(\mathbf{v}) \lfloor\mathbf{x}\!\boxplus\!\mathbf{u}\rfloor\! \mathbf{A}\!\left(\!\mathbf{v}\!\right)^T
    \end{aligned}
\end{equation}
where $\mathbf{R}(\mathbf{w})=\rm{Exp}(\mathbf{w})\in SO(3)$ as a rotation about an axis-angle represented by the vector $\mathbf{w}\in\mathbb{R}^3$, and $\mathbf{B}(\mathbf{x})\in\mathbb{R}^{3\times 2}=\begin{bmatrix}\mathbf{b}_1&\mathbf{b}_2\end{bmatrix}$ as two orthonormal bases lying in the tangent plane of $\mathbf{x}\in\mathbb{S}^2(r)$. $\mathbf{A}(\cdot)$ is defined in equation (\ref{eq:A_sup}) and
\begin{equation}
    \label{eq:S2_sup1}
\setlength{\abovedisplayskip}{0.15cm} 
\setlength{\belowdisplayskip}{0.15cm} 
    \begin{aligned}
     \mathbf{N}\!\left(\mathbf{x},\!\mathbf{y}\right)&\!=\!\frac{\partial\left(\mathbf{x}\boxminus\mathbf{y}\right)}{\partial\mathbf{x}}\!=\!\mathbf{B}\!\left(\mathbf{y}\right)^{T}\!\!\left(\frac{\theta}{\| \! \lfloor\mathbf{y} \! \rfloor\mathbf{x}\|}\lfloor\! \mathbf{y} \! \rfloor\!+\!\lfloor \! \mathbf{y} \! \rfloor\mathbf{x}\!\cdot\!\mathbf{P}\!\left(\! \mathbf{x}, \! \mathbf{y}\right)\right)\\
     \mathbf{M}\!\left(\mathbf{x},\!\mathbf{u}\right)&\! = \! \frac{\partial\left(\mathbf{x}\boxplus\mathbf{u}\right)}{\partial\mathbf{u}} \!=\! - {\rm Exp}(\mathbf B (\! \mathbf x \! ) \mathbf u) \lfloor \! \mathbf x \! \rfloor  \mathbf A \! \left(\mathbf B( \! \mathbf x \! ) \mathbf u \right)^{T} \mathbf B( \! \mathbf x \! ) \\
\mathbf{P}\!\left(\mathbf{x},\mathbf{y}\right) &\!=\!\frac{1}{{r}^4}\left(\frac{-\mathbf{y}^T\mathbf{x}\|\lfloor\mathbf{y}\rfloor\mathbf{x}\|+{r}^4\theta}{\|\lfloor\mathbf{y}\rfloor\mathbf{x}\|^3}\mathbf x^T \lfloor \mathbf y \rfloor^2 \!-\!\mathbf{y}^T \right)
\end{aligned}
\end{equation}

As stated in Sec.~\ref{boxplus_method}, the perturbation on $\mathbf{x}\in\mathbb{S}^2(r)$ can be achieved by rotating along a vector in the tangent plane of $\mathbf{x}$, the result would still remain on $\mathbb{S}^2(r)$. Thus the tangent planes can be chosen as the homeomorphic place of $\mathbb{S}^2(r)$ to define the $\boxplus$ operation and the perturbation could be parameterized as $\mathbf{u}\in\mathbb{R}^2$ in the frame of the tangent plane. In particular, the calculation of the $\boxplus$ is divided into two steps: 1) firstly the $2$-dimension vector $\mathbf{u}$ in the frame of local tangent plane at point $\mathbf{x}$ is transferred to a $3$-dimension vector by two orthonormal bases $\mathbf{b}_1, \mathbf{b}_2\in\mathbb{R}^3$ lying in this tangent plane, 2) then $\mathbf{x}$ is rotated along this $3$-dimension vector as shown in Fig.~\ref{fig:S2_boxplus} in Sec.~\ref{boxplus_method}. And the $\boxminus$ is defined as the inverse calculation of the $\boxplus$, which obtains the rotation vector in the frame of the local tangent plane between two states in $\mathbb{S}^2(r)$. It is noticed that $\mathbf{b}_1, \mathbf{b}_2$ is $\mathbf{x}$-depending and they can be denoted as $\mathbf{B}(\mathbf{x})=\begin{bmatrix}\mathbf{b}_1&\mathbf{b}_2\end{bmatrix}$. Further denote $\mathbf{R}(\mathbf{w})=\rm{Exp}(\mathbf{w})\in SO(3)$ as a rotation about an axis-angle represented by the vector $\mathbf{w}\in\mathbb{R}^3$, we have 
\iffalse
\begin{figure}[t]
	\centering
	\includegraphics[width=0.7\columnwidth]{figures/S2_a_noei_sup.png}
	\caption{Illustration of the $\boxplus$ operation on the $\mathcal{S}^2(r)$ space.
		\label{fig:S2_boxplus_sup}}
	\vspace{-0cm}
\end{figure}
\fi

\begin{equation}
    \label{S2:def_sup}
\setlength{\abovedisplayskip}{0.15cm} 
\setlength{\belowdisplayskip}{0.15cm} 
    \begin{aligned}
     &\mathbf{x}\boxplus\mathbf{u}  \triangleq \mathbf{R}(\begin{bmatrix} \mathbf{b}_1 & \mathbf{b}_2
\end{bmatrix} \mathbf{u})\cdot\mathbf{x}=\mathbf{R}(\mathbf{B}(\mathbf{x})\mathbf{u})\cdot\mathbf{x}\\
     &\mathbf{y}\boxminus\mathbf{x}=\mathbf{B}\left(\mathbf{x}\right)^T \left(\theta \frac{\lfloor\mathbf{x}\rfloor\mathbf{y}}{\|\lfloor\mathbf{x}\rfloor\mathbf{y}\|}\right), \theta \!= \! \rm{atan2}\left(\|\lfloor\mathbf{x}\rfloor\mathbf{y}\|,\mathbf{x}^T\mathbf{y}\right)
    \end{aligned}
\end{equation}

The above results do not specify the basis $\mathbf B(\mathbf x)$, which can be made arbitrary as long as it forms an orthonormal basis in the tangent plane of $\mathbf x$. As an example, we could adopt the method in \cite{Hertzberg2011Sensorfusionsupp} (see Fig.~\ref{fig:S2_basis_sup}): rotate one of the three canonical basis $\mathbf{e}_i, i = 1,2,3$ of $\mathbb{R}^3$ to $\mathbf{x}$ (along the \text{\it geodesics}) and the rest two base vectors after the same rotation would form $\mathbf{B}(\mathbf x)$. i.e.,
\begin{equation}
    \label{e:Rx_sup}
\setlength{\abovedisplayskip}{0.15cm} 
\setlength{\belowdisplayskip}{0.15cm} 
    \begin{aligned}
     &\mathbf{R}_i({\mathbf{x}})\!=\!\mathbf{R}\left(\frac{\lfloor \mathbf{e}_i \rfloor \mathbf{x}}{\| \lfloor \mathbf{e}_i \rfloor \mathbf{x} \| } \text{atan2} \left(\|\lfloor \mathbf{e}_i \rfloor \mathbf{x}\|,\mathbf{e}^T_i \mathbf{x}\right)\right), \\
     &\mathbf B(\mathbf x) \!=\! \mathbf R_i ({\mathbf x})\begin{bmatrix}
     \mathbf e_j & \mathbf e_k
     \end{bmatrix}.
    \end{aligned}
\end{equation}
where $j = i+1, k = i + 2$ but wrapped below 3.

As stated in Sec.~\ref{boxplus_method}, in practical control systems, $\mathbb{S}^2(r)$ is usually used to represent a vector of fixed length $r$, which is driven by an input vector not lying in the defined local homeomorphic space (i.e., the tangent plane) of $\mathbb{S}^2(r)$. And because the perturbation on the $\mathbf{x}\in\mathbb{S}^2(r)$ would result in a rotation of $\mathbf{x}$ within the $\mathbb{S}^2(r)$ space. An $\oplus$ operation is defined as
\begin{equation}
\setlength{\abovedisplayskip}{0.15cm} 
\setlength{\belowdisplayskip}{0.15cm} 
    \mathbf{x}\oplus\mathbf{v}=\mathbf{R}(\mathbf{v})\cdot\mathbf{x}
\end{equation}
where $\mathbf{v}\in\mathbf{R}^3$.

Then the differentiations are calculated using chain rules as:
\begin{equation}
\setlength{\abovedisplayskip}{0.15cm} 
\setlength{\belowdisplayskip}{0.15cm} 
\begin{aligned}
     &\frac{\partial\left(\!\left(\!\left(\!\mathbf{x}\boxplus\mathbf{u}\right)\oplus \mathbf{v}\right)\boxminus\mathbf{y}\right)}{\partial\mathbf{u}}\!=\!\mathbf{N}\! \left(\!(\mathbf{x}\! \boxplus \! \mathbf{u}) \! \oplus \! \mathbf{v},\!\mathbf{y} \! \right) \! {\mathbf R}(\mathbf{v})  \mathbf{M}\left(\mathbf{x},\!\mathbf{u}\right)\\
     &\frac{\partial\left(\!\left(\!\left(\!\mathbf{x}\boxplus\mathbf{u}\right)\oplus \mathbf{v}\right)\boxminus\mathbf{y}\right)}{\partial\mathbf{v}}\!=\!-\mathbf{N}\! \left(\!(\mathbf{x}\! \boxplus \! \mathbf{u}) \! \oplus \! \mathbf{v},\!\mathbf{y} \! \right) \! {\mathbf R}(\mathbf{v}) \lfloor\mathbf{x}\!\boxplus\!\mathbf{u}\rfloor\! \mathbf{A}\!\left(\!\mathbf{v}\!\right)^T
\end{aligned}
\end{equation}
where in respective step of the chain rule, $\mathbf{N}\left(\mathbf{x},\mathbf{y}\right)$ and $\mathbf{M}\left(\mathbf{x},\mathbf{u}\right)$ are calculated as:
\begin{equation}
    \label{eq:S2_sup}
\setlength{\abovedisplayskip}{0.15cm} 
\setlength{\belowdisplayskip}{0.15cm} 
    \begin{aligned}
     \mathbf{N}\!\left(\mathbf{x},\!\mathbf{y}\right)&\!=\!\frac{\partial\left(\mathbf{x}\boxminus\mathbf{y}\right)}{\partial\mathbf{x}}\!=\!\mathbf{B}\!\left(\mathbf{y}\right)^{T}\!\!\left(\frac{\theta}{\| \! \lfloor\mathbf{y} \! \rfloor\mathbf{x}\|}\lfloor\! \mathbf{y} \! \rfloor\!+\!\lfloor \! \mathbf{y} \! \rfloor\mathbf{x}\!\cdot\!\mathbf{P}\!\left(\! \mathbf{x}, \! \mathbf{y}\right)\right)\\
     \mathbf{M}\!\left(\mathbf{x},\!\mathbf{u}\right)&\! = \! \frac{\partial\left(\mathbf{x}\boxplus\mathbf{u}\right)}{\partial\mathbf{u}} \!=\! - {\rm Exp}(\mathbf B (\! \mathbf x \! ) \mathbf u) \lfloor \! \mathbf x \! \rfloor  \mathbf A \! \left(\mathbf B( \! \mathbf x \! ) \mathbf u \right)^{T} \mathbf B( \! \mathbf x \! ) \\
\mathbf{P}\!\left(\mathbf{x},\mathbf{y}\right) &\!=\!\frac{1}{{r}^4}\left(\frac{-\mathbf{y}^T\mathbf{x}\|\lfloor\mathbf{y}\rfloor\mathbf{x}\|+{r}^4\theta}{\|\lfloor\mathbf{y}\rfloor\mathbf{x}\|^3}\mathbf x^T \lfloor \mathbf y \rfloor^2 \!-\!\mathbf{y}^T \right)
\end{aligned}
\end{equation}

\begin{figure}[t]
	\centering
	\includegraphics[width=0.7\columnwidth]{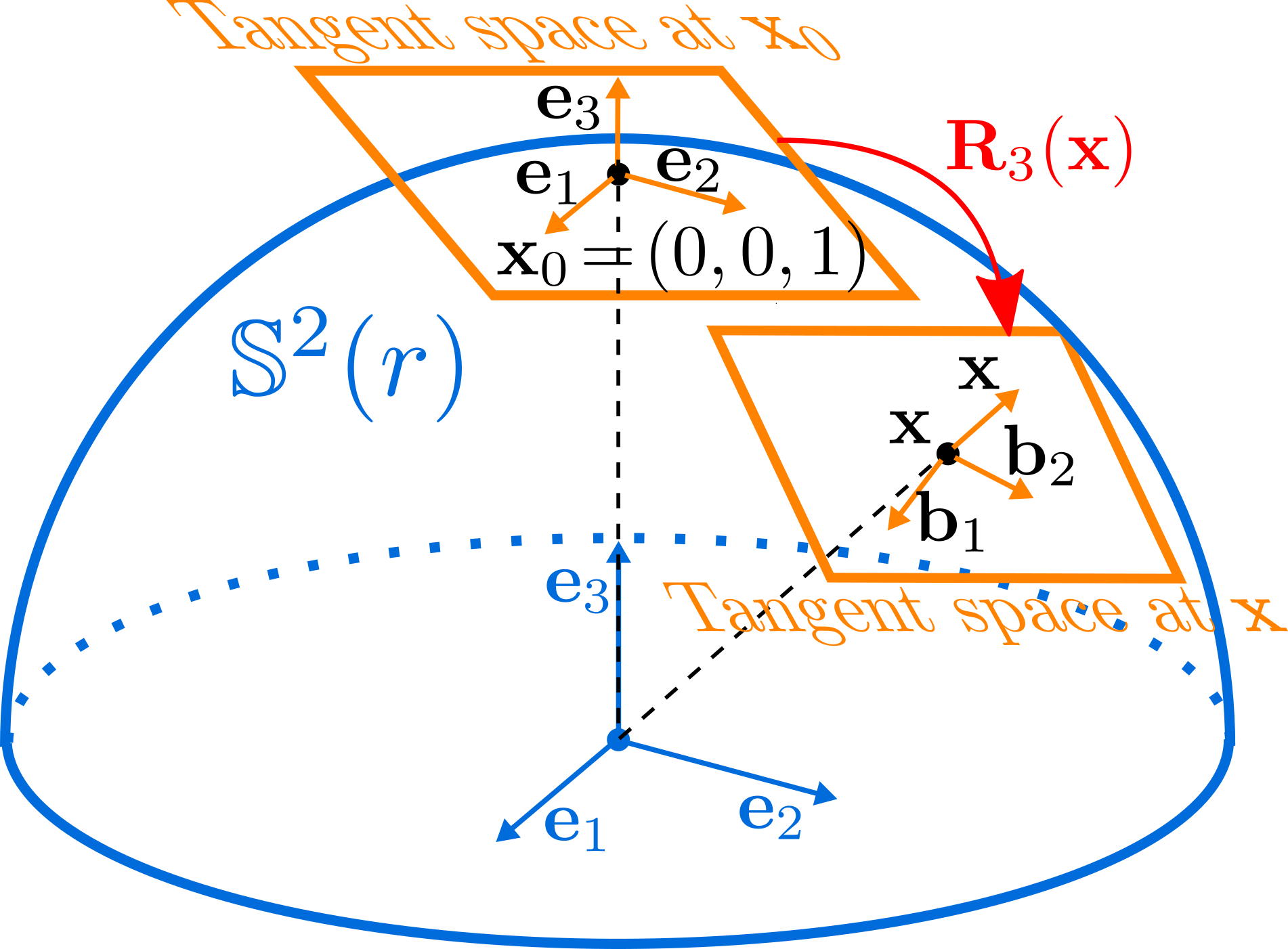}
	\caption{Method adopted in~\protect\cite{Hertzberg2011Sensorfusionsupp} to obtain the orthonormal basis ($\mathbf{b}_1$, $\mathbf{b}_2$) in the tangent plane of $\mathbf{x}\in\mathbb{S}^2(r)$. $\mathbf{R}_3(\mathbf{x})$ is defined by rotating $\mathbf{e}_3$ to $\mathbf{x}$, $\mathbf{b}_1$ and $\mathbf{b}_2$ are obtained through rotating $\mathbf{e}_1$ and $\mathbf{e}_2$ by $\mathbf{R}_3(\mathbf{x})$ respectively.
		}
		\label{fig:S2_basis_sup}
	\vspace{-0.3cm}
\end{figure}

\subsection{A partial differential on \texorpdfstring{ $SO(3)$}{Lg}}~\label{app:pdiff_on_SO3}
\begin{lemma}~\label{lemma:pd_SO3}
If $\mathbf{x}\in SO(3)$, $\mathbf{a}\in\mathbb{R}^n$, then $\left.\frac{\partial(\mathbf{x}\boxplus\mathbf{u})\cdot\mathbf{a}}{\partial\mathbf{u}}\right|_{\mathbf{u}=\mathbf{0}}=-\mathbf{x}\cdot\lfloor\mathbf{a}\rfloor$
\end{lemma}
\begin{proof}
For $\mathbf{x}\in SO(3)$ and a vector $\mathbf{a}\in\mathbb{R}^3$:
\begin{equation}
\setlength{\abovedisplayskip}{0.15cm} 
\setlength{\belowdisplayskip}{0.15cm} 
    \begin{aligned}
       \left.\frac{\partial\left((\mathbf{x}\boxplus\mathbf{u})\cdot\mathbf{a}\right)}{\partial\mathbf{u}}\right|_{\mathbf{u}=\mathbf{0}}&= \lim\limits_{\mathbf{u}\to 0} \frac{(\mathbf{x}\boxplus\mathbf{u})\cdot\mathbf{a}-\mathbf{x}\cdot\mathbf{a}}{\mathbf{u}}\\
       &=\lim\limits_{\mathbf{u}\to 0}\frac{\mathbf{x}\cdot\rm{Exp}(\mathbf{u})\cdot\mathbf{a}-\mathbf{x}\cdot\mathbf{a}}{\mathbf{u}}\\
       &=\lim\limits_{\mathbf{u}\to 0}\frac{\mathbf{x}\cdot(\mathbf{I}+\lfloor\mathbf{u}\rfloor)\cdot\mathbf{a}-\mathbf{x}\cdot\mathbf{a}}{\mathbf{u}}\\
       &=\lim\limits_{\mathbf{u}\to 0}\frac{\mathbf{x}\cdot\lfloor\mathbf{u}\rfloor\cdot\mathbf{a}}{\mathbf{u}}=\lim\limits_{\mathbf{u}\to 0}\frac{-\mathbf{x}\cdot\lfloor\mathbf{a}\rfloor\cdot\mathbf{u}}{\mathbf{u}}\\
       &=-\mathbf{x}\cdot\lfloor\mathbf{a}\rfloor
    \end{aligned}
\end{equation}
\end{proof}
\iffalse
For the state or measurement $\mathbf{x}\in\mathbb{S}^2(r)$, $\frac{\partial\left(\mathbf{x}\boxplus\mathbf{u}\right)}{\partial\mathbf{u}}$, $\frac{\partial\left(\mathbf{x}\boxminus\mathbf{y}\right)}{\partial\mathbf{x}}$ are shown explicitly in equation (\ref{eq:S2_sup}), and 
\begin{equation}
    \frac{\partial\left(\mathbf{x}\oplus\mathbf{v}\right)}{\partial\mathbf{x}}={\rm Exp}\left(\mathbf{v}\right)
\end{equation}
\fi

\bibliography{paper} 

\end{document}